\documentclass{article}

\usepackage[letterpaper,top=2cm,bottom=2cm,left=3cm,right=3cm,marginparwidth=1.75cm]{geometry}
\usepackage{natbib}

\usepackage[utf8]{inputenc} % allow utf-8 input
\usepackage[T1]{fontenc}    % use 8-bit T1 fonts
\usepackage[colorlinks,citecolor=blue]{hyperref}       % hyperlinks
\usepackage{url}            % simple URL typesetting
\usepackage{booktabs}       % professional-quality tables
\usepackage{amsfonts}       % blackboard math symbols
\usepackage{nicefrac}       % compact symbols for 1/2, etc.
\usepackage{microtype}      % microtypography
\usepackage{xcolor}         % colors

\usepackage{algorithm}
\usepackage{algorithmic}
\usepackage{booktabs}

\usepackage{amsmath}
\usepackage{amsthm} % proof
\usepackage{amsfonts,amssymb}
\usepackage{graphicx}

\usepackage{makecell}
\usepackage{booktabs}
\usepackage{threeparttable}
\usepackage{multirow,multicol}
\usepackage{bm}
\usepackage{subfigure}
\usepackage{enumitem}

\newtheorem{lemma}{Lemma}
\newtheorem{example}{Example}
\newtheorem{theorem}{Theorem}

\newtheorem{definition}{Definition}

\ifx\assumption\undefined
\newtheorem{assumption}{Assumption}

\fi

\newcommand{\Norm}[1]{\left\|#1\right\|}

\newcommand{\dotprod}[2]{\left\langle#1,#2\right\rangle}

\newcommand{\Kr}{K_{\rho}}
\newcommand{\paren}[1]{\left( #1 \right)}

\newcommand{\bw}{\mathbf{w}}

\newcommand{\bx}{\mathbf{x}}

\newcommand{\by}{\mathbf{y}}

\newcommand{\bz}{\mathbf{z}}
\newcommand{\bn}{\mathbf{n}}
\newcommand{\bg}{\mathbf{g}}

\newcommand{\EE}{\mathbb{E}}
\newcommand{\cO}{\mathcal{O}}

\newcommand{\RR}{\mathbb{R}}

\allowdisplaybreaks[1]

\title{\textbf{Theoretical Analysis on how Learning Rate Warmup Accelerates Convergence}}

\author{%
{Yuxing Liu$^{1}$\thanks{Equal Contribution.}, \quad
Yuze Ge$^{1}$\footnotemark[1], \quad
Rui Pan$^{1}$, \quad Kang An$^{2}$, \quad
Tong Zhang$^{1}$}
\\ \\
{
$^{1}$University of Illinois Urbana-Champaign}
\quad {$^{2}$Rice University } \\
\normalsize\texttt{\{yuxing6,ruip4,tozhang\}@illinois.edu, yzge42@gmail.com, kang.an@rice.edu}
}

\begin{document}

\maketitle

\begin{abstract}
    Learning rate warmup is a popular and practical technique in training large-scale deep neural networks. Despite the huge success in practice, the theoretical advantages of this strategy of gradually increasing the learning rate at the beginning of the training process have not been fully understood. To resolve this gap between theory and practice, we first propose a novel family of generalized smoothness assumptions, and validate its applicability both theoretically and empirically. 
    Under the novel smoothness assumption, we study the convergence properties of gradient descent (GD) in both deterministic and stochastic settings. It is shown that learning rate warmup consistently accelerates GD, and GD with warmup can converge at most $\Theta(T)$ times faster than with a non-increasing learning rate schedule in some specific cases, providing insights into the benefits of this strategy from an optimization theory perspective.
\end{abstract}

\section{Introduction}
Mathematically, training a machine learning model can be formulated as a minimization problem:
\begin{align*}
    \min_{\bw \in \RR^d} \frac{1}{N} \sum_{i=1}^N f_i(\bw) ,
\end{align*}
where first-order optimizers using the gradient information are normally applied to find a solution to the problem. 
Carefully tuning the learning rates (or step sizes) is crucial in this optimization procedure, especially when the problem scale is large. 
A time-varying learning rate schedule is very commonly used both in theory (e.g., for Nesterov accelerated gradient method)~\citep{malitsky2020adaptive,teboulle2023elementary, boyd2004convex} and in practice (e.g., cosine schedule)~\citep{ he2016deep, vaswani2017attention,loshchilov2017sgdr,touvron2023llama}. 

Learning rate warmup is a strategy commonly incorporated in those schedules during the initial phase of training deep neural networks. In this stage, the learning rate, denoted as $\eta$, is set to a value lower than its target or base level. This initially small learning rate is then gradually increased over a number of training iterations until it reaches the intended peak value. A prevalent example of this is the linear warmup strategy~\citep{goyal2017accurate}, which sets a $0$ initial value and increases it linearly to the target learning rate in the initial phase. The warmup strategy has been widely observed to be powerful across many practical tasks~\citep{he2016deep, goyal2017accurate, vaswani2017attention}.

Despite the impressive practical success of learning rate warmup, a rigorous theoretical explanation for why warmup works still remains unclear. 
Various studies have explored or empirically validated explanations for the benefits of learning rate warmup, including limiting the magnitude of weight updates and reducing variance~\citep{gotmare2018closer,liu2020variance,gilmer2022loss,kalra2024warmup,kosson2024analyzing}.
Among them, \cite{gilmer2022loss,kalra2024warmup} elaborate on the intuition that the main advantage of learning rate warmup is that small initial learning rates allow the model to safely go into smoother regions of the loss landscape, characterized by smaller local smoothness (or sharpness), i.e., the largest singular value of the Hessian, in the initial phase of training.
This is beneficial since the applicable learning rate scale at a specific point $\bw$ generally needs to be bounded by $2 / L(\bw)$, where $L(\bw)$ is the local smoothness~\citep{cohen2021gradient}, which implies that first going into a smoother region enables larger learning rates in the following training process, resulting in faster convergence.

The connection between learning rate warmup and local smoothness inspires us to study the benefits of the warmup strategy from an optimization perspective.
To mathematically model the varying local smoothness during training, we propose a novel family of smoothness assumptions that connect local smoothness with the suboptimality gap of the loss function, i.e., $f(\bw) - f^*$. Note that this is a closely relevant but different family of assumptions with existing generalized smoothness assumptions that link the local smoothness with gradient norm~\citep{zhang2020gradient,li2023convex}.
We show that this new family of generalized smoothness assumption is typically weaker than existing generalized smoothness assumptions, and provide examples to show it applicability for analyzing the convergence of neural networks both empirically and theoretically.
Based on this novel family of assumptions, we study the convergence of the standard gradient descent (GD) and stochastic gradient descent (SGD) algorithms. The novel assumption's rigorous characterization of the evolution of local smoothness during the optimization process enables the proof. By comparing algorithms with and without a warmup phase, we find that using warmup shows a consistent gain in accelerating convergence, which can even achieve $\Theta(T)$ times faster convergence speed for GD and $\Theta(\sqrt{T})$ times for SGD.

Our main contributions are summarized as follows:
\begin{enumerate}
    \item We propose a novel family of generalized smoothness assumptions, connecting the local smoothness with the suboptimality gap. We prove that this novel family of assumptions is strictly weaker than the existing generalized $(\rho,K_0,K_\rho)$-smoothness with respect to the gradient norm~\citep{zhang2020gradient,li2023convex} for $\rho < 2$. 
    {Experimental validation on typical deep learning models, along with several neural network examples, demonstrates the applicability of our generalized smoothness assumptions to practical optimization tasks, especially training deep neural networks.}

    \item 
    Based on our generalized smoothness assumptions, we theoretically prove that using a warm-up learning rate schedule can accelerate the convergence of gradient descent (GD) and stochastic gradient descent (SGD) methods, thereby bridging the gap between theory and practice in training neural networks. Specifically, it is shown that under a specific way of warming up learning rates, GD can achieve $\Theta(T)$ times faster convergence rates compared to directly using non-increasing learning rates.
    For SGD, we apply the ABC inequality~\citep{khaled2023better} as the noise assumption, which is general and implies further benefits of doing warmup in accelerating convergence in a noisy setting.

\end{enumerate}

\section{Related Work}
\paragraph{Learning rate warmup.}
Learning rate warmup is a widely employed heuristic for training deep neural networks.
The use of learning rate warmup dates back at least to \cite{he2016deep}, which used a small constant learning rate during the first stage of training. Later, the linear warmup strategy was introduced by \cite{goyal2017accurate}, and soon became popular for training a large range of models, including ResNets~\citep{he2016deep} and transformers~\citep{vaswani2017attention}. 
Empirical evidence showed that learning rate warmup can enhance training stability to allow large learning rates and improve model performance~\citep{gotmare2018a,gilmer2022loss,kalra2024warmup}.

\paragraph{Intuitions for the benefits of warmup.}

In \cite{goyal2017accurate}, the authors proposed that to use a larger batch size, the learning rate should be scaled up proportionally. \cite{l.2018dont, jastrzębski2018three} theoretically studied how the ratio between batch size and learning rate affects the training dynamics of SGD. However, in many cases, the learning rate cannot directly increase proportionally to the batch size in order to maintain training stability. Thus, warmup was introduced by \cite{goyal2017accurate} as a trick for gradually increasing learning rates.
After that, studies on the warmup mechanism appeared. \cite{gotmare2018closer} found that warmup prevents training instability by limiting the updates to deep-layer weights through empirical analysis.
\cite{liu2020variance} specifically studied Adam~\citep{kingma2014adam} and attributed training instability to the large variance caused by the adaptive step size of Adam and viewed warmup as a method of variance reduction. 
Other work suggested that learning rate warmup enables the model to enter 
smoother regions of the loss landscape, leading to a gradual decrease in local smoothness (sharpness)~\citep{gilmer2022loss, kalra2024warmup}. Based on the relation between learning rates and the local smoothness~\citep{nesterov2018lectures,cohen2021gradient}, this enables larger learning rates in the following training process, thereby accelerating the convergence. \cite{wen2024understanding} also provided a similar understanding by proposing an intuitive river-valley interpretation of the neural networks' landscape.

\paragraph{Generalized Smoothness.} The smoothness condition plays a significant role in optimization theory. 
For a twice-differentiable function, the standard $L$-smooth assumption assumes an upper bound $L$ on the largest singular value of the Hessian~\citep{nesterov2018lectures}, where $L$ is a constant.
\cite{zhang2020gradient} was probably the first to generalize the upper bound $L$ to be a linear function of the current gradient norm, i.e., $\Norm{\nabla^2 f(\bw)} \le L(\bw) = L_0 + L_1 \Norm{\nabla f(\bw)}$, which is strictly weaker than the standard smoothness condition and is verified to be valid in some small neural networks. The idea was followed by \cite{zhang2020improved}, which derived finer properties of the generalized smoothness. Further extensions of this generalized smoothness have also been developed since then. \cite{li2023convex} extended the linear function of $\Norm{\nabla f(\bw)}$ to $\Norm{\nabla f(\bw)}^\rho$ with $\rho \ge 1$ and proved that GD {with a constant learning rate} converges if and only if $\rho < 2$. In another direction, \cite{crawshaw2022robustness,liu2024adagrad} developed anisotropic versions of the generalized smoothness assumption.

\paragraph{Convergence under Generalized Smoothness.}
The convergence of SGD under the $L$-smoothness assumption has been extensively studied. For the $(L_0,L_1)$-smoothness, most analyses focused on varying learning rates, such as SGD with clipping~\citep{zhang2020gradient, zhang2020improved, qian2021understanding}, SignSGD~\citep{crawshaw2022robustness}, and normalized SGD~\citep{zhao2021convergence}.
Moreover, their analyses often rely on the bounded noise assumption or the subgaussian noise assumption. \cite{li2023convex} proved the convergence of SGD with constant learning rate under $(\rho,L_0,L_\rho)$-smoothness with $0\leq\rho<2$, by bounding the gradients along the optimization trajectory. Their constant learning rate depends on the intial suboptimality gap, which in turn depends on both the loss function and initialization. 
\citet{tyurin2025toward} proposes a specific adaptive learning rate, under which GD converges for $(\rho, L_0,L_\rho)$-smooth functions for any $\rho>0$. However, the proposed learning rate involves computing an integral, which typically does not have a closed-form expression, and is also not necessarily monotonic, making it less practical and different from our settings.
Note that the lower bound of SGD under the $L$-smoothness and bounded variance conditions is $\Omega(1/T^{1/4})$~\citep{arjevani2023lower}. The above analyses, under generalized smoothness conditions, also achieve the $O(1/T^{1/4})$ bound, though some of them rely on stronger noise assumptions.

\section{A Family of Novel Generalized Smoothness Assumptions}\label{sec:FamilySmoothness}

We first review the existing smoothness assumptions. For a twice continuously differentiable function $f:\RR^d \to \RR$, the standard $L$-smooth assumption assumes that the spectral norm of the Hessian of the loss function is uniformly bounded, i.e.,
$$\Norm{\nabla^2 f(\bw)}\leq L,\quad \forall \bw\in\mathbb{R}^d.$$
Although the $L$-smoothness assumption is widely used in optimization theory, it fails to capture the local smoothness of the loss function at different points, and even some simple and common functions, such as the exponential function, do not satisfy this assumption~\citep{zhang2020gradient}.

To generalize $L$-smoothness, \cite{li2023convex} proposed the $(\rho,L_0,L_\rho)$-smoothness:
$$\Norm{\nabla^2 f(\bw)}\leq L_0+L_\rho\Norm{\nabla f(\bw)}^\rho,\quad\forall\bw\in\mathbb{R}^d.$$
When $\rho=1$, it reduces to the $(L_0,L_1)$-smoothness~\citep{zhang2020gradient}. The $(\rho,L_0,L_\rho)$-smoothness assumes that local smoothness is bounded by an increasing polynomial function of the gradient norm and is considered to be more consistent with the deep neural networks than the 
$L$-smooth assumption based on some empirical verifications~\citep{zhang2020gradient}. We are particularly interested in the 
$0\leq\rho<2$ case, where local smoothness is bounded by a sub-quadratic function of the gradient norm. This is because \cite{li2023convex} 
showed that GD may diverge for $(\rho,L_0,L_\rho)$ functions with $\rho\geq 2$. 

Although $(\rho,L_0,L_\rho)$-smoothness has been empirically validated as an effective assumption for characterizing the loss landscape of deep neural networks, there are still some limitations.
Firstly, there exist simple examples showing that neural networks do not satisfy the $(\rho,L_0,L_\rho)$-smoothness with $0\leq\rho<2$~\citep{patel2022global}. We discuss some examples in Section~\ref{sec:FamilySmoothness_nerualnets} in detail. This implies that, based on the results in \cite{li2023convex}, fundamental first-order optimizers like GD can diverge even under some simple examples, which is inconsistent with real practice.
Moreover, for nonconvex functions, the gradient norm is not necessarily monotonically decreasing during the optimization process of GD, 
{making the $(\rho,L_0,L_\rho)$-smoothness assumption inappropriate to characterize the decreasing trend of the sharpness, especially in the early stages of training neural networks~\citep{kalra2024warmup, gilmer2022loss}.}
These limitations raise a need for developing a novel family of generalized assumptions.

\subsection{A Novel Family of Generalized Smoothness}
We consider the following $(\rho, K_0, K_\rho)$-smoothness, which relates the local smoothness with the function suboptimality gap.
\begin{definition}\label{defi:K_smooth}
    We say a twice differentiable function $f:\mathbb{R}^d\to\mathbb{R}$ is $\left(\rho,K_0,\Kr\right)$ smooth if
    \begin{equation}\label{eq:defi_K_smoothn}
        \Norm{\nabla^2 f(\bw)}\leq K_0+\Kr\left(f(\bw)-f^\star\right)^\rho
    \end{equation}
    for $K_0,\Kr\geq0$ and $\rho>0$,
    where we assume $f^\star=\inf_{\bw \in \RR^d} f(\bw) > -\infty$.
\end{definition}
Note that the lower bound $f^*$ is standard in nonconvex analysis, which should also be satisfied by neural networks.
When $\Kr=0$, our generalized smoothness reduces to the classical $L$-smoothness. It is not hard to see that the $(\rho,K_0,\Kr)$-smoothness is strictly weaker than the $L$-smoothness since exponential functions are $(\rho,K_0,\Kr)$-smooth but not $L$-smooth. Moreover, we can prove that the $(\rho,K_0,\Kr)$-smoothness family is also weaker than the $(\rho,L_0,L_\rho)$-smoothness for $0\leq\rho<2$.

\begin{lemma}\label{lem:L_in_K_smooth}
    If a function $f:\RR^d \to \RR$ is $(\rho, L_0, L_\rho)$-smooth with $0\leq \rho<2$, then it is $(\alpha, K_0, K_\alpha)$-smooth with $\alpha=\frac{\rho}{2-\rho}$.
\end{lemma}
Based on Lemma~\ref{lem:L_in_K_smooth}, properties of $(\rho,L_0,L_\rho)$-smoothness for $0\le \rho<2$ as well as its applicability to deep neural networks can be inherited by $(\rho,K_0,\Kr)$-smoothness. Moreover, the following simple example shows that $(\rho,K_0,\Kr)$-smoothness is strictly weaker than 
$(\rho,L_0,L_\rho)$-smoothness with $0\leq\rho<2$ and cannot be covered by $\rho\ge 2$.

\begin{example}\label{example:trigonometric}
    The following function
    \begin{align*}
    f(x) = \left\{ \begin{array}{cc}
       2x + x\sin x , & x \in [0,+\infty), \\
       2(\mathrm{e}^x - 1) ,  & x \in (-\infty,0).
    \end{array} \right. 
\end{align*}
    is $(1,K_0,K_1)$-smooth but not $(\rho,L_0,L_\rho)$-smooth for any $\rho>0$.
\end{example}

This example also illustrates that compared to $(\rho,L_0,L_\rho)$-smoothness, $(\rho,K_0,\Kr)$-smoothness is better at capturing the properties of functions with multiple stationary points or local minima, which is a common case in deep neural network training.

\subsection{Generalized Smoothness in Neural Networks}\label{sec:FamilySmoothness_nerualnets}

We have shown that $(\rho,K_0,\Kr)$-smoothness is a more general assumption than $(\rho,L_0,L_\rho)$-smoothness for $0\le \rho<2$. Next, we demonstrate that  $(\rho,K_0,\Kr)$-smoothness is more applicable to deep neural networks. We adopt the two examples in~\cite{patel2022global}, where for binary classification tasks, simple feed forward network and recurrent neural network both fail to satisfy the $(\rho,L_0,L_\rho)$-smoothness with $0\leq\rho<2$, but satisfy the 
$(\rho,K_0,K_\rho)$-smoothness.

\begin{example}[Example 1, \cite{patel2022global}]
    Consider the following simple multi-layer feed forward network for binary classification:
    \begin{align*}
        &z_i=\sigma\left(w_iz_{i-1}\right),\quad i=1,2,3\\
        &\hat{y}=\varphi\left(w_4z_3\right),
    \end{align*}
    where $z_0$ is the input feature, $\sigma$ is the activation function and $\varphi$ is the sigmoid function.
    Given a sample point $(z_0,y)$, we aim to predict $y$.
    Let $f(\bw)$ be the cross entropy loss plus a a ridge penalty.
    Then for some simple distribution, $f(\bw)$ is $(\rho, K_0, K_\rho)$-smooth for $\rho\geq 3$ but not $(\rho, L_0, L_\rho)$-smooth for any $0\leq\rho<2$.
\end{example}

\begin{example}[Example 2, \cite{patel2022global}]
    Consider the following simple recurrent neural network for binary classification:
    \begin{align*}
        &h_i=\sigma\left(w_1 h_{i-1}+w_2 z_i\right),\quad i=0,1,2,3\\
        & \hat{y}=\varphi(w_3h_3),
    \end{align*}
    where $\sigma$ is the activation function and $\varphi$ is the sigmoid function. Given a sample point $(z_0,z_1,z_2,z_3,y)$, 
    we sequentially observe $z_0,\dots,z_3$ and aim to predict $y$.
    Let $f(\bw)$ be the cross entropy loss plus a a ridge penalty.
    Then for some simple distribution, $f(\bw)$ is $(\rho, K_0, K_\rho)$-smooth for $\rho\geq 3$ but not $(\rho, L_0, L_\rho)$-smooth for any $0\leq\rho<2$.
\end{example}

We elaborate on these two examples in detail in Appendix~\ref{apx:examples}. In both cases, the loss function of neural networks either do not satisfy the $(\rho,L_0,L_\rho)$-smoothness or only satisfy the case with $\rho\geq 2$. For the latter, GD {with constant learning rates} cannot guarantee convergence without additional assumptions. In contrast, in Section~\ref{sec:GD}, we show that GD can converge for $(\rho,K_0,K_\rho)$-smooth functions for any $\rho\geq 0$, highlighting the advantage of the $(\rho,K_0,K_\rho)$-smoothness assumption.

\begin{figure}[t]
\vspace{-0.5cm}
    \centering
    \subfigure[ResNet ]{\includegraphics[width=0.45\textwidth]{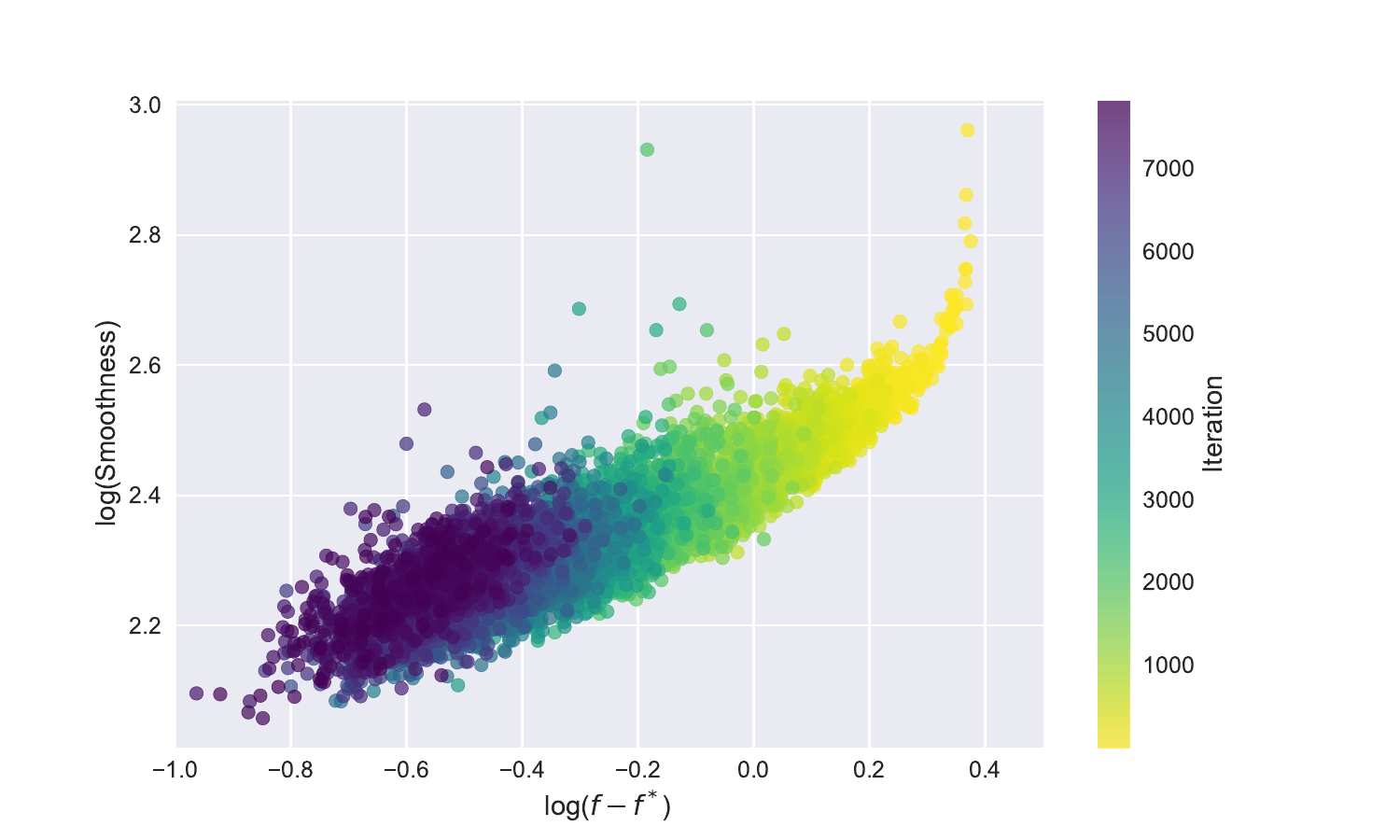}}\hfill
    \subfigure[NanoGPT]{\includegraphics[width=0.45\textwidth]{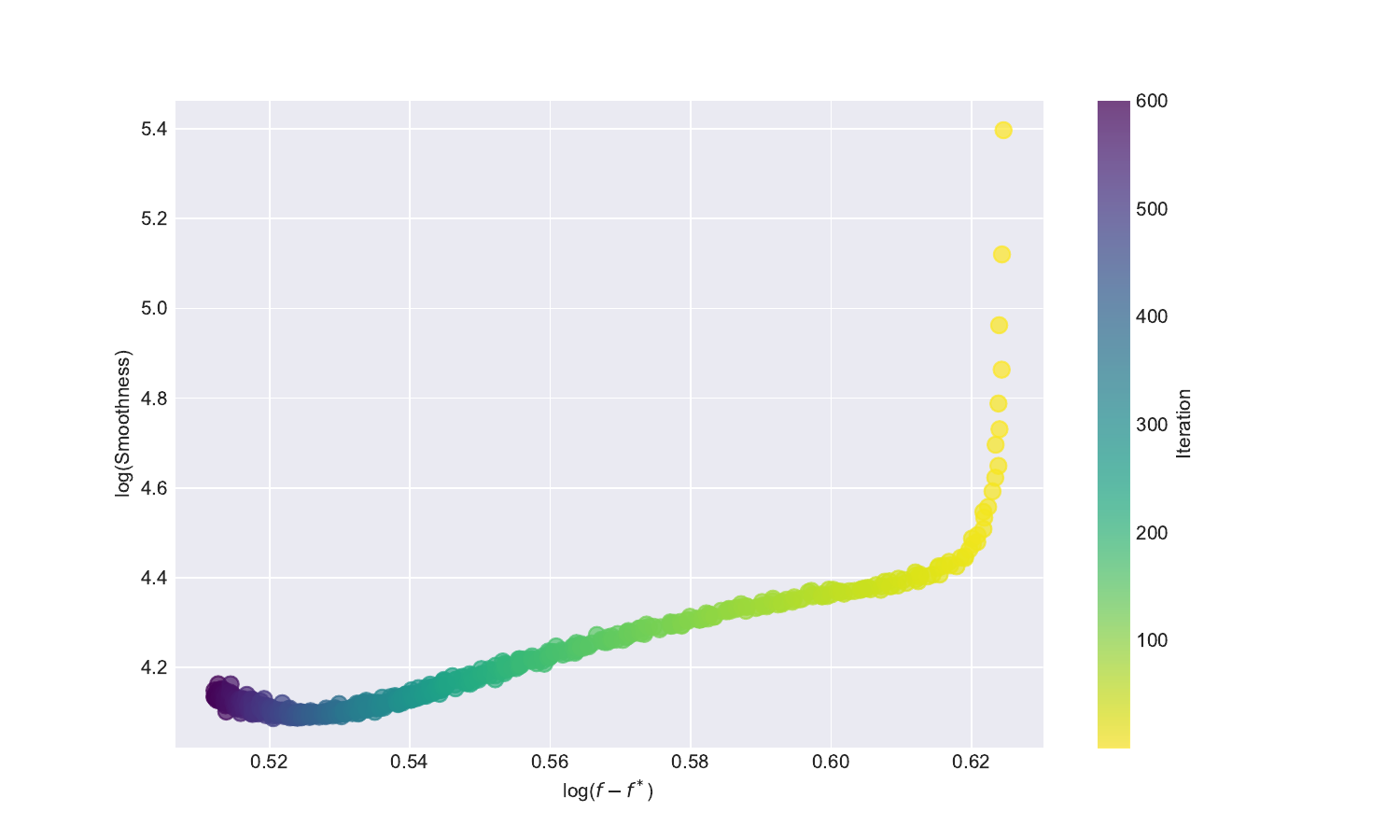}}\hfill
  \vspace{-0.3cm}
  \caption{Local smoothness vs. function suboptimality gap on training (a) ResNet18 on CIFAR-10 (b) NanoGPT on Tiny TinyShakespeare character dataset. Both $x$ and $y$ axes are in log scale and the color bar indicates the iteration number. We use $f^* = 0$ in the plots.}
  % \vspace{-0.5cm}
  \label{fig:assumption-verification}
\end{figure}

\subsection{Empirical Validation of the Assumption}

To empirically investigate the posited relationship between local smoothness and the loss sub-optimality gap within neural networks, we perform numerical experiments. As direct Hessian computation is often intractable, we approximate local smoothness following~\cite{zhang2020gradient,crawshaw2022robustness}. Given consecutive iterates $\bw_t$ and $\bw_{t+1}$, we define the update direction $\mathbf{d}_t \triangleq \bw_{t+1} - \bw_{t}$. The smoothness $\hat{L}(\bw_t)$ is then estimated by:
\begin{align*}
\hat{L}(\bw_t) = \max_{\gamma \in \{\delta_1,\dots,\delta_n\} } \frac{\Norm{\nabla f(\bw_t + \gamma\mathbf{d}_t) - \nabla f(\bw_{t})}_2}{\Norm{\gamma \mathbf{d}_t}_2}
\end{align*}
where the sample points are $\delta_i=i / n$; we use $n=6$, yielding $\gamma \in\{1 / 6,2 / 6,3 / 6,4 / 6,5 / 6, 1\}$.

Our experimental validation includes both Convolutional Neural Networks (CNNs) and Transformers. The CNN configuration involves training a ResNet18 on CIFAR-10 for 20 epochs. The Transformer configuration consists of training a NanoGPT model (6 blocks, 384 embedding dimension, 6 attention heads) on the TinyShakespeare character dataset for 600 steps. Both models are trained with SGD with momentum ($\mathrm{lr}=1e-4$). Both experiments were conducted using a single NVIDIA A100(40GB) PCIE GPU.
As shown in the log-log plots, a polynomial dependence of the local smoothness on the function suboptimality gap is generally clear, showing the applicability of Assumption~\ref{assumption:K_smooth}. Also, as one can observe from the plots, the local smoothness can be extremely large at the beginning of the training process, which also provides evidence for the importance of using small learning rates in the initial phase of training.

\subsection{Properties of the Novel Generalized Smoothness}
Similar to $(\rho,L_0,L_\rho)$-smoothness~\citep{li2023convex}, 
Definition~\ref{defi:K_smooth} indicates that $\nabla f$ is locally Lipschitz continuous. Therefore, by careful analysis through integration, we are able to obtain the following locally Lipschitz continuous property of $\nabla f$.
We first define two constants $C_1, C_2$ which only depend on $K_0,\Kr$ and $\rho$:
$$C_1=\frac{1}{\left(2+\sqrt{2}\right)\sqrt{3^\rho \Kr}}, \quad
    C_2=\frac{1}{2\sqrt{3}+\sqrt{6}}\frac{K_0^{\frac{1}{2\rho}-\frac{1}{2}}}{\Kr^{\frac{1}{2\rho}}}.$$
\begin{lemma}\label{lem:improved_descent_smooth}
    Suppose $f$ is $(\rho,K_0,\Kr)$-smooth. Let $\Delta=f(\bx)-f^\star$, 
    $$L(\Delta):=2K_0+ \Kr\left(2\Delta\right)^\rho \text{ and  } r(\Delta):=\min\left\{C_1\Delta^{-\frac{\rho-1}{2}}, C_2\right\}.$$
    Then for any $\bx,\by\in\RR^d$ satisfying $\Norm{\by-\bx}\leq r(\Delta)$, we have 
    $$\Norm{\nabla f(\by)-\nabla f(\bx)}\leq L(\Delta)\Norm{\by-\bx}$$
    and
    $$f(\by)\leq f(\bx)+\dotprod{\nabla f(\bx)}{\by-\bx}+\frac{L(\Delta)}{2}\Norm{\by-\bx}^2.$$
\end{lemma}

Lemma~\ref{lem:improved_descent_smooth} is useful for our convergence analysis. As long as the consecutive iterates $\|\bw_{t+1}-\bw_t\|$ are small enough, we can obtain a descent lemma and 
proceed with an analysis similar to that used under the $L$-smoothness assumption.

\section{Theory of GD}\label{sec:GD}
In this section, we analyze GD for $(\rho,K_0,\Kr)$-smooth functions:
$$\bw_{t+1}=\bw_t-\eta_t\nabla f(\bw_t).$$
We consider two learning rate settings: constant learning rate and increasing learning rate, i.e., $\eta_t\leq\eta_{t+1},t=0,\dots,T-1$. The increasing learning rate strategy can be viewed as a specific type of learning rate warmup, which will be validated empirically in Section~\ref{sec:validation_schedule_empirical}.
We show that the increasing learning rate leads to a faster convergence rate compared to the constant learning rate.
We first list the assumptions we require for convergence analysis.

\begin{assumption}\label{assumption:lower_bounded}
    We assume $f(\bw_0)-f^\star<\infty$, where $f^\star=\inf_{\bw\in\mathbb{R}^d} f(\bw)$.
\end{assumption}

\begin{assumption}\label{assumption:K_smooth}
    $f(\bw)$ is $(\rho,K_0,\Kr)$-smooth.
\end{assumption}

We use $\Delta_t \triangleq f(\bw_t)-f^\star$ for simplicity in the following analysis.

\subsection{Upper Bounds}\label{sec:GD_upper}

We first present the results under the general nonconvex scheme.

\begin{theorem}\label{thm:GD}
    Suppose Assumptions~\ref{assumption:lower_bounded} and \ref{assumption:K_smooth} hold. $\{\bw_t\}$ is generated by GD.
    Let the learning rate 
    $\eta_t= \frac{1}{4\sqrt{2}+4}\min\left\{\frac{1}{K_0}, \frac{1}{3^\rho \Kr}\Delta_t^{-\rho}\right\}$.
    Then it holds that $\Delta_t\geq\Delta_{t+1}$ for all $t\in [T]$, and 
    \begin{align}\label{eq:thm_GD_warmup}
    \min_{t< T}\Norm{\nabla f(\bw_t)}^2\leq \frac{2\left(f(\bw_0)-f^\star\right)}{\sum_{t=0}^{T-1}\eta_t}=\mathcal{O}\left(\frac{K_0\Delta_0}{T}+\frac{\Kr\Delta_0\sum_{t=0}^{T-1}\Delta_t^\rho}{T^2}\right).
    \end{align}
    Moreover, if we use a constant learning rate $\eta= \frac{1}{4\sqrt{2}+4}\min\left\{\frac{1}{K_0}, \frac{1}{3^\rho\Kr}\Delta_0^{-\rho}\right\}$, then we have 
    $\Delta_t\geq\Delta_{t+1}$ for all $t\in[T]$, and 
    \begin{align}\label{eq:thm_GD_constant}
    \min_{t< T}\|\nabla f(\bw_t)\|^2\leq\frac{2\left(f(\bw_0)-f^\star\right)}{\eta T}=\mathcal{O}\left(\frac{K_0\Delta_0+\Kr\Delta_0^{\rho+1}}{T}\right).
    \end{align}
\end{theorem}

We can conduct a simple comparison between the two results in Theorem~\ref{thm:GD}.
Since the function gap $\Delta_t$ is monotonically decreasing during the optimization process, the learning rate schedule $\{\eta_t\}$ is monotonically 
increasing, thus can be viewed as a specific adaptive strategy of learning rate warmup.
By $\sum_{t=0}^{T-1} \Delta_t^\rho\leq T\Delta_0^{\rho}$, we know that the convergence rate with learning rate warmup is better than that with a constant learning rate, showing an acceleration effect when $K_\rho$ is significant, i.e., the local smoothness is varying and highly dependent on the suboptimality. 
This can likely happen since $K_0$ can be quite small, as it doesn't need to globally bound the Hessian norm as in the case of $L$-smoothness. 
Moreover, to provide more insights into how significant this gap can be, we further analyze the convex convergence of GD as presented in Theorem~\ref{thm:GD_convex}.
\begin{theorem}\label{thm:GD_convex}
{
    Suppose Assumptions~\ref{assumption:lower_bounded} and \ref{assumption:K_smooth} hold. Further assume that $f$ is convex. Define $\bw^\star=\arg\min_{\bw\in\mathbb{R}^d} f(\bw)$ and $D_0=\|\bw_0-\bw^\star\|$.
    If we use the learning rate schedule $\eta_t=\frac{1}{8\sqrt{2}+8}\min\left\{\frac{1}{K_0}, \frac{1}{3^\rho\Kr}\Delta_t^{-\rho}\right\}$, then we have $\Delta_{t+1}\leq\Delta_t$, and 
    $$
    f(\bw_{T-1})-f(\bw^\star)\leq \mathcal{O}\paren{\frac{D_0^2 K_0}{T}+\frac{\paren{D_0^2\Kr}^{\max \left\{\frac{1}{1-\rho}, 1\right\}}\Delta_0^{\max\left\{\rho-1,0\right\}}}{T^{\max\left\{\frac{1}{1-\rho},0\right\}}}}.$$
    Moreover, if we use the constant learning rate $\eta_t=\eta=\frac{1}{8\sqrt{2}+8}\min\left\{\frac{1}{K_0}, \frac{1}{3^\rho\Kr}\Delta_0^{-\rho}\right\}$, then we have $\Delta_{t+1}\leq\Delta_t$, and
    $$
    f(\bw_{T-1})-f(\bw^\star)\leq \mathcal{O}\paren{\frac{D_0^2K_0}{T}+\frac{D_0^2\Kr\Delta_0^\rho}{T}}.$$
}
\end{theorem}
{
From Theorem~\ref{thm:GD_convex}, to obtain an $\epsilon$-optimal solution $f(\bw)-f(\bw^\star)\leq\epsilon$, the required iteration number is 
$ \mathcal{O}\left(\frac{K_0 D_0^2}{\epsilon}+\frac{\Kr D_0^2 \Delta_0^{\max\{0,\rho-1 \}}}{\epsilon^{\max\{0,1-\rho\}}}\right)$ for the warm-up schedule and
$\mathcal{O}\left(\frac{K_0 D_0^2}{\epsilon}+\frac{K_\rho D_0^2 \Delta_0^\rho}{\epsilon}\right)$ for the constant learning rate. Therefore, we can clearly see that the convergence rate of GD with the specific warm-up schedule is strictly better than that of GD with a constant learning rate if $\Kr>0$. Specifically, if $\rho\geq 1$, the convergence rate of GD with the warmup learning rate schedule is $\mathcal{O}\left(\frac{K_0 D_0^2}{\epsilon}+\Kr D_0^2 \Delta_0^{\rho-1 }\right)$, which implies an acceleration of $\Theta(\Delta_0 T)$ compared to using a constant learning rate.
}

We also come across the following simple but intuitive example to illustrate that the difference in $K_\rho$ terms can be significant.

\begin{example}\label{example:river_valley}
    Consider a specific $2$-dimensional function $f(x,y) = h(x) + g(y)$, with
    \begin{align*}
        h(x) = \left\{ \begin{array}{cc}
            \mathrm{e}^{-\sqrt{K_1}x - 1} - \frac{1}{2}, & x \in (-\infty, -\frac{1}{\sqrt{K_1}}) \\
            \frac{1}{2} K_1 x^2 , & x\in [-\frac{1}{\sqrt{K_1}}, \frac{1}{\sqrt{K_1}}] \\
            \mathrm{e}^{\sqrt{K_1}x - 1} - \frac{1}{2}, & x \in (\frac{1}{\sqrt{K_1}}, +\infty)
        \end{array} \right.
        ,\quad
        g(y) = \frac{1}{2} K_1 y^2.
    \end{align*}
    {Note that $f$ is $(1,K_1, K_1)$-smooth and $f^* = 0$}.
    This function is a simple construction that approximates the river-valley loss landscape presented in \cite{wen2024understanding}, which is believed to capture the properties of neural networks. The landscape is very sharp along the $x$-axis and mild along the $y$-axis, where the $y$-axis can be interpreted as the river. 
    Consider the initialization $x_0 = \frac{\log(\Delta_0)}{\sqrt{K_1}}$ with $\Delta_0 > \mathrm{e}$ and $y_0 > 1$. 
    Then GD with warmup can converge $\tilde{\Theta}(\Delta_0)$ times faster than using constant learning rates. 
    A detailed explanation for this can be found in Appendix~\ref{apx:proof_example_river_valley}. 
    This gap also highlights the importance of warmup when the training doesn't have a good initialization ($\Delta_0$ is large).
\end{example}

Therefore, we can provide a theoretical explanation for the empirical advantages of learning rate warmup~\citep{kalra2024warmup, gilmer2022loss} based on the theorems. 
At the beginning of training, 
the initialization may be poor, leading to a large function gap $\Delta_0$ and high local smoothness $\Norm{\nabla^2 f(\bw_0)}\leq K_0+K_1\Delta_0^\rho$. Thus, the initial learning rate of a constant (or non-increasing) learning rate schedule must be sufficiently small ($\mathcal{O}(\Kr^{-1}\Delta_0^{-\rho})$), to prevent oscillation or divergence. As training progresses, the function gap $\Delta_t$ decreases, leading to a reduction in $\Norm{\nabla^2 f(\bw_t)}$, which in turn allows a larger learning rate. Learning rate warmup accelerates this process, getting the model into regions of the loss landscape with lower local smoothness more quickly. Moreover, it enables the use of larger learning rates after entering the smooth regions, which is often denoted as the target learning rate in the warm-up strategy.
Our theory shows that this acceleration can be significant.

\subsection{Empirical Validation of the Theoretical Schedule}\label{sec:validation_schedule_empirical}

In this section, we provide some empirical evidence to support our claim that the theoretically derived schedule $\{ \eta_t \}$ in Theorem~\ref{thm:GD} is a valid representation of the warmup schedules. 

\paragraph{Synthetic Experiment.} We do an empirical validation of the simple synthetic river-valley minimization problem described in Example~\ref{example:river_valley}. In the specific experimental setup, we set $\Delta_0 = 1000$ and $K_1 = 10$, $x_0 = \frac{\log(\Delta_0)}{\sqrt{K_1}}$ and $y_0 = 2$. We consider 3 schedules for comparison: constant, our theoretical warmup described in Theorem~\ref{thm:GD}, and linear warmup. To ensure a fair comparison, we carefully tune the learning rate scale for all schedules, i.e., we find an optimal constant to multiply the learning rate to achieve the fastest convergence (without leading to divergence). The results are shown in Figure~\ref{fig:synthetic-comparison}, where we can observe that the theoretical warmup schedule and linear warmup schedule achieve a similar significant acceleration in loss convergence compared to the constant schedule. The learning rate dynamics figure also shows that both warmup schedules enable a much larger stable learning rate compared to the constant schedule without warmup.

\begin{figure}[t]
\vspace{-0.5cm}
    \centering
    \subfigure[loss convergence]{\includegraphics[width=0.45\textwidth]{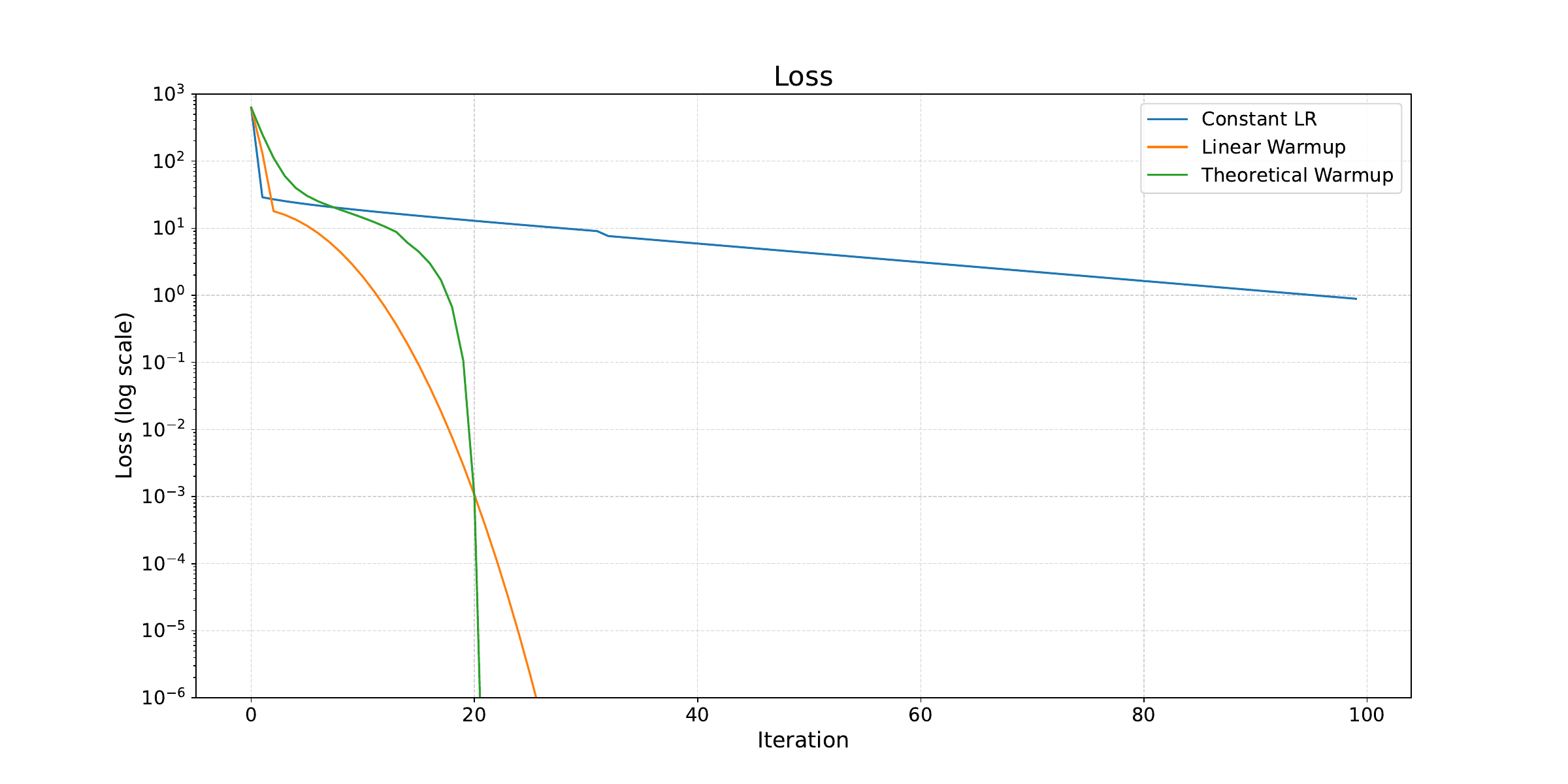}}\hfill
    \subfigure[lr shcedule in the training]{\includegraphics[width=0.45\textwidth]{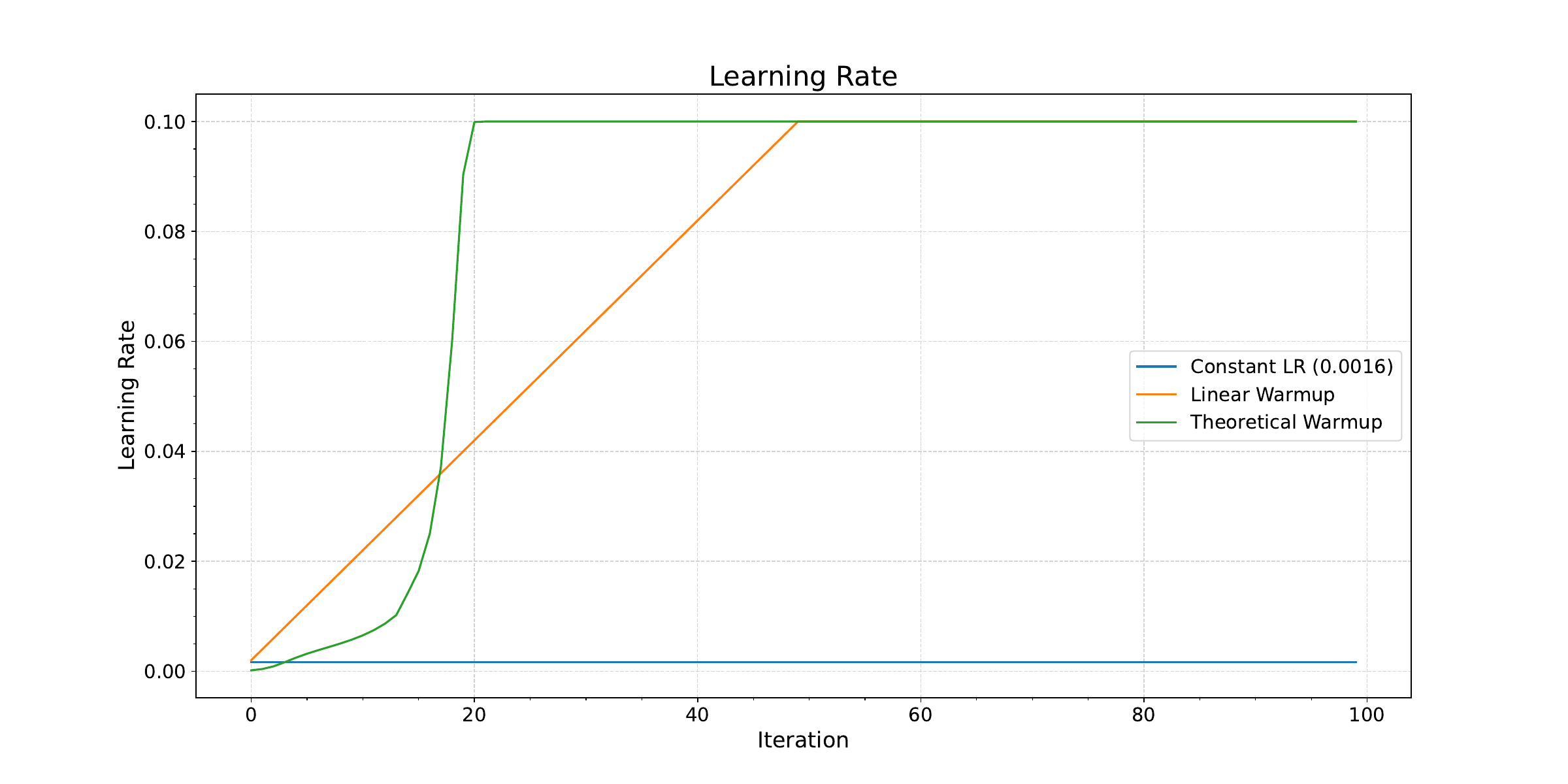}}\hfill
  \vspace{-0.3cm}
  \caption{An empirical experiment based on the synthetic problem setting in Example~\ref{example:river_valley}. The loss convergence curves are on the left side, and the learning rate dynamics are on the right side.}
  % \vspace{-0.5cm}
  \label{fig:synthetic-comparison}
\end{figure}

\paragraph{ResNet on CIFAR.} We train a ResNet18 on CIFAR-10 for 100 epochs with SGD (without momentum), using the linear warmup schedule, no warmup (constant) schedule, and our theoretical schedule $\eta_t= \frac{1}{4\sqrt{2}+4}\min\left\{\frac{1}{K_0}, \frac{1}{3^\rho \Kr}\Delta_t^{-\rho}\right\}$. Following common practice, we set the first $10$ epochs as the warm-up phase, and we do cosine decay after this initial phase. We set $\rho=1$ for the theoretical schedule, and tune $K_0 \in\{1,4,8,16\}$ and $K_\rho \in\{1 / 4,1 / 2,1,4\}$. For the constant and linear warmup schedules, we set the target learning rate to be the same as the target learning rate of the theoretical schedule, i.e., $\frac{1}{4\sqrt{2}+4} \frac{1}{K_0}$, to ensure a fair comparison.

As displayed in Figure~\ref{fig:lr-schedule-comparison}, we can observe that the theoretical schedule increases similarly to the linear warmup schedule, but is steeper in the first place, making a more concave curve. Also, since we use mini-batch gradients instead of full gradients, the loss is not monotonically decreasing, so there is some small oscillation before the schedule reaches a plateau. This learning rate schedule shows a valid warmup phase and achieves the plateau faster than the linear warmup schedule.
Moreover, we list the performance in Table~\ref{tab:accuracy_results_resnet}, showing that the theoretical schedule achieves even better performance than linear warmup, outperforming the constant schedule. Therefore, the theoretical schedule employed in Theorem~\ref{thm:GD} can be considered a valid representative of the warmup schedules, and thus, our theory built on this schedule does present the benefits of doing warmup.

\begin{table}[h]
    \centering
    \begin{tabular}{lccc}
    \toprule
    & Theoretical Warmup & Linear Warmup & No Warmup (Constant) \\
    \midrule
    Test Epoch Accuracy & $0.8589 \pm 0.0023$ & $0.8577 \pm 0.0023$ & $0.8562 \pm 0.0019$ \\
    \bottomrule
    \end{tabular}
    \caption{The test epoch accuracy after $100$ epochs of training with different warm-up schedules. The results are reported as mean $\pm$ standard deviation over $6$ independent runs.}
    \label{tab:accuracy_results_resnet}
\end{table}

\begin{figure}[t]
\vspace{-0.5cm}
    \centering
    \subfigure[lr schedule in warmup period]{\includegraphics[width=0.45\textwidth]{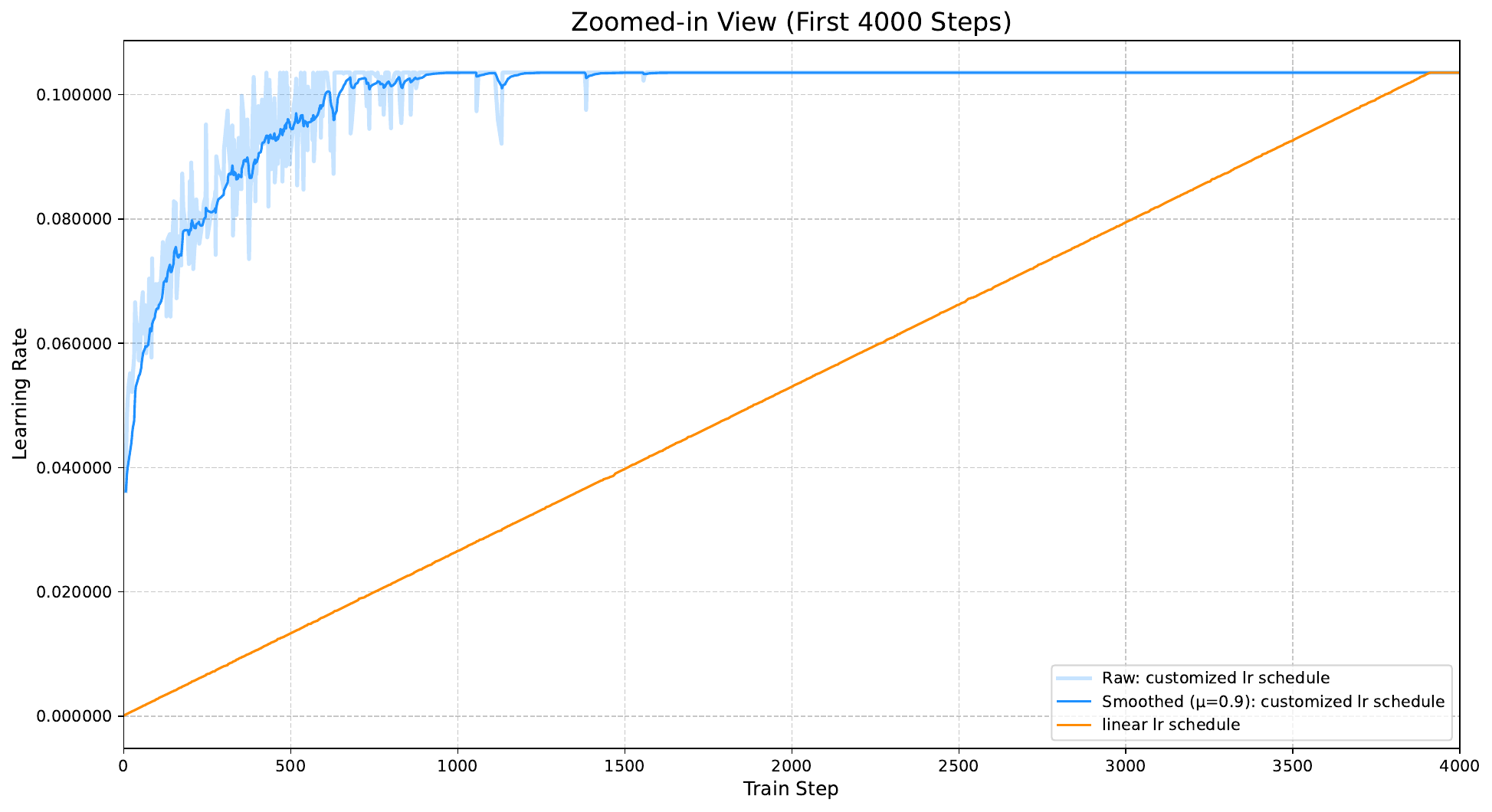}}\hfill
    \subfigure[lr shcedule in the whole training]{\includegraphics[width=0.45\textwidth]{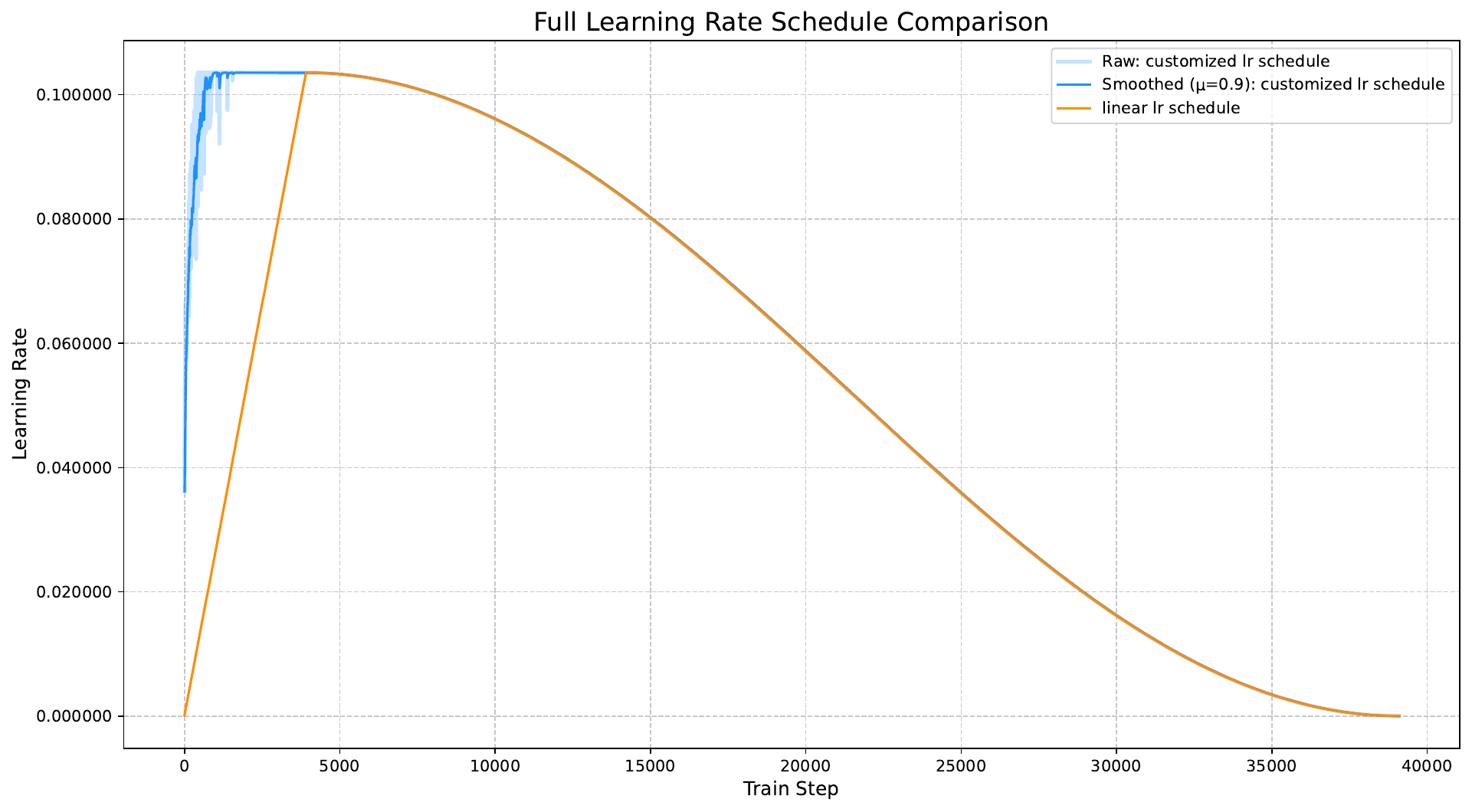}}\hfill
  \vspace{-0.3cm}
  \caption{A comparison between warmup learning rate schedules in ResNet training. The blue line is the theoretical warmup schedule derived in Theorem~\ref{thm:GD}, and the yellow line is the standard linear warmup. We do smoothing for the blue line in the plot to make it clearer.}
  % \vspace{-0.5cm}
  \label{fig:lr-schedule-comparison}
\end{figure}

\subsection{Lower Bound of GD}
In this section, we consider the lower bound of GD with non-increasing learning rate schedules under the special case of $\rho=1$.

\begin{theorem}\label{thm:lower_bound_gd}
    Given $K_1,\epsilon$ as the desired accuracy, $\Delta$ as the initial loss gap, for GD with any non-increasing learning rate sequence $\{\eta_t\}$, there exists a function $f:\RR \to \RR$ that is $(1,\epsilon \sqrt{K_1},4\pi K_1)$-smooth, lower bounded by $f^*$ and $f(\bw_0)-f^\star\leq8\Delta$, such that GD needs at least
    \begin{align*}
        \Omega \left( \frac{K_1 \Delta^2}{\epsilon^2} \right)
    \end{align*}
    iterations to achieve $\Norm{\nabla f(\bw)} \le \epsilon$.
\end{theorem}

The lower bound matches the upper bound result for constant learning rates in Theorem~\ref{thm:GD}, implying the tightness of the bounds. 
The proof of Theorem~\ref{thm:lower_bound_gd} is presented in Appendix~\ref{apx:proof_gd_lower_bound}. The specific lower bound construction is based on the use of trigonometric functions, which satisfy the $(1, K_0, K_1)$-smoothness but cannot be adopted by the $(\rho,L_0,L_\rho)$-smoothness assumptions for any $\rho>0$ as noted in Example~\ref{example:trigonometric}.

Note that compared to the lower bound for $(1,L_0,L_1)$-smoothness presented in \cite{zhang2020gradient}, 
Theorem~\ref{thm:lower_bound_gd} is more general since it allows general non-increasing learning rate schedules rather than only constant learning rates. This requires novel construction and proof techniques and may be of interest for future study on lower bounds.
Also, Theorem~\ref{thm:lower_bound_gd} does not have additional logarithmic terms in the lower bound, which serves as evidence for the fact that $(1,K_0,K_1)$-smoothness is strictly weaker than $(1,L_0,L_1)$-smoothness and more difficult to optimize.

\section{Theory of SGD}\label{sec:SGD}
In this section, we analyze SGD for $(\rho,K_0,\Kr)$-smooth functions:
$$\bw_{t+1}=\bw_t-\eta_t\bg_t,$$
where $\bg_t=\nabla f(\bw_t,\xi_t)$. We first make the following assumptions on the noise.

\begin{assumption}\label{assumption:a.s.bounded}
    $\EE_\xi \nabla f(\bw,\xi)=\nabla f(\bw)$ and $\Norm{\nabla f(\bw,\xi)-\nabla f(\bw)}\leq \sigma$ for some $\sigma>0$ and all $\bw\in\mathbb{R}^d$, with probability $1$.
\end{assumption}

\begin{assumption}\label{assumption:a.s.ABC}
    $\EE_\xi \nabla f(\bw,\xi)=\nabla f(\bw)$ and $\Norm{\nabla f(\bw,\xi)-\nabla f(\bw)}^2\leq A\paren{f(\bw)-f^\star}+B\Norm{\nabla f(\bw)}^2+ \sigma^2$ for some $A\geq0,B\geq 0, \sigma>0$ and all $\bw\in\mathbb{R}^d$, with probability $1$.
\end{assumption}

Assumption~\ref{assumption:a.s.bounded} is commonly used in stochastic optimization, especially under generalized smoothness settings~\citep{zhang2020improved, zhang2020gradient, li2023convergence, crawshaw2022robustness}, for example $(L_0,L_1)$-smoothness.

Assuption~\ref{assumption:a.s.ABC} was originally introduced by~\cite{khaled2023better} in the expectation form, and is a more general noise assumption compared with the bounded variance assumption or the relaxed growth condition~\citep{bottou2018optimization}.
It covers a wide range of randomness sources, such as subsampling~\citep{gower2019sgd} and gradient compression~\citep{alistarh2017qsgd, khirirat2018distributed}, which may not be captured by the bounded variance assumption. We adopt this assumption in our theoretical analysis to consider more general settings, and also explore how learning rate warmup can help with convergence when the noise term is hard to handle.
Also note that the assumptions we considered here can be relaxed to the sub-gaussian noise assumptions, with all the theorems still valid up to some logarithmic terms, and we include the corresponding details in Appendix~\ref{apx:subgaussian}.

\subsection{Bounded Noise}\label{sec:SGD_bounded}

In this section, we use the following notations 
\begin{align*}
    r_t\triangleq \min\left\{C_1\Delta_t^{-\frac{\rho-1}{2}},C_2\right\}, \quad L_t \triangleq 2K_0+\Kr\paren{2\Delta_t}^\rho, \quad \text{and} \quad G_t\triangleq\sqrt{K_0\Delta_t+\Kr3^\rho\Delta_t^{\rho+1}}
\end{align*}
to simplify the increasing learning rates in the theorems. 
For Theorem~\ref{thm:SGD_constant_a.s.bounded}, we simply replace $\Delta_t$ in $r_t,L_t,G_t$ with $4\Delta_0$ to obtain $r,L,G$, and use them to simplify the constant learning rate.

\begin{theorem}\label{thm:SGD_adaptive_a.s.bounded}
    Suppose Assumptions~\ref{assumption:lower_bounded}, \ref{assumption:K_smooth} and \ref{assumption:a.s.bounded} hold with $\rho\geq 1$. $\{\bw_t\}$ is generated by SGD.
    Let 
    $\eta_t=\min \left\{\frac{1}{8(\sqrt{2}+1) K_0}, \frac{1}{8(\sqrt{2}+1) K_\rho\left(3 \Delta_t\right)^\rho}, \frac{r_t}{2 \sigma}, \sqrt{\frac{\Delta_0}{\sigma^2 T L_t}}, \frac{\Delta_0}{2 \sigma G_t \sqrt{T \log \frac{1}{\delta}}}\right\}.$
    Then with probability at least $1-\delta$,
    it holds that $\Delta_t\leq 4\Delta_0,\forall t\in[T]$ and 
    \begin{align*}
        \min_{t<T}\Norm{\nabla f(\bw_t)}^2&\leq \mathcal{O}\left(\frac{\Delta_0}{T}\paren{K_0+\Kr\Delta_{avg,\rho}}+\sigma\frac{\Delta_0}{T}\paren{C_1^{-1}\Delta_{avg,\frac{\rho-1}{2}}+C_2^{-1}}\right)\\
        &+ \mathcal{O}\paren{\sigma\sqrt{\frac{\Delta_0\log\frac{1}{\delta}}{T}}\paren{K_0+\Kr\Delta_{avg,\rho}}^{1/2}},
    \end{align*}
    where $\Delta_{avg,\rho}=\sum_{t=0}^{T-1}\Delta_t^\rho/T$.
    \end{theorem}

\begin{theorem}\label{thm:SGD_constant_a.s.bounded}
    Under the same assumptions as Theorem~\ref{thm:SGD_constant_a.s.bounded}, if we use a constant learning rate 
    $\eta_t \equiv \eta=\min \left\{\frac{1}{8(\sqrt{2}+1) K_0}, \frac{1}{8(\sqrt{2}+1) K_\rho\left(12 \Delta_0\right)^\rho}, \frac{r}{2 \sigma}, \sqrt{\frac{\Delta_0}{\sigma^2 T L}}, \frac{\Delta_0}{\sigma G \sqrt{2 T \log \frac{1}{\delta}}}, \frac{\Delta_0}{\sigma \alpha}\right\}$, 
    where $\alpha=\left(G+L C_2\right)\left(1+\sqrt{2 \log \frac{1}{\delta}}\right)$.
    Then with probability at least $1-\delta$, it holds that 
    $\Delta_t\leq 4\Delta_0,\forall t\in[T]$ and 
    \begin{align*}
        \min_{t<T}\Norm{\nabla f(\bw_t)}^2&\leq \mathcal{O}\left(\frac{\Delta_0}{T}\paren{K_0+\Kr\Delta_{0}^\rho}+\sigma\frac{\Delta_0}{T}\paren{C_1^{-1}\Delta_{0}^{\frac{\rho-1}{2}}+C_2^{-1}}\right)\\
        &+ \mathcal{O}\paren{\sigma\sqrt{\frac{\Delta_0\log\frac{1}{\delta}}{T}}\paren{K_0+\Kr\Delta_{0}^{\rho}}^{1/2}+\sigma\frac{\sqrt{\log\frac{1}{\delta}}}{T}C_2\paren{K_0+\Kr\Delta_0^\rho}}.
    \end{align*}
\end{theorem}

We can see that in the stochastic settings, both constant learning rates and the learning rate schedule adapted to $\Delta_t$ achieve the $\cO(1/\sqrt{T})$ convergence rate with high probability, which matches the optimal rate~\citep{arjevani2023lower}.
However, it is not hard to see that
to ensure convergence, we take the adaptive learning rate $\eta_t=g(\Delta_t)$ and the constant learning rate $\eta = g(4\Delta_0)$ in the same pattern, where $g$ is a monotonically decreasing function. Therefore, we have $\eta_t\geq\eta$ and $\Delta_{avg,\rho} \leq \paren{4\Delta_0}^\rho$, and the terms introduced by $K_\rho$ in the convergence rates are also correspondingly related to $\Delta_{avg,\rho}$ and $\paren{4\Delta_0}^\rho$, indicating that the convergence rate in Theorem~\ref{thm:SGD_adaptive_a.s.bounded} is better than that in Theorem~\ref{thm:SGD_constant_a.s.bounded}. 
We also note that $\Delta_{avg,\rho}$ can be significantly smaller than $\paren{4\Delta_0}^\rho$, which is reflected in the following convex example. 

\begin{example}
    Consider the case that $f$ is convex and noise is dominant in the convergence, then $\Delta_t$ is generally in the order of 
    $\mathcal{O}(1/\sqrt{t})$ following a similar analysis through the combination of \cite{liu2023revisiting} and Theorem~\ref{thm:GD_convex}. 
    Then $\sum_{t=0}^{T-1}\Delta_t^\rho =\mathcal{O}\paren{\log T} $ if $\rho=2$ and ${\mathcal{O}}\paren{1}$ if $\rho>2$. In this case, $\Delta_{\text {avg, } \rho}$ can be improved over $\Delta_0^\rho$ up to a factor of $\mathcal{O}(T)$, resulting in a $\Theta(\sqrt{T})$ times smaller convergence rate.
\end{example}

In the stochastic setting, we cannot guarantee that $\Delta_t$ decreases at every step. Nevertheless, in practice, as long as SGD does not diverge, $\Delta_t$ typically shows a decreasing trend. Hence, $\eta_t$ presented in Theorem~\ref{thm:SGD_adaptive_a.s.bounded} is approximately increasing and can be interpreted as a specific adaptive learning rate warmup strategy, which is also verified in Section~\ref{sec:validation_schedule_empirical}. Moreover, as training progresses, local smoothness also decreases with the decrease of $\Delta_t$, allowing the model to reach flatter regions of the loss landscape and use larger learning rates afterwards, just like the deterministic case.

\subsection{ABC-Inequality}\label{sec:SGD_ABC}

We use the same notations $r_t,L_t,G_t$ as in Section~\ref{sec:SGD_bounded} to simplify the increasing learning rate in Theorem~\ref{thm:SGD_adaptive_a.s.ABC}. We replace $\Delta_t$ in $r_t,L_t,G_t$ with $8\Delta_0$ to obtain $r,L,G$ and use them to simplify the constant learning rate in Theorem~\ref{thm:SGD_constant_a.s.ABC}.

\begin{theorem}\label{thm:SGD_adaptive_a.s.ABC}
    Suppose Assumptions~\ref{assumption:lower_bounded}, \ref{assumption:K_smooth} and \ref{assumption:a.s.ABC} hold with $\rho\geq 1$. $\{\bw_t\}$ is generated by SGD. 
    If the learning rate $\eta_t$ be adapted to $\Delta_t$ as described in \eqref{eq:learning_rate_thm_ABC_adaptive},
    then with probability at least $1-\delta$, it holds that 
    $\Delta_t\leq 4\Delta_0,\forall t\in[T]$ and
    \begin{align*}
        &\min_{t<T} \Norm{\nabla f(\bw_t)}^2\\
        \leq&
        \mathcal{O}\paren{\frac{\Delta_0}{T}\paren{K_0+\Kr\Delta_{avg,\rho}}+\sigma\frac{\Delta_0}{T}\paren{C_1^{-1}\Delta_{avg,\frac{\rho-1}{2}}+C_2^{-1}}
        +\frac{\Delta_0}{T}\sqrt{A}\paren{\sqrt{K_0}+\sqrt{\Kr\Delta_{avg,\rho}}}}\\
        +&\mathcal{O}\paren{\sqrt{\frac{\Delta_0\log\frac{1}{\delta}}{T}}\paren{ \paren{\sigma+\sqrt{A\Delta_0}}\paren{K_0+\Kr\Delta_{avg,\rho}}^{1/2}+\sqrt{B\Delta_0}\paren{K_0+\Kr\Delta_{avg,\rho}}}},
    \end{align*}
    where $\Delta_{avg,\rho}=\sum_{t=0}^{T-1}\Delta_t^\rho/T$.
\end{theorem}

\begin{theorem}\label{thm:SGD_constant_a.s.ABC}
    Under the same assumptions as Theorem~\ref{thm:SGD_adaptive_a.s.ABC}, if we use a constant learning rate as in \eqref{eq:learning_rate_thm_ABC_constant}, 
    then with probability at least $1-\delta$, it holds that $\Delta_t\leq 8\Delta_0,\forall t\in[T]$ and
    $$
    \begin{aligned}
    &\min_{t<T} \Norm{\nabla f(\bw_t)}^2 \\
    &\leq \mathcal{O}\left(\frac{\Delta_0}{T}\left(K_0+K_\rho \Delta_0^\rho\right)+\sigma \frac{\Delta_0}{T}\left(C_1^{-1} \Delta_0^{\frac{\rho-1}{2}}+C_2^{-1}\right)+\frac{\Delta_0}{T} \sqrt{A}\left(\sqrt{K_0}+\sqrt{K_\rho \Delta_0^\rho}\right)\right) \\
    & +\mathcal{O}\left(\sqrt{\frac{\Delta_0 \log \frac{1}{\delta}}{T}}\left(\left(\sigma+\sqrt{A \Delta_0}\right)\left(K_0+K_\rho \Delta_0^\rho\right)^{1 / 2}+\sqrt{B \Delta_0}\left(K_0+K_\rho \Delta_0^\rho\right)\right)\right).
    \end{aligned}
    $$
    
\end{theorem}

Due to space limitations, we present only the main terms in Theorem~\ref{thm:SGD_constant_a.s.ABC} and the complete result, together with the learning rate choices, are presented in Appendix~\ref{apx:pf_ABC}.
Note that if $A = B = 0$, then Theorem~\ref{thm:SGD_adaptive_a.s.ABC} and Theorem~\ref{thm:SGD_constant_a.s.ABC} cover the results of Theorem~\ref{thm:SGD_adaptive_a.s.bounded} and Theorem~\ref{thm:SGD_constant_a.s.bounded}, respectively. Similar to the discussion in Section~\ref{sec:SGD_bounded}, since $\Delta_t$ should be generally decreasing in the training process, $\eta_t$ is approximately increasing and can be regarded as a specific learning rate warmup strategy.

Moreover, as we can see from the two convergence rates, the extra gradient noise in Assumption~\ref{assumption:a.s.ABC} introduces even more benefits of warmup. As we noted in the previous examples, $\Delta_{avg,\rho}$ can be significantly smaller than $\Delta_0^\rho$. Thus, by comparing the results in Theorem~\ref{thm:SGD_adaptive_a.s.ABC} and \ref{thm:SGD_constant_a.s.ABC}, we notice that the specific warmup schedule can reduce the dependence of convergence rates on both $A$ and $B$, which further demonstrates that learning rate warmup may be beneficial not only when the local smoothness is largely varying over the landscape, but also when the gradient noise is large and related to the landscape.

\section{Conclusion}
The paper investigates a theoretical explanation for the benefits of the learning rate warmup strategy. We proposed a novel family of generalized smoothness assumptions to better describe the local smoothness variation in the training process. Then, under the novel smoothness assumptions, we proved that GD and SGD can both benefit from the warmup strategy, showing potentially a $\Theta(T)$ times acceleration for the deterministic setting and $\Theta(\sqrt{T})$ times for the stochastic settings in convergence speed over using only a constant or non-increasing learning rate schedule. Moreover, when a more general noise assumption is considered, we show that warmup can also be beneficial in handling the extra noise terms, further highlighting the importance of doing warmup.
A limitation of this work is that the analysis only applies to SGD, but not to SGD with momentum or Adam, which are generally more popular in practical tasks.
{However, we believe that our analysis can be extended to these optimizers, since the described benefits of warm-up in this paper arise mainly from our generalized smoothness assumptions, which are independent of the choice of an optimizer.}
Also, the lower bound results currently only apply to the $(1,K_0,K_1)$-smoothness setting, which may not be general enough. These are potentially interesting future topics.

\bibliographystyle{plainnat}
\bibliography{ref}

\begin{thebibliography}{37}
\providecommand{\natexlab}[1]{#1}
\providecommand{\url}[1]{\texttt{#1}}
\expandafter\ifx\csname urlstyle\endcsname\relax
  \providecommand{\doi}[1]{doi: #1}\else
  \providecommand{\doi}{doi: \begingroup \urlstyle{rm}\Url}\fi

\bibitem[Alistarh et~al.(2017)Alistarh, Grubic, Li, Tomioka, and Vojnovic]{alistarh2017qsgd}
Dan Alistarh, Demjan Grubic, Jerry Li, Ryota Tomioka, and Milan Vojnovic.
\newblock Qsgd: Communication-{Efficient} {SGD} via {Gradient} {Quantization} and {Encoding}.
\newblock In \emph{Conference on {Neural} {Information} {Processing} {Systems} ({NeurIPS})}, pages 1709--1720, 2017.

\bibitem[Arjevani et~al.(2023)Arjevani, Carmon, Duchi, Foster, Srebro, and Woodworth]{arjevani2023lower}
Yossi Arjevani, Yair Carmon, John~C. Duchi, Dylan~J. Foster, Nathan Srebro, and Blake Woodworth.
\newblock Lower bounds for non-convex stochastic optimization.
\newblock \emph{Mathematical Programming}, 199\penalty0 (1-2):\penalty0 165--214, 2023.

\bibitem[Bottou et~al.(2018)Bottou, Curtis, and Nocedal]{bottou2018optimization}
L{\' e}on Bottou, Frank~E. Curtis, and Jorge Nocedal.
\newblock Optimization {Methods} for {Large}-{Scale} {Machine} {Learning}.
\newblock \emph{SIAM Review}, 60\penalty0 (2):\penalty0 223--311, 2018.

\bibitem[Boyd and Vandenberghe(2004)]{boyd2004convex}
Stephen~P Boyd and Lieven Vandenberghe.
\newblock \emph{Convex optimization}.
\newblock Cambridge university press, 2004.

\bibitem[Cohen et~al.(2021)Cohen, Kaur, Li, Kolter, and Talwalkar]{cohen2021gradient}
Jeremy Cohen, Simran Kaur, Yuanzhi Li, J~Zico Kolter, and Ameet Talwalkar.
\newblock Gradient descent on neural networks typically occurs at the edge of stability.
\newblock In \emph{International Conference on Learning Representations}, 2021.
\newblock URL \url{https://openreview.net/forum?id=jh-rTtvkGeM}.

\bibitem[Crawshaw et~al.(2022)Crawshaw, Liu, Orabona, Zhang, and Zhuang]{crawshaw2022robustness}
Michael Crawshaw, Mingrui Liu, Francesco Orabona, Wei Zhang, and Zhenxun Zhuang.
\newblock Robustness to unbounded smoothness of generalized signsgd.
\newblock \emph{Advances in neural information processing systems}, 35:\penalty0 9955--9968, 2022.

\bibitem[Gilmer et~al.(2022)Gilmer, Ghorbani, Garg, Kudugunta, Neyshabur, Cardoze, Dahl, Nado, and Firat]{gilmer2022loss}
Justin Gilmer, Behrooz Ghorbani, Ankush Garg, Sneha Kudugunta, Behnam Neyshabur, David Cardoze, George~Edward Dahl, Zachary Nado, and Orhan Firat.
\newblock A loss curvature perspective on training instabilities of deep learning models.
\newblock In \emph{International Conference on Learning Representations}, 2022.

\bibitem[Gotmare et~al.(2018)Gotmare, Keskar, Xiong, and Socher]{gotmare2018closer}
Akhilesh Gotmare, Nitish~Shirish Keskar, Caiming Xiong, and Richard Socher.
\newblock A closer look at deep learning heuristics: Learning rate restarts, warmup and distillation.
\newblock \emph{arXiv preprint arXiv:1810.13243}, 2018.

\bibitem[Gotmare et~al.(2019)Gotmare, Keskar, Xiong, and Socher]{gotmare2018a}
Akhilesh Gotmare, Nitish~Shirish Keskar, Caiming Xiong, and Richard Socher.
\newblock A closer look at deep learning heuristics: Learning rate restarts, warmup and distillation.
\newblock In \emph{International Conference on Learning Representations}, 2019.
\newblock URL \url{https://openreview.net/forum?id=r14EOsCqKX}.

\bibitem[Gower et~al.(2019)Gower, Loizou, Qian, Sailanbayev, Shulgin, and Richt{\' a}rik]{gower2019sgd}
Robert~Mansel Gower, Nicolas Loizou, Xun Qian, Alibek Sailanbayev, Egor Shulgin, and Peter Richt{\' a}rik.
\newblock Sgd: General {Analysis} and {Improved} {Rates}.
\newblock In \emph{International {Conference} on {Machine} {Learning}}, volume abs/1901.09401, 2019.

\bibitem[Goyal et~al.(2017)Goyal, Doll{\'a}r, Girshick, Noordhuis, Wesolowski, Kyrola, Tulloch, Jia, and He]{goyal2017accurate}
Priya Goyal, Piotr Doll{\'a}r, Ross Girshick, Pieter Noordhuis, Lukasz Wesolowski, Aapo Kyrola, Andrew Tulloch, Yangqing Jia, and Kaiming He.
\newblock Accurate, large minibatch sgd: Training imagenet in 1 hour.
\newblock \emph{arXiv preprint arXiv:1706.02677}, 2017.

\bibitem[He et~al.(2016)He, Zhang, Ren, and Sun]{he2016deep}
Kaiming He, Xiangyu Zhang, Shaoqing Ren, and Jian Sun.
\newblock Deep residual learning for image recognition.
\newblock In \emph{Proceedings of the IEEE conference on computer vision and pattern recognition}, pages 770--778, 2016.

\bibitem[Jastrzębski et~al.(2018)Jastrzębski, Kenton, Arpit, Ballas, Fischer, Storkey, and Bengio]{jastrzębski2018three}
Stanisław Jastrzębski, Zac Kenton, Devansh Arpit, Nicolas Ballas, Asja Fischer, Amos Storkey, and Yoshua Bengio.
\newblock Three factors influencing minima in {SGD}, 2018.
\newblock URL \url{https://openreview.net/forum?id=rJma2bZCW}.

\bibitem[Kalra and Barkeshli(2024)]{kalra2024warmup}
Dayal~Singh Kalra and Maissam Barkeshli.
\newblock Why warmup the learning rate? underlying mechanisms and improvements.
\newblock \emph{Advances in Neural Information Processing Systems}, 37:\penalty0 111760--111801, 2024.

\bibitem[Khaled and Richt{\' a}rik(2023)]{khaled2023better}
Ahmed Khaled and Peter Richt{\' a}rik.
\newblock Better {Theory} for {SGD} in the {Nonconvex} {World}.
\newblock \emph{Transactions on Machine Learning Research (TMLR)}, 2023, 2023.

\bibitem[Khirirat et~al.(2018)Khirirat, Feyzmahdavian, and Johansson]{khirirat2018distributed}
Sarit Khirirat, Hamid~Reza Feyzmahdavian, and Mikael Johansson.
\newblock Distributed learning with compressed gradients.
\newblock \emph{arXiv preprint arXiv:1806.06573}, 2018.

\bibitem[Kingma and Ba(2014)]{kingma2014adam}
Diederik~P Kingma and Jimmy Ba.
\newblock Adam: A method for stochastic optimization.
\newblock \emph{arXiv preprint arXiv:1412.6980}, 2014.

\bibitem[Kosson et~al.(2024)Kosson, Messmer, and Jaggi]{kosson2024analyzing}
Atli Kosson, Bettina Messmer, and Martin Jaggi.
\newblock Analyzing \& reducing the need for learning rate warmup in gpt training.
\newblock \emph{Advances in Neural Information Processing Systems}, 37:\penalty0 2914--2942, 2024.

\bibitem[Li et~al.(2023{\natexlab{a}})Li, Qian, Tian, Rakhlin, and Jadbabaie]{li2023convex}
Haochuan Li, Jian Qian, Yi~Tian, Alexander Rakhlin, and Ali Jadbabaie.
\newblock Convex and {Non}-convex {Optimization} {Under} {Generalized} {Smoothness}.
\newblock In \emph{Conference on {Neural} {Information} {Processing} {Systems} ({NeurIPS})}, 2023{\natexlab{a}}.

\bibitem[Li et~al.(2023{\natexlab{b}})Li, Rakhlin, and Jadbabaie]{li2023convergence}
Haochuan Li, Alexander Rakhlin, and Ali Jadbabaie.
\newblock Convergence of {Adam} {Under} {Relaxed} {Assumptions}.
\newblock In \emph{Conference on {Neural} {Information} {Processing} {Systems} ({NeurIPS})}, 2023{\natexlab{b}}.

\bibitem[Liu et~al.(2020)Liu, Jiang, He, Chen, Liu, Gao, and Han]{liu2020variance}
Liyuan Liu, Haoming Jiang, Pengcheng He, Weizhu Chen, Xiaodong Liu, Jianfeng Gao, and Jiawei Han.
\newblock On the {Variance} of the {Adaptive} {Learning} {Rate} and {Beyond}.
\newblock In \emph{International {Conference} on {Learning} {Representations} ({ICLR})}, 2020.

\bibitem[Liu et~al.(2024)Liu, Pan, and Zhang]{liu2024adagrad}
Yuxing Liu, Rui Pan, and Tong Zhang.
\newblock Adagrad under anisotropic smoothness.
\newblock \emph{arXiv preprint arXiv:2406.15244}, 2024.

\bibitem[Liu and Zhou(2023)]{liu2023revisiting}
Zijian Liu and Zhengyuan Zhou.
\newblock Revisiting the last-iterate convergence of stochastic gradient methods.
\newblock \emph{arXiv preprint arXiv:2312.08531}, 2023.

\bibitem[Loshchilov and Hutter(2017)]{loshchilov2017sgdr}
Ilya Loshchilov and Frank Hutter.
\newblock {SGDR}: Stochastic gradient descent with warm restarts.
\newblock In \emph{International Conference on Learning Representations}, 2017.
\newblock URL \url{https://openreview.net/forum?id=Skq89Scxx}.

\bibitem[Malitsky and Mishchenko(2020)]{malitsky2020adaptive}
Yura Malitsky and Konstantin Mishchenko.
\newblock Adaptive {Gradient} {Descent} without {Descent}.
\newblock In \emph{International {Conference} on {Machine} {Learning} ({ICML})}, pages 6702--6712, 2020.

\bibitem[Nesterov et~al.(2018)]{nesterov2018lectures}
Yurii Nesterov et~al.
\newblock \emph{Lectures on convex optimization}, volume 137.
\newblock Springer, 2018.

\bibitem[Patel et~al.(2022)Patel, Zhang, and Tian]{patel2022global}
Vivak Patel, Shushu Zhang, and Bowen Tian.
\newblock Global {Convergence} and {Stability} of {Stochastic} {Gradient} {Descent}.
\newblock In \emph{Conference on {Neural} {Information} {Processing} {Systems} ({NeurIPS})}, 2022.

\bibitem[Qian et~al.(2021)Qian, Wu, Zhuang, Wang, and Xiao]{qian2021understanding}
Jiang Qian, Yuren Wu, Bojin Zhuang, Shaojun Wang, and Jing Xiao.
\newblock Understanding {Gradient} {Clipping} {In} {Incremental} {Gradient} {Methods}.
\newblock In \emph{International {Conference} on {Artificial} {Intelligence} and {Statistics} ({AISTATS})}, pages 1504--1512, 2021.

\bibitem[Smith et~al.(2018)Smith, Kindermans, and Le]{l.2018dont}
Samuel~L. Smith, Pieter-Jan Kindermans, and Quoc~V. Le.
\newblock Don't decay the learning rate, increase the batch size.
\newblock In \emph{International Conference on Learning Representations}, 2018.
\newblock URL \url{https://openreview.net/forum?id=B1Yy1BxCZ}.

\bibitem[Teboulle and Vaisbourd(2023)]{teboulle2023elementary}
Marc Teboulle and Yakov Vaisbourd.
\newblock An elementary approach to tight worst case complexity analysis of gradient based methods.
\newblock \emph{Mathematical Programming}, 201\penalty0 (1):\penalty0 63--96, 2023.

\bibitem[Touvron et~al.(2023)Touvron, Martin, Stone, Albert, Almahairi, Babaei, Bashlykov, Batra, Bhargava, Bhosale, et~al.]{touvron2023llama}
Hugo Touvron, Louis Martin, Kevin Stone, Peter Albert, Amjad Almahairi, Yasmine Babaei, Nikolay Bashlykov, Soumya Batra, Prajjwal Bhargava, Shruti Bhosale, et~al.
\newblock Llama 2: Open foundation and fine-tuned chat models.
\newblock \emph{arXiv preprint arXiv:2307.09288}, 2023.

\bibitem[Tyurin(2025)]{tyurin2025toward}
Alexander Tyurin.
\newblock Toward a {Unified} {Theory} of {Gradient} {Descent} under {Generalized} {Smoothness}.
\newblock In \emph{Forty-second {International} {Conference} on {Machine} {Learning}}, 2025.

\bibitem[Vaswani et~al.(2017)Vaswani, Shazeer, Parmar, Uszkoreit, Jones, Gomez, Kaiser, and Polosukhin]{vaswani2017attention}
Ashish Vaswani, Noam Shazeer, Niki Parmar, Jakob Uszkoreit, Llion Jones, Aidan~N Gomez, {\L}ukasz Kaiser, and Illia Polosukhin.
\newblock Attention is all you need.
\newblock \emph{Advances in neural information processing systems}, 30, 2017.

\bibitem[Wen et~al.(2024)Wen, Li, Wang, Hall, Liang, and Ma]{wen2024understanding}
Kaiyue Wen, Zhiyuan Li, Jason Wang, David Hall, Percy Liang, and Tengyu Ma.
\newblock Understanding warmup-stable-decay learning rates: A river valley loss landscape perspective.
\newblock \emph{arXiv preprint arXiv:2410.05192}, 2024.

\bibitem[Zhang et~al.(2020{\natexlab{a}})Zhang, Jin, Fang, and Wang]{zhang2020improved}
Bohang Zhang, Jikai Jin, Cong Fang, and Liwei Wang.
\newblock Improved analysis of clipping algorithms for non-convex optimization.
\newblock \emph{Advances in Neural Information Processing Systems}, 33:\penalty0 15511--15521, 2020{\natexlab{a}}.

\bibitem[Zhang et~al.(2020{\natexlab{b}})Zhang, He, Sra, and Jadbabaie]{zhang2020gradient}
Jingzhao Zhang, Tianxing He, Suvrit Sra, and Ali Jadbabaie.
\newblock Why {Gradient} {Clipping} {Accelerates} {Training}: A {Theoretical} {Justification} for {Adaptivity}.
\newblock In \emph{International {Conference} on {Learning} {Representations} ({ICLR})}, 2020{\natexlab{b}}.

\bibitem[Zhao et~al.(2021)Zhao, Xie, and Li]{zhao2021convergence}
Shen-Yi Zhao, Yin-Peng Xie, and Wu-Jun Li.
\newblock On the convergence and improvement of stochastic normalized gradient descent.
\newblock \emph{Science China Information Sciences}, 64:\penalty0 1--13, 2021.

\end{thebibliography}

\appendix

\allowdisplaybreaks[4]

\section{Full ResNet Results}

The data for all the $6$ runs of the ResNet experiment is listed in Table~\ref{tab:full_accuracy_results_resnet}.

\begin{table}[h]
    \small
    \centering
    \begin{tabular}{lllc}
    \toprule
    Warm-up Schedule & Metric & Individual Runs & Mean $\pm$ Std. Dev. \\
    \midrule
    \multirow{2}{*}{Theoretical Warmup} & Val Acc. & [0.8567, 0.8600, 0.8623, 0.8595, 0.8553, 0.8599] & $0.8589 \pm 0.0023$ \\
    & Train Loss & [0.0425, 0.0276, 0.0031, 0.0504, 0.0138, 0.0633] & $0.0335 \pm 0.0208$ \\
    \midrule
    \multirow{2}{*}{Linear Warmup} & Val Acc. & [0.8549, 0.8593, 0.8570, 0.8612, 0.8550, 0.8590] & $0.8577 \pm 0.0023$ \\
    & Train Loss & [0.1043, 0.0413, 0.0122, 0.0242, 0.0244, 0.0242] & $0.0384 \pm 0.0306$ \\
    \midrule
    \multirow{2}{*}{No Warmup} & Val Acc. & [0.8532, 0.8546, 0.8559, 0.8573, 0.8578, 0.8585] & $0.8562 \pm 0.0019$ \\
    & Train Loss & [0.0196, 0.0300, 0.0344, 0.0355, 0.0675, 0.0819] & $0.0448 \pm 0.0221$ \\
    \bottomrule
    \end{tabular}
    \caption{Detailed results for different warm-up schedules, including individual run data, mean, and standard deviation over $6$ runs.}
    \label{tab:full_accuracy_results_resnet}
\end{table}

\section{Examples in Section~\ref{sec:FamilySmoothness_nerualnets}}\label{apx:examples}
The two examples are from \cite{patel2022global}. Readers can also refer to their paper for a detailed description.

\subsection{Feed Forward Neural Network}
We consider that $\sigma$ is linear and $\varphi$ is sigmoid. Suppose we have two sample points $(y,z)=(0,0)$ and $(y,z)=(1,1)$ with equal probability. The output $\hat{y}$ satisfies: $\hat{y}=\frac{1}{2}$ if $z=0$ and $\hat{y}=\paren{1+\exp\left\{-w_1w_2w_3w_4\right\}}^{-1}$ if $z=1$.
The binary cross entropy with ridge penalty can be written as
$$f_{z,y}(\bw)=-y\log\hat{y}-(1-y)\log\paren{1-\hat{y}}+\frac{1}{2}\sum_{i=1}^4 w_i^2.$$
Taking expectation over $(z,y)$, we obtain that
$$f(\bw)=\frac{1}{2}\log 2+\frac{1}{2}\log\paren{1+\exp\left\{-w_1w_2w_3w_4\right\}}+\frac{1}{2}\Norm{\bw}^2.$$
A simple calculation shows that
$$\nabla f(\bw)=\frac{-0.5}{1+\exp\left\{-w_1w_2w_3w_4\right\}}\left[\begin{array}{c}
     w_2w_3w_4  \\
     w_1w_3w_4  \\
     w_1w_2w_4  \\
     w_1w_2w_3  
\end{array}\right]+\left[\begin{array}{c}
     w_1  \\
     w_2  \\
     w_3  \\
     w_4  
\end{array}\right],$$
and
\begin{equation}\nonumber
       \begin{aligned}
\nabla^2 f(\bw) & =\underbrace{\frac{-0.5}{1+\exp \left\{w_1w_2w_3w_4\right\}}\left[\begin{array}{cccc}
0 & w_3w_4 & w_w w_4 & w_2 w_3 \\
w_3 w_4 & 0 & w_1 w_4 & w_1 w_3 \\
w_2 w_4 & w_1 w_4 & 0 & w_1 w_2 \\
w_2 w_3 & w_1 w_3 & w_1 w_2 & 0
\end{array}\right]}_{A} \\
& +\underbrace{\frac{0.5 \exp \left(w_1w_2w_3w_4\right)}{\left(1+\exp \left\{w_1w_2w_3w_4\right\}\right)^2}\left[\begin{array}{l}
w_2w_3w_4\\
w_1w_3w_4 \\
w_1w_2w_4 \\
w_1w_2w_3
\end{array}\right]\left[\begin{array}{l}
w_2w_3w_4\\
w_1w_3w_4 \\
w_1w_2w_4 \\
w_1w_2w_3
\end{array}\right]^\top}_{B}+I_4
\end{aligned}
\end{equation}

We first show that $f$ is not $(\rho,L_0,L_\rho)$-smooth for any $0\leq\rho<2$. Let $\bw=\paren{0,w_4,w_4,w_4}^\top$. Then 
$$\Norm{\nabla f(\bw)}_1\leq \frac{1}{4}|w_4|^3+3|w_4|.$$
Since $\Norm{\nabla^2 f(\bw)}_F$ is lower bounded by $\nabla^2 f(\bw)_{(1,1)}$, we have
$$\Norm{\nabla^2 f(\bw)}_F\geq \frac{1}{8}w_4^6.$$
Therefore, if $f$ is $(\rho,L_0,L_\rho)$-smooth, then it must hold that $\rho\geq 2$.

Next, we show that $f$ is $(\rho,K_0,K_\rho)$-smooth for some $\rho>0$. Note that 
$$\|A\|_F^2= \frac{0.5}{\left(1+\exp\{w_1w_2w_3w_4\}\right)^2}\left(\sum_{1\leq i<j\leq 4}\left(w_iw_j\right)^2\right)\leq \left(\sum_{i=1}^4 w_i^2\right)^2,$$
\begin{equation}\nonumber
    \begin{aligned}
        \|B\|_F^2&=\frac{0.25\exp\{2w_1w_2w_3w_4\}}{\left(1+\exp\{w_1w_2w_3w_4\}\right)^4}\left(\sum_{1\leq i<j<k\leq 4}\left(w_iw_jw_k\right)^4+2\left(w_1w_2w_3w_4 \right)^2\sum_{1\leq i< j\leq 4}\left(w_iw_j\right)^2\right)\\
        &\leq \sum_{1\leq i<j<k\leq 4}\left(w_iw_jw_k\right)^4+\sum_{1\leq i< j\leq 4}\left(w_iw_j\right)^2\\
        &\leq \left(\sum_{i=1}^4 w_i^2\right)^6+\left(\sum_{i=1}^4 w_i^2\right)^2,
    \end{aligned}    
\end{equation}
where in the first inequality we use $\frac{e^{2x}}{(1+e^x)^4}x^2\leq 1$ and $\frac{e^{2x}}{(1+e^x)^4}\leq 1$. Combining the above results, we obtain that
$$\|\nabla^2 f(\bw)\|^2_F\leq 3\left(\sum_{i=1}^4 w_i^2\right)^6+6\left(\sum_{i=1}^4 w_i^2\right)^2+12.$$
Moreover, it is not hard to see that $f^\star\leq \log 2$ and 
$$f(\bw)-f^\star\geq \frac{1}{2}\sum_{i=1}^4 w_i^2-\frac{1}{2}\log 2.$$
Therefore, we conclude that $f$ is $(\rho,K_0,K_1)$-smooth with some $K_0,K_1>0$ and $\rho\geq 3$.

\subsection{Recurrent Neural Network}
We consider that $\sigma$ is linear and $\varphi$ is sigmoid. Suppose we have two sample points $(\bz,y)=(1,0,0,0,1)$ and $(0,0,0,0,0)$ with equal probability. Fix $h_0=0$ and $w_3=1$. We have
$\hat{y}=\frac{\exp\left\{w_1^3w_2z_0\right\}}{1+\exp\left\{w_1^3w_2z_0\right\}}.$
The binary cross entropy with ridge penalty can be written as
$$f_{\bz,y}(\bw)=-y\log\hat{y}-\paren{1-y}\log\paren{1-\hat{y}}+\frac{1}{2}\sum_{i=1}^2w_i^2.$$
Taking expectation over $(\bz,y)$ we obtain that
\begin{equation}\nonumber
    \begin{aligned}
        &f(\bw)=\frac{1}{2}\left(\log 2+\log \left(1+\exp \left(w_1^3 w_2\right)\right)-w_1^3 w_2+w_1^2+w_2^2\right)\\
        &\nabla{f}(\bw)=\left[\begin{array}{c}
\frac{-3 w_1^2 w_2}{2} \frac{1}{1+\exp \left(w_1^3 w_2\right)}+w_1 \\
\frac{-w_1^3}{2} \frac{1}{1+\exp \left(w_1^3 w_2\right)}+w_2
\end{array}\right]\\
        &\nabla^2{f}(\bw)=\left[\begin{array}{cc}
\frac{9 w_1^4 w_2^2 \exp \left(w_1^3 w_2\right)}{2\left(1+\exp \left(w_1^3 w_2\right)\right)^2}-\frac{3 w_1 w_2}{1+\exp \left(w_1^3 w_2\right)}+1 & \frac{3 w_1^5 w_2 \exp \left(w_1^3 w_2\right)}{2\left(1+\exp \left(w_1^3 w_2\right)\right)^2}-\frac{3 w_1^2}{2} \frac{1}{1+\exp \left(w_1^3 w_2\right)} \\
\frac{3 w_1^5 w_2 \exp \left(w_1^3 w_2\right)}{2\left(1+\exp \left(w_1^3 w_2\right)\right)^2}-\frac{3 w_1^2}{2} \frac{1}{1+\exp \left(w_1^3 w_2\right)} & \frac{w_1^6 \exp \left(w_1^3 w_2\right)}{2\left(1+\exp \left(w_1^3 w_2\right)\right)^2}+1
\end{array}\right].
    \end{aligned}
\end{equation}
We first show that $f$ is not $(\rho,L_0,L_\rho)$-smooth for any $0\leq\rho<2$. Let $w_2=0$. We have
$$\Norm{f(\bw)}_1=|w_1|+\frac{|w_1|^3}{4}.$$
Consider the bottom right entry of $\nabla^2 f(\bw)$, we have 
$$\Norm{\nabla^2 f(\bw)}_F>\frac{w_1^6}{8}.$$
Therefore, if $f$ is $(\rho,L_0,L_\rho)$-smooth, then it must hold that $\rho\geq 2$.

Next, we show that $f$ is $(\rho,K_0,\Kr)$-smooth for some $\rho\geq 0$. We directly compute
 \begin{equation}\nonumber
\begin{aligned}
&\|\nabla^2 f(\bw)\|_F^2\leq \frac{\exp\left\{2w_1^3w_2\right\}}{\left(1+\exp\left\{w_1^3w_2\right\}\right)^4}\left(\frac{243}{4}w_1^8w_2^4+9w_1^{10}w_2^2+\frac{1}{2}w_1^{12}\right)\\
&+\frac{9}{1+\exp\left\{w_1^3w_2\right\}}\left(3w_1^2w_2^2+w_1^4\right)+5 \\
=&\underbrace{\frac{\exp\left\{2w_1^3w_2\right\}}{\left(1+\exp\left\{w_1^3w_2\right\}\right)^4}w_1^8\left(\frac{243}{4}w_2^4+9w_1^2w_2^2\right)}_{A}+\underbrace{\frac{\exp\left\{2w_1^3w_2\right\}}{\left(1+\exp\left\{w_1^3w_2\right\}\right)^4}\frac{1}{2}w_1^{12}}_{B}\\
&+\underbrace{\frac{9}{1+\exp\left\{w_1^3w_2\right\}}\left(3w_1^2w_2^2+w_1^4\right)}_{C} + 5.
\end{aligned}
\end{equation}
Since $\frac{e^{2x}}{(1+e^x)^4}x^2\leq 1$ and $\frac{e^{2x}}{(1+e^x)^4}\leq 1$, we have
\begin{align*}
    &A=\frac{\exp\left\{2w_1^3w_2\right\}}{\left(1+\exp\left\{w_1^3w_2\right\}\right)^4}\left(w_1^3w_2\right)^2w_1^2\left(\frac{243}{4}w_2^2+9w_1^2\right)\\
    &\leq \frac{243}{4}w_1^2(w_1^2+w_2^2)\leq \frac{243}{4}(w_1^2+w_2^2)^2,\\
    &B\leq \frac{1}{2}w_1^{12}\leq \frac{1}{2}(w_1^2+w_2^2)^6,\\
    &C\leq 9\times(3w_1^2w_2^2+w_1^4)\leq 9\times\frac{3}{2}(w_1^2+w_2^2)^2.
\end{align*}
Combining the above results, we obtain that
\begin{equation}\nonumber
    \begin{aligned}
        \|\nabla^2 f(\theta)\|_F^2\leq& \frac{243}{4}\left(w_1^2+w_2^2\right)^2+\frac{1}{2}(w_1^2+w_2^2)^6+\frac{27}{2}(w_1^2+w_2^2)^2+5\\
        \leq& 256\left(w_1^2+w_2^2\right)^2+\left(w_1^2+w_2^2\right)^6+5 \\
    \end{aligned}
\end{equation}
Moreover, note that $f^\star\leq \log 2$ and 
$$f(\bw)-f^\star\geq w_1^2+w_2^2-\frac{1}{2}\log 2.$$
We conclude that $f$ is $(\rho,K_0,K_1)$-smooth with some $K_0,K_1>0$ and $\rho\geq 3$.

\section{Proofs for Section~\ref{sec:FamilySmoothness}}

\subsection{Proof of Lemma~\ref{lem:L_in_K_smooth}}
\begin{proof}
        Since $f$ is $(\rho, L_0, L_\rho)$-smooth with $0\leq\rho <2$, by \cite[Lemma 3.5]{li2023convex} we have that for all $\bw\in\RR^{d}$, 
    $$
f(\bw)-f^\star\geq \frac{\|\nabla f(\bw)\|^2}{2L_0+2^{\rho+1}L_\rho\|\nabla f(\bw)\|^\rho}.$$

    If $2L_0 \leq 2^{\rho+1}L_\rho\|\nabla f(\bw)\|^\rho$, we obtain that 
    $$f(\bw)-f^\star\geq \frac{\|\nabla f(\bw)\|^2}{2^{\rho+2}L_\rho\|\nabla f(\bw)\|^\rho}=\frac{\|\nabla f(\bw)\|^{2-\rho}}{2^{\rho+2}L_\rho}.$$
    By the definition of $(\rho,L_0,L_\rho)$-smoothness we have 
    $$\|\nabla^2 f(\bw)\|\leq L_0+L_\rho\|\nabla f(\bw)\|^\rho\leq L_0+L_\rho^{\frac{2}{2-\rho}}2^{\frac{\rho(\rho+2)}{2-\rho}}\left(f(\bw)-f^\star\right)^{\frac{\rho}{2-\rho}}.$$

    If $2L_0 > 2^{\rho+1}L_\rho\|\nabla f(\bw)\|^\rho$, $\|\nabla f(\bw)\|$ is bounded:
    $$\|\nabla f(\bw)\|^\rho < \frac{L_0}{2^\rho L_\rho}.$$
    Again, by the definition of $(\rho,L_0,L_\rho)$-smoothness we have
    $$\|\nabla^2 f(\bw)\|\leq L_0+L_\rho\|\nabla f(\bw)\|^\rho\leq L_0+\frac{L_0}{2^\rho}\leq 2L_0$$

    Combining the two cases, we obtain that
    $$\|\nabla^2 f(\bw)\|\leq 2L_0+L_\rho^{\frac{2}{2-\rho}}2^{\frac{\rho(\rho+2)}{2-\rho}}\left(f(\bw)-f^\star\right)^{\frac{\rho}{2-\rho}}.$$
    This completes the proof.
\end{proof}

\subsection{Proof of Lemma~\ref{lem:improved_descent_smooth}}

For any two points $\bx,\by\in\mathbb{R}^d$, we 
define 
\begin{equation}\label{eq:defi_ht}
    h(t):=\int_0^t K_0+K_\rho\left(f\left(\bx+v(\by-\bx)\right)-f^\star\right)^\rho \mathrm{d}v,t\in[0,1].
\end{equation}
By the definition of $(\rho,K_0,K_\rho)$-smoothness we have 
\begin{equation}\label{eq:ht_integral}
    \int_0^t \|\nabla^2 f(\bx+v(\by-\bx))\|\mathrm{d}v\leq h(t).
\end{equation}

Note that 
\begin{equation}\nonumber
    \begin{aligned}
        &\|\nabla f(\by)-\nabla f(\bx)\|=\|\int_0^1 \nabla^2 f\left(\bx+t(\by-\bx)\right)(\by-\bx)\mathrm{d}t\|\\
        &\leq \|\by-\bx\|\int_0^1\|\nabla^2 f\left(\bx+t(\by-\bx)\right)\|\mathrm{d}t
        \leq h(1)\|\by-\bx\|,
    \end{aligned}
\end{equation}
and 
\begin{equation}\label{eq:descent_h1}
    \begin{aligned}
        &f(\by)-f(\bx)=\int_0^1 \langle \nabla f(\bx+t(\by-\bx)-\nabla f(\bx), \by-\bx\rangle\mathrm{d}t + \langle\nabla f(\bx),\by-\bx \rangle\\
        &\leq \|\by-\bx\|\int_0^1\|\nabla f(\bx+t(\by-\bx))-\nabla f(\bx)\|\
        +\langle\nabla f(\bx),\by-\bx \rangle\\
        &\leq \langle\nabla f(\bx),\by-\bx \rangle + \frac{1}{2}h(1)\|\by-\bx\|^2.
    \end{aligned}
\end{equation}

To prove Lemma~\ref{lem:improved_descent_smooth}, it suffices to bound $h(1)$. We need the following Grönwall's inequality.

\begin{lemma}[Lemma A.3,
\cite{li2023convex}]\label{lem:gronwall}
     Let $u:[a,b] \to [0,\infty)$ and $l:[0,\infty) \to (0,\infty)$ be two continuous functions. Suppose $u'(t) \le l(u(t))$ for all $t \in [a,b]$, then it holds for all $t \in [a,b]$ that
    \begin{align}\nonumber
        \int_{u(a)}^{u(t)} \frac{1}{l(w)} \mathrm{d} w \le t - a.
    \end{align}
\end{lemma}

\begin{lemma}\label{lem:K_smooth_defi2}
    Suppose Assumption~\ref{assumption:K_smooth} holds. For any $\bx,\by\in\mathbb{R}^d$, define $h(t)$ as in \eqref{eq:defi_ht}.
    Let $m>0$ be any positive number and 
    $a=K_0+\Kr\paren{m+f(\bx)-f^\star}^\rho$.
    We have that $h(1)\leq a$ if
    $a\Norm{\by-\bx}^2+\Norm{\by-\bx}\Norm{\nabla f(\bx)}\leq m.$
\end{lemma}

\begin{proof}
    For any $\bx,\by\in\RR^{d}$, we have
\begin{equation}\nonumber
    \begin{aligned}
        &f(\by)-f(\bx)=\int_0^1\langle\nabla f\left(\bx+w(\by-\bx)\right),\by-\bx\rangle \mathrm{d}w\\
        =&\int_0^1\int_0^w \langle \nabla^2 f(\bx+v(\by-\bx))(\by-\bx),\by-\bx\rangle \mathrm{d}v\mathrm{d}w + \langle\nabla f(\bx),\by-\bx\rangle\\
        \leq & \|\by-\bx\|^2\int_0^1\int_0^w  \|\nabla^2 f(\bx+v(\by-\bx))\|\mathrm{d}v\mathrm{d}w + \langle \nabla f(\bx), \by-\bx\rangle.
    \end{aligned}    
\end{equation}
    Replacing $\by$ with $\bx+t(\by-\bx)$ we obtain that
    \begin{equation}\nonumber
        \begin{aligned}
            &f\left(\bx+t(\by-\bx)\right)-f(\bx)\leq t^2\|\by-\bx\|^2\int_0^1\int_0^w\|\nabla^2 f\left(\bx+vt(\by-\bx)\right)\|\mathrm{d}v\mathrm{d}w+t\langle\nabla f(\bx),\by-\bx\rangle\\
            &= t^2\|\by-\bx\|^2\int_0^1\int_0^{tw}\|\nabla^2 f\left(\bx+v(\by-\bx)\right)\|\mathrm{d}v\mathrm{d}w+t\langle\nabla f(\bx),\by-\bx\rangle\\
            &\leq t^2\|\by-\bx\|^2\int_0^1 h(wt)\mathrm{d}w + t\|\nabla f(\bx)\|\|\by-\bx\|\\
            &\leq \|\by-\bx\|^2 h(t) + \|\nabla f(\bx)\|\|\by-\bx\|,
        \end{aligned}
    \end{equation}
    where the second inequality is due to \eqref{eq:ht_integral} and the last inequality is due to the fact that $h(\cdot)$ is positive and monotonically increasing and $0\leq t \leq 1$.
    Then,
    \begin{equation}\nonumber
        \begin{aligned}
            h^\prime(t)&=K_0+\Kr\left(f(\left(\bx+t\left(\by-\bx\right)\right)-f^\star\right)^\rho=K_0+\Kr\left(f(\left(\bx+t\left(\by-\bx\right)\right)-f(\bx)+f(\bx)-f^\star\right)^\rho\\
            &\leq K_0+\Kr \left(\|\by-\bx\|^2h(t)+\|\by-\bx\|\|\nabla f(\bx)\|+f(\bx)-f^\star\right)^\rho.
        \end{aligned}
    \end{equation}
    By Lemma~\ref{lem:gronwall}, let $l(w)=K_0+\Kr\left(\|\by-\bx\|^2w+\|\by-\bx\|\|\nabla f(\bx)\|+f(\bx)-f^\star\right)^\rho$, we obtain that
    $$\int_{h(0)}^{h(1)}\frac{1}{l(w)}\mathrm{d}w\leq 1.$$
    If $l(a)\leq a$ for some $a>0$, then $\int_0^{h(1)}\frac{1}{l(w)}\mathrm{d}w\leq 1\leq \frac{a}{l(a)}\leq\int_0^a\frac{1}{l(w)}\mathrm{d}w$. By the monotonicity of the integral, we have $h(1)\leq a$.
    Since we let $a=K_0+\Kr\paren{m+f(\bx)-f^\star}^\rho$, $l(a)\leq a$ is equiavlent to
    $$a\Norm{\by-\bx}^2+\Norm{\by-\bx}\Norm{\nabla f(\bx)}\leq m.$$
\end{proof}

\begin{lemma}[Bounded Gradient]\label{lem:bounded_gradient}
    Suppose Assumption~\ref{assumption:K_smooth} holds. Let $\Delta=f(\bx)-f^\star$. It holds that 
    \begin{equation}\label{eq:bounded_gradient}
        \|\nabla f(\bx)\|\leq 2\sqrt{K_0\Delta+\Kr 3^\rho\Delta^{\rho+1}}.
    \end{equation}
\end{lemma}

\begin{proof}
    By Lemma~\ref{lem:K_smooth_defi2}, let $m=2\Delta$, we obtain that
    $h(1)\leq K_0+\Kr\paren{3\Delta}^\rho=:a$, if 
    $$\|\by-\bx\|^2\paren{K_0+\Kr\paren{3\Delta}^\rho}+\|\by-\bx\|\|\nabla f(\bx)\|\leq 2\Delta.$$
    Equivalently, 
    $$\|\by-\bx\|\leq \frac{-\|\nabla f(\bx)\|+\sqrt{\|\nabla f(\bx)\|^2+8a\Delta}}{2a}=:r.$$
    For $\by$ satisfying $\Norm{\by-\bx}\leq r$, by \eqref{eq:descent_h1}, we have 
    $$f(\by)-f(\bx)\leq \langle\nabla f(\bx),\by-\bx\rangle+\frac{h(1)}{2}\|\by-\bx\|^2\leq \langle\nabla f(\bx),\by-\bx\rangle+\frac{a}{2}\|\by-\bx\|^2.$$
    Letting $\by=\bx-\frac{\eta}{\|\nabla f(\bx)\|}\nabla f(\bx)$, we obtain that
    $$-\Delta\leq f(\by)-f(\bx)\leq -\eta\|\nabla f(\bx)\|+\frac{a}{2}\eta^2.$$
    Therefore, we obtain that
    \begin{equation}\label{eq:eta_inequation}
        g(\eta):=\frac{a}{2}\eta^2-\eta\|\nabla f(\bx)\|+\Delta\geq 0, \forall \eta\in[0,r].
    \end{equation}
    It is not hard to see that $\arg\min_{\eta\in\RR} g(\eta)=\frac{\|\nabla f(\bx)\|}{a}$. We then consider two cases: $\frac{\|\nabla f(\bx)\|}{a}\leq r$ and $\frac{\|\nabla f(\bx)\|}{a}> r$.
     Suppose $\frac{\|\nabla f(\bx)\|}{a}\leq r$. Equivalently, $\|\nabla f(\bx)\|\leq \sqrt{a\Delta}.$
    By \eqref{eq:eta_inequation} we need
    $$\|\nabla f(\bx)\|\leq \sqrt{2a\Delta}.$$
    Now suppose $\frac{\|\nabla f(\bx)\|}{a}> r$. Equivalently, $\|\nabla f(\bx)\|>\sqrt{a\Delta}$. By \eqref{eq:eta_inequation} we need
    $$\frac{a}{2}r^2-r\|\nabla f(\bx)\|+\Delta\geq 0.$$
    Equivalently, 
    $$\|\nabla f(\bx)\|\leq \sqrt{\frac{8}{3}a\Delta}.$$
    Combing the above two cases, we conclude that $\|\nabla f(\bx)\|\leq \sqrt{\frac{8}{3}a\Delta}\leq 2\sqrt{a\Delta}=2\sqrt{K_0\Delta+\Kr3^\rho\Delta^{\rho+1}}$. 
\end{proof}

\textbf{Proof of Lemma~\ref{lem:improved_descent_smooth}}
\begin{proof}
    By Lemma~\ref{lem:K_smooth_defi2}, for any $m>0$, we have $h(1)\leq K_0+\Kr\paren{m+\Delta}^\rho=:a$, if
    \begin{align*}
        \Norm{\by-\bx}\leq\frac{2m}{\Norm{\nabla f(\bx)}+\sqrt{\Norm{\nabla f(\bx)}^2+4am}}=:r.
    \end{align*}
     Let $A=K_0\Delta+\Kr 3^\rho\Delta^{\rho+1}$ and $B=A+m\left(K_0+\Kr\left(m+\Delta\right)^\rho\right)$.
    By Lemma~\ref{lem:bounded_gradient}, we have $\Norm{\nabla f(\bx)}\leq 2\sqrt{A}$, and thus
    $$r\geq \frac{m}{\sqrt{A}+\sqrt{B}}.$$

    Let $$m=\max\left\{\Delta, \frac{1}{3}\left(\frac{K_0}{K_1}\right)^{1/\rho}\right\}.$$
    If $K_0\leq K_1 3^\rho\Delta^\rho$, 
    we have $m=\Delta$, $A\leq 2K_13^\rho\Delta^{\rho+1}$ 
    and 
    $$B=2K_0\Delta + K_13^\rho\Delta^{\rho+1}+K_12^\rho\Delta^{\rho+1}\leq 4K_13^\rho\Delta^{\rho+1}.$$
    Thus 
    $$\frac{m}{\sqrt{A}+\sqrt{B}}\geq \frac{1}{\left(2+\sqrt{2}\right)\sqrt{3^\rho K_1}\Delta^{\frac{\rho-1}{2}}}=:C_1\Delta^{-\frac{\rho-1}{2}}.$$
    If $K_0 > K_13^\rho\Delta^\rho$, 
    we have $m=\frac{1}{3}\left(\frac{K_0}{K_1}\right)^{1/\rho}$, $A\leq 2K_0\Delta\leq 2K_0m$ 
    and 
    $$B\leq 2K_0m + m\left(K_0+K_1\left(2m\right)^\rho\right)\leq 4mK_0.$$
    Thus 
    $$\frac{m}{\sqrt{A}+\sqrt{B}}\geq \frac{m}{\left(2+\sqrt{2}\right)\sqrt{K_0m}}=\frac{1}{2\sqrt{3}+\sqrt{6}}\frac{K_0^{\frac{1}{2\rho}-\frac{1}{2}}}{K_1^{\frac{1}{2\rho}}}=:C_2.$$
    Next we bound $a=K_0+K_1\left(m+\Delta\right)^\rho$. If $m=\Delta$, we have
    $$K_0+K_1\left(m+\Delta\right)^\rho=K_0+K_12^\rho\Delta^\rho.$$
    If $m=\frac{1}{3}\left(\frac{K_0}{K_1}\right)^{1/\rho}$, we have
    $$K_0+K_1\left(m+\Delta\right)^\rho\leq K_0+K_12^\rho m^\rho=K_0+\left(\frac{2}{3}\right)^\rho K_0\leq 2K_0.$$
    Combining the above results, we get the desired result.
\end{proof}

\section{Proofs for Section~\ref{sec:GD}}
For simplicity, we let $\Delta_t=f(\bw_t)-f^\star$, $r_t=\min\left\{C_1\Delta_t^{-\frac{\rho-1}{2}},C_2\right\}$ and $L_t=2K_0+\Kr\paren{2\Delta_t}^\rho$.

\subsection{Proof of Theorem~\ref{thm:GD}}\label{apx:proof_GD}

\begin{proof}
    We first note that if $K_0\leq \Kr\paren{3\Delta_t}^\rho$, we have 
    $r_t=C_1\Delta_t^{-\frac{\rho-1}{2}}$ and $\Norm{\nabla f(\bw_t)}\leq 2\sqrt{K_0\Delta_t+\Kr 3^\rho\Delta_t^{\rho+1}}\leq 2\sqrt{2\Kr3^\rho\Delta_t^{\rho+1}}$, and if $K_0>\Kr\paren{3\Delta_t}^{\rho}$, we have $r_t=C_2$ and $\Norm{\nabla f(\bw_t)}\leq 2\sqrt{\frac{2}{3} K_0\paren{\frac{K_0}{K_1}}^{1/\rho}}$.

    Therefore, to ensure $\Norm{\bw_{t+1}-\bw_t}=\eta_t\Norm{\nabla f(\bw_t)}\leq r_t$, it suffices to set $\eta_t=\frac{1}{4\sqrt{2}+4}\min\left\{\frac{1}{K_0},\frac{1}{3^\rho\Kr\Delta_t^{\rho}}\right\}$. Then by Lemma~\ref{lem:improved_descent_smooth}, 
    \begin{equation}\nonumber
        \begin{aligned}
            f(\bw_{t+1})&\leq f(\bw_t)+\langle\nabla f(\bw_t),\bw_{t+1}-\bw_t\rangle+\frac{L_t}{2}\Norm{\bw_{t+1}-\bw_t}^2\\
            &=f(\bw_t)-\eta_t\Norm{\nabla f(\bw_t)}^2+\frac{L_t}{2}\eta_t^2\Norm{\nabla f(\bw_t)}^2\\
            &\leq f(\bw_t)-\frac{\eta_t}{2}\Norm{\nabla f(\bw_t)}^2\leq f(\bw_t),
        \end{aligned}
    \end{equation}
    where the last inequality is due to $\eta_t\leq \frac{1}{\paren{2+2\sqrt{2}}\paren{K_0+\Kr\paren{3\Delta_t}^\rho}}\leq \frac{1}{2K_0+\Kr\paren{2\Delta_t}^\rho}=\frac{1}{L_t}.$
    Telescoping the above inequation from $t=0$ to $t=T-1$ we obtain that
    $$\sum_{t=0}^{T-1}\eta_t\Norm{\nabla f(\bw_t)}^2\leq 2\left(f(\bw_0)-f(\bw_T)\right)\leq 2\Delta_0. $$
    Note that $1/\eta_t\leq (4\sqrt{2}+4)\paren{K_0+\Kr\paren{3\Delta_t}^\rho}$. 
    Using the QM-GM inequality, we have 
    $$\sum_{t=0}^{T-1} \eta_t\geq \frac{T^2}{\sum_{t=0}^{T-1}1/\eta_t}\geq \frac{1}{4\sqrt{2}+4}\frac{T^2}{\sum_{t=0}^{T-1}K_0+\Kr\paren{3\Delta_t}^\rho}.$$
    This completes the proof for the increasing learning rate.

    Now suppose we use the constant learning rate $\eta\leq \frac{1}{4\sqrt{2}+4}\min\left\{\frac{1}{K_0},\frac{1}{\Kr\paren{3\Delta_0}^\rho}\right\}$. Similar to the increasing learning rate, we have 
    $\Norm{\bw_1-\bw_0}\leq r_0$ and 
    $$f(\bw_1)\leq f(\bw_0)-\frac{\eta}{2}\Norm{\nabla f(\bw_0)}^2\leq f(\bw_0).$$
    This means $\Delta_1\leq\Delta_0$ and thus $\eta\leq \frac{1}{4\sqrt{2}+4}\min\left\{\frac{1}{K_0},\frac{1}{\Kr\paren{3\Delta_1}^\rho}\right\}$. 
    By induction, it is not hard to see that $\eta\leq \frac{1}{4\sqrt{2}+4}\min\left\{\frac{1}{K_0},\frac{1}{\Kr\paren{3\Delta_t}^\rho}\right\},\forall t\in[T]$. Therefore, following the same analysis as the increasing learning rate, we have
    $$\eta\sum_{t=0}^{T-1}\Norm{\nabla f(\bw_t)}^2\leq 2\Delta_0.$$
    This completes the proof.
\end{proof}

\subsection{Proof of Theorem~\ref{thm:GD_convex}}
{
We first prove the cocoercivity for $\paren{\rho,K_0,\Kr}$-smooth convex functions, similar to the property of the $L$-smooth convex functions. The proof follows a similar approach to that in \citep{li2023convex}.
}
\begin{lemma}\label{lem:improved_descent_smooth_2}
    Under the same conditions as in Lemma~\ref{lem:improved_descent_smooth}, for any $\bx_1,\bx_2$ such that $\|\bx_1-\bx\|\leq r(\bx)$ and $\Norm{\bx_2-\bx}\leq r(\bx)$, where $r(\bx)$ is defined in Lemma~\ref{lem:improved_descent_smooth}. Then we have
    $$\Norm{\nabla f(\bx_1)-\nabla f(\bx_2) }\leq L(\bx)\Norm{\bx_1-\bx_2},$$
    and 
    $$f(\bx_2)\leq f(\bx_1)+\dotprod{\nabla f(\bx_1)}{\bx_2-\bx_1}+\frac{L(\bx)}{2}\Norm{\bx_2-\bx_1}^2,$$
    where $L(\bx)$ is defined in Lemma~\ref{lem:improved_descent_smooth}.
\end{lemma}
\begin{proof}
    Let $m>0$, $\Delta=f(\bx)-f^\star$, $a=K_0+\Kr\paren{m+\Delta}^\rho$ and 
    $r:=\frac{2m}{\Norm{\nabla f(\bx)}+\sqrt{\Norm{\nabla f(\bx)}^2+4am}}$.
    Let $\Norm{\by-\bx}\leq r$, by Lemma~\ref{lem:K_smooth_defi2}, we have 
    \begin{equation}\label{eq:bounded_function_gap}
        \begin{aligned}
            f(\by)&\leq f(\bx)+\langle\nabla f(\bx),\by-\bx\rangle + \frac{a}{2}\Norm{\by-\bx}^2\\
        &\leq f(\bx)+\Norm{\nabla f(\bx)}\Norm{\by-\bx}+\frac{a}{2}\Norm{\by-\bx}^2\\
        &\leq f(\bx)+\frac{1}{8a}\paren{\sqrt{\Norm{\nabla f(\bx)}^2+4am}-\Norm{\nabla f(\bx)}}\paren{3\Norm{\nabla f(\bx)}+\sqrt{\Norm{\nabla f(\bx)}^2+4am}}\\
        & =f(\bx)+\frac{m}{2} + \frac{\Norm{\nabla f(\bx)}}{4a}\paren{\sqrt{\Norm{\nabla f(\bx)}^2+4am}-\Norm{\nabla f(\bx)}}\\
        &=f(\bx)+\frac{m}{2}+\frac{m}{1+\sqrt{1+\frac{4am}{\Norm{\nabla f(\bx)}^2}}}\leq f(\bx)+m.
        \end{aligned}
    \end{equation}
    Then, for $\|\bx_1-\bx\|\leq r$ and $\Norm{\bx_2-\bx}\leq r$, we have
    \begin{align*}
        \Norm{\nabla f(\bx_1)-\nabla f(\bx_2)}&\leq \Norm{\bx_1-\bx_2}\int_{0}^1\Norm{\nabla^2 f\left(\bx_1+t\paren{\bx_2-\bx_1}\right)}\mathrm{d}t\\
        &\leq \Norm{\bx_1-\bx_2}\int_0^1\paren{K_0+K_\rho \paren{f(\bx_1+t(\bx_2-\bx_1))-f^\star}^\rho}\mathrm{d}t\\
        &\leq \Norm{\bx_2-\bx_1}\paren{K_0+\Kr\paren{f(\bx)-f^\star+m}^\rho},
    \end{align*}
    where the second inequality is due to Assumption~\ref{assumption:K_smooth} and the in the last inequality we use $\Norm{\bx_1+t(\bx_2-\bx_1)-\bx}\leq r$ for $t\in [0,1]$ and 
    \eqref{eq:bounded_function_gap}. Setting $m=\max\left\{f(\bx)-f^\star, \frac{1}{3}\paren{\frac{K_0}{K_\rho}}^{1/\rho}\right\}$ and following the proof of Lemma~\ref{lem:improved_descent_smooth} 
    we obtain the desired result. 
\end{proof}

\begin{lemma}\label{lem:cocoercivity}
    Suppose Assumption~\ref{assumption:K_smooth} holds and $f$ is convex. Then for any given $\bx\in\mathbb{R}^d$, we have
    $$\frac{1}{L(\bx)}\Norm{\nabla f(\bx)-\nabla f(\by)}^2\leq f(\bx)-f(\by)-\dotprod{\nabla f(\by)}{\bx-\by}, \quad \forall \by \text{ such that } \Norm{\by-\bx}\leq \frac{r(\bx)}{2},$$
    where $L(\bx)=2K_0+\Kr\paren{2\Delta}^\rho, \Delta=f(\bx)-f^\star$ and $r(\bx)=\min\left\{C_1\Delta^{-\frac{\rho-1}{2}}, C_2\right\}$ as defined in Lemma~\ref{lem:improved_descent_smooth}.
\end{lemma}

\begin{proof}
    Define $\phi_{\bx}(\bz):=f(\bz)-\dotprod{\nabla f(\bx)}{\bz}$.
    It is not hard to verify that $\phi_{\bx}$ is $\paren{\rho,K_0,\Kr}$-smooth. Note that if $\Norm{\by-\bx}\leq \frac{r(\bx)}{2}$, we have
    $$\Norm{\by-\frac{1}{L(\bx)}\nabla\phi_{\bx}(\by)-\bx}\leq \Norm{\by-\bx}+\frac{1}{L(\bx)}\Norm{\nabla \phi_{\bx}(\by)}\leq 2\Norm{\by-\bx}\leq r(\bx),$$
    where the second last inequality is due to Lemma~\ref{lem:improved_descent_smooth}.
    Applying Lemma~\ref{lem:improved_descent_smooth_2} with points $\by-\frac{1}{L(\bx)}\nabla \phi_{\bx}(\by)$ and $\by$, we obtain that
    \begin{align*}
        \phi_{\bx}\paren{\by-\frac{1}{L(\bx)\nabla\phi_{\bx}(\by)}}&\leq \phi_{\bx}(\by)+\dotprod{\nabla \phi_{\bx}(\by)}{-\frac{1}{L(\bx)}\nabla_{\bx}\phi(\by)}+\frac{L(\bx)}{2}\Norm{\frac{1}{L(\bx)}\nabla\phi_{\bx}(\by)}^2\\
        &=\phi_{\bx}(\by)-\frac{1}{2L(\bx)}\Norm{\nabla \phi_{\bx}(\by)}^2.
    \end{align*}
    Substituting the definition of $\phi_{\bx}$ and noting that $\bx=\arg\min_{\bz}\phi_{\bx}(\bz)$, we obtain the desired result.

\end{proof}

\textbf{Proof of Theorem~\ref{thm:GD_convex}}

Let $\Delta_t=f(\bw_t)-f(\bw^\star)$ and $D_t=\Norm{\bw_t-\bw^\star}$ for simplicity. We first note that $\Norm{\bw_{t+1}-\bw_t}\leq \frac{r_t}{2}$ following a similar analysis in the proof of Theorem~\ref{thm:GD}.
We then calculate that
\begin{equation}\nonumber
    \begin{aligned}
        &-2\eta_t\dotprod{\bw_t-\bw^\star}{\nabla f(\bw_t)}+\eta_t^2\Norm{\nabla f(\bw_t)}^2\\
        &=-2\eta_t\paren{f(\bw_t)-f(\bw^\star)}+\eta_t^2\Norm{\nabla f(\bw_t)}^2+2\eta_t\paren{f(\bw_t)-f(\bw^\star)+\dotprod{\nabla f(\bw_t)}{\bw^\star-\bw_t}}\\
        &\leq -2\eta_t\paren{f(\bw_t)-f(\bw^\star)}+\eta_t^2\Norm{\nabla f(\bw_t)}^2-\frac{2\eta_t}{L_t}\Norm{\nabla f(\bw_t)}^2\\
        &=-2\eta_t\Delta_t+\eta_t\Norm{\nabla f(\bw_t)}^2\paren{\eta_t-\frac{2}{L_t}},
    \end{aligned}
\end{equation}
where the inequality is due to Lemma~\ref{lem:cocoercivity}. Similar to the analysis in the proof of Theorem~\ref{thm:GD}, $\eta_t\leq \frac{2}{L_t}$.
Therefore, we obtain that
\begin{align*}
    \Norm{\bw_{t+1}-\bw^\star}^2&=\Norm{\bw_t-\bw^\star}^2-2\eta_t\dotprod{\bw_t-\bw^\star}{\nabla f(\bw_t)}+\eta_t^2\Norm{\nabla f(\bw_t)}^2\\
    &\leq \Norm{\bw_t-\bw^\star}^2-2\eta_t \paren{f(\bw_t)-f(\bw^\star)}.
\end{align*}
Telescoping the above inequality from $t=0$ to $T-1$, we obtain that
\begin{align*}
    \sum_{t=0}^{T-1}\eta_t\Delta_t\leq \frac{1}{2}D_0^2.
\end{align*}
By the definition of $\eta_t$, we have
$$
\sum_{t=0}^{T-1}\min\left\{\frac{\Delta_t}{K_0}, \frac{\Delta_t}{\Kr\paren{3\Delta_t}^\rho}\right\}\leq \paren{4\sqrt{2}+4}D_0^2,
$$
and thus
$$
\min_{t\in[T]}\min\left\{\frac{\Delta_t}{K_0}, \frac{\Delta_t}{\Kr\paren{3\Delta_t}^\rho}\right\}\leq \paren{4\sqrt{2}+4}\frac{D_0^2}{T}.
$$
Note that following the same analysis in the proof of Theorem~\ref{thm:GD}, we have that $f(\bw_{t+1})\leq f(\bw_t)-\frac{\eta_t}{2}\Norm{\nabla f(\bw_t)}^2$, which implies that $\Delta_t$ is decreasing. If $\rho\geq 1$, we have that
$$
\min_{t\in[T]}\min\left\{\frac{\Delta_t}{K_0}, \frac{\Delta_t}{\Kr\paren{3\Delta_t}^\rho}\right\}=\min\left\{\frac{\Delta_{T-1}}{K_0}, \frac{\Delta_{0}^{1-\rho}}{3^\rho\Kr}\right\}\leq \paren{4\sqrt{2}+4}\frac{D_0^2}{T}.
$$
This implies that either $\Delta_{T-1}\leq \paren{4\sqrt{2}+4}\frac{D_0^2 K_0}{T}$ or $T\leq \paren{4\sqrt{2}+4}\frac{3^\rho D_0^2\Kr} {\Delta_0^{1-\rho}}$.
If $0<\rho<1$, we have that
$$
\min_{t\in[T]}\min\left\{\frac{\Delta_t}{K_0}, \frac{\Delta_t}{\Kr\paren{3\Delta_t}^\rho}\right\}=\min\left\{\frac{\Delta_{T-1}}{K_0}, \frac{\Delta_{T-1}^{1-\rho}}{3^\rho\Kr}\right\}\leq \paren{4\sqrt{2}+4}\frac{D_0^2}{T}.
$$
This implies that either $\Delta_{T-1}\leq\paren{4\sqrt{2}+4}\frac{D_0^2K_0}{T}$ or $\Delta_{T-1}^{1-\rho}\leq \paren{4\sqrt{2}+4}\frac{3^\rho D_0^2\Kr}{T}$.

For constant learning rate $\eta_t=\eta=\frac{1}{8\sqrt{2}+8}\min\left\{\frac{1}{K_0},\frac{1}{\Kr \paren{3\Delta_0}^\rho}\right\}$, we have 
$$
\Delta_{T-1}\leq \frac{D_0^2}{2\eta T}\leq \paren{4\sqrt{2}+4}\paren{\frac{D_0^2K_0}{T}+\frac{D_0^2\Kr\paren{3\Delta_0}^\rho}{T}}.
$$
This completes the proof.

\subsection{Proof of Theorem~\ref{thm:lower_bound_gd}}\label{apx:proof_gd_lower_bound}

\begin{proof}

We consider three cases of $\{\eta_t\}$:
\begin{enumerate}
    \item\label{lower_case1} $\eta_t\leq\frac{2}{K_1\Delta},\forall t$.
    \item\label{lower_case2} $\eta_t>\frac{2}{K_1\Delta},\forall t$.
    \item\label{lower_case3} $\exists \tau,$  $\eta_t>\frac{2}{K_1\Delta},t\leq\tau$ and $\eta_t\leq\frac{2}{K_1\Delta},t>\tau$.
\end{enumerate}

\textbf{Case~\ref{lower_case1}}

 We construct the following function:
    \begin{equation}\label{eq:lower_h}
        h(x) = \left\{ \begin{array}{cc}
           -2 \epsilon \left( x+\frac{1}{\sqrt{K_1}} \right) + \frac{5 \epsilon}{4 \sqrt{K_1}}  , & x \in (-\infty, -\frac{1}{\sqrt{K_1}}) \\
            \frac{\epsilon}{4} \left( 6 \sqrt{K_1} x^2 - K_1^{\frac{3}{2}} x^4 \right) , & x\in \left[ -\frac{1}{\sqrt{K_1}}, \frac{1}{\sqrt{K_1}} \right] \\
            2 \epsilon \left( x - \frac{1}{\sqrt{K_1}} \right) + \frac{5 \epsilon}{4 \sqrt{K_1}} , & x \in \left( \frac{1}{\sqrt{K_1}}, +\infty \right)
        \end{array} \right.
    \end{equation}
    with initial point $y_0 = \frac{1}{\sqrt{K_1}} + \frac{\Delta}{\epsilon}$. We can verify that $h$ is $(\epsilon \sqrt{K_1},0)$-smooth and $h(y_0) - f^* \le 2\Delta + \epsilon$. Then as $\eta_t \le \frac{2}{K_1 \Delta}$ for all $t \ge 0$, we have $y_t \ge y_{t-1} - \frac{4\epsilon}{K_1 \Delta}$ if $y_t \ge \frac{1}{\sqrt{K_1}}$. Therefore, it takes at least 
    \begin{align*}
        \frac{y_0 - \frac{1}{\sqrt{K_1}}}{\frac{4\epsilon}{K_1 \Delta}} \ge \frac{K_1 \Delta^2}{4 \epsilon^2}
    \end{align*}
    iterations.

\textbf{Case~\ref{lower_case2}}

Given $x_0$, we define $x_t=x_0-\sum_{s=0}^{t-1}\eta_s\sqrt{K_1}\Delta,\forall t\in\mathbb{N}$.
We define 
$$f(x)=f_t(x),x\in(x_{t+1},x_t],$$
and
\begin{equation}\nonumber
    f_t(x) = a_t \sin(b_t x+c_t) + d_t, \quad x\in (x_{t+1}, x_t]
\end{equation}
with
\begin{align*}
    b_t =& \frac{2 \pi}{x_{t} - x_{t+1}} = \frac{2\pi}{\eta_t \sqrt{K_1} \Delta} \le \pi \sqrt{K_1} , \quad
    c_t = \arctan\left( - \frac{\eta_t}{2 \pi }K_1 \Delta \right) - b_t x_t \\
    a_t =&\frac{\eta_t}{2 \pi} K_1 \Delta^2 \sqrt{1 + \frac{\eta_t^2}{4\pi^2 } K_1^2 \Delta^2} , \quad 
    d_t = a_t+\Delta-\frac{\alpha_t}{\alpha_t+\sqrt{1+\alpha_t^2}}\Delta,
\end{align*}
where $\alpha_t=\frac{\sqrt{K_1}}{b_t}$. It is not hard to verify that
\begin{align*}
        f_t(x_{t+1}) =& f_{t+1}(x_{t+1}) = \Delta, \\
        f_t'(x_{t+1}) =& f_{t+1}'(x_{t+1}) = \sqrt{K_1} \Delta, \\
        f_t''(x_{t+1}) =& f_{t+1}''(x_{t+1}) = K_1 \Delta. \\
    \end{align*}
Then we can link all these $f_t$ together.
Moreover, note that
$$
\left|f_t^{\prime \prime}(x)\right|=\left|a_t b_t^2 \sin \left(b_t x+c_t\right)\right| \leq a_t b_t^2=K_1 \Delta \sqrt{1+\frac{1}{\alpha_t^2}} \leq 2 \pi K_1 \Delta,
$$
where in the last inequality we use $\alpha_t=\frac{1}{2 \pi} \eta_t K_1 \Delta>\frac{1}{\pi}$. Also note that
$$
f_t(x)=a_t \sin \left(b_t x+c_t\right)+d_t \geq d_t-a_t=\Delta-\frac{\alpha_t}{\alpha_t+\sqrt{1+\alpha_t^2}} \Delta>\Delta-\frac{1}{2} \Delta=\frac{1}{2} \Delta,
$$
where in the second inequality we use $\frac{\alpha_t}{\alpha_t+\sqrt{1+\alpha_t^2}}<\frac{1}{2}$.
If $f^\star=0$, we immediately obtain that $f$ is $(1,0,4\pi K_1)$-smooth. Next, we extend $f(x)$ from $x_0$ to $+\infty$ to achieve this.
Define 
$$
G(t)=\Delta\left[1+2 \sqrt{K_1}\left(t-x_0\right)+2 K_1\left(t-x_0\right)^2\right] e^{-\sqrt{K_1}\left(t-x_0\right)}, t>x_0.
$$
It is not hard to verify that $G(x_0)=\Delta$, $G^\prime(x_0)=\sqrt{K_1}\Delta$ and $G^{\prime\prime}(x_0)=K_1\Delta$. Moreover, we have $G(t)>0,t>x_0$ and $G(t)\to 0$ as $t\to\infty$. Therefore, $G^\star=0$. We compute that
$$
G^{\prime \prime}(t)=\Delta K_1 e^{-\sqrt{K_1}\left(t-x_0\right)}\left[1-6 \sqrt{K_1}\left(t-x_0\right)+2 K_1\left(t-x_0\right)^2\right].
$$
Therefore, we have 
\begin{align*}
    |G^{\prime \prime}(t)|&\leq \Delta K_1 e^{-\sqrt{K_1}\left(t-x_0\right)}\left[1+6 \sqrt{K_1}\left(t-x_0\right)+2 K_1\left(t-x_0\right)^2\right]\\
\end{align*}
We immediately obtain that $G(t)$ is $(1,0,3K_1)$-smooth. Finally, we define 
$$
f(x)=\begin{cases}
    G(x), \quad &x>x_0,\\
    f_t(x), \quad & x\in(x_{t+1},x_t],t\in\mathbb{N}
\end{cases}$$
We have that $f$ is $(1,0,4\pi K_1)$-smooth, $f^\star=0$, $f(x_0)-f^\star=\Delta$ and $f^{\prime}(x_t)=\sqrt{K_1}\Delta,\forall t\in\mathbb{N}$.

\textbf{Case~\ref{lower_case3}}
We let $\epsilon\leq \frac{2\sqrt{K_1}\Delta}{5}$ and $\epsilon\leq\frac{1}{2\eta_\tau\sqrt{K_1}}$ for simplicity.

We construct the function:
\begin{equation}\label{eq:lower_overall_function}
    f(x) = \left\{ \begin{array}{cc}
          h(x) ,  & x \in (-\infty, x_{\tau+1}] \\
          g(x) ,  & x \in ( x_{\tau+1}, x_\tau] \\
          f_t(x) , & x \in ( x_{t+1}, x_t ] \text{ for all } t \le \tau - 1
        \end{array} \right.
\end{equation}
where $h$ is defined in \eqref{eq:lower_h}, $g$ and $f_t$ are functions to be defined.
$x_{t+1} = x_0 - \sum_{s=0}^t \eta_s \sqrt{K_1} \Delta$ and the initial point is $x_0 = y_0 + \sum_{s=0}^{\tau} \eta_s \sqrt{K_1} \Delta$. $y_0>\frac{1}{\sqrt{K_1}}$ lies in the domain of $h$ and $h(y_0)=2\Delta+\frac{5\epsilon}{4\sqrt{K_1}}+M$, where $M>0$ is a constant to be determined.

The basic idea of our construction in this case is to let the $(\tau+1)$-th iterate be 
$x_{\tau+1}=y_0$.
Then, the worst-case convergence in Case~\ref{lower_case1} can be applied. 
We then want the iterates with large learning rates ($t\leq\tau$), making no progress in convergence. We construct trigonometric functions $f_t,t\leq\tau-1$ to achieve this.
Finally, we use a polynomial function $g$ to link the functions $f_{\tau-1}$ and $h$. 

For $t\leq\tau-1$, we define
    \begin{equation}\label{eq:lower_trigonometric}
        f_t(x) = a_t \sin(b_t x+c_t) + d_t, \quad x\in (x_{t+1}, x_t]
    \end{equation}
    with
    \begin{align*}
        b_t =& \frac{2 \pi}{x_{t} - x_{t+1}} = \frac{2\pi}{\eta_t \sqrt{K_1} \Delta} \le \pi \sqrt{K_1} , \quad
        c_t = \arctan\left( - \frac{\eta_t}{2 \pi }K_1 \Delta \right) - b_t x_t \\
        a_t =&\frac{\eta_t}{2 \pi} K_1 \Delta^2 \sqrt{1 + \frac{\eta_t^2}{4\pi^2 } K_1^2 \Delta^2} , \quad 
        d_t = a_t+g(x_\tau)-\frac{\alpha_t}{\alpha_t+\sqrt{1+\alpha_t^2}}\Delta,
    \end{align*}
    where $\alpha_t=\frac{\sqrt{K_1}}{b_t}$ and 
    $g(x_\tau)\in[7\Delta,8\Delta]$ is to be determined.
    Note that $f^* = 0$ (achieved when $x=0$).
    By Lemma~\ref{lem:lower_property_trigonometric}, we have that $f_t$ is $(1,0,K_1)$-smooth.
     Also, with the above parameter choices, we have
    \begin{align*}
        f_t(x_{t+1}) =& f_{t+1}(x_{t+1}) = g(x_\tau), \\
        f_t'(x_{t+1}) =& f_{t+1}'(x_{t+1}) = \sqrt{K_1} \Delta, \\
        f_t''(x_{t+1}) =& f_{t+1}''(x_{t+1}) = K_1 \Delta. \\
    \end{align*}
    Then we can link all these $f_t,t\leq\tau-1$ together to achieve $(1,0, K_1)$-smooth function, as in the boundary $x_t$ and $x_{t+1}$, the function value and derivatives are the same.

Next, we try to construct a polynomial function $g$ to link $f_{\tau-1}$ and $h$. 
We first show how to interpolate a normalized polynomial function $\bar{g}$.
Let
$\bar{g}(z)=az^4+bz^3+cz^2+dz+e,z\in[0,1]$
and
$g^\prime(0)=A,g^\prime(1)=B,g^{\prime\prime}(0)=C,g^{\prime\prime}(1)=D$. We have
$a=\frac{2A-2B+C+D}{4},b=\frac{-3A+3B-2C-D}{3},c=\frac{C}{2},d=A$. 

We scale $\bar{g}$ to obtain our desired link function $g$.
Let 
$g(y)=\bar{g}\left(\frac{y-x_{\tau+1}}{x_\tau-x_{\tau+1}}\right),y\in[x_{\tau+1},x_\tau]$.
To link $f_{\tau-1}$ and $h$, we need 
\begin{align*}
    & g^{\prime}(x_{\tau+1})=h^\prime(x_{\tau+1}),\quad 
    g^{\prime}(x_\tau)=f^{\prime}_{\tau-1}(x_\tau),\\
    & g^{\prime\prime}(x_{\tau+1})=h^{\prime\prime}(x_{\tau+1}),\quad 
    g^{\prime\prime}(x_\tau)=f^{\prime\prime}_{\tau-1}(x_\tau).\\
\end{align*}
Substituting the corresponding values, we obtain that
\begin{align*}
    &2\epsilon=\frac{1}{x_{\tau}-x_{\tau+1}}\bar{g}^\prime(0), \quad \sqrt{K_1}\Delta=\frac{1}{x_{\tau}-x_{\tau+1}}\bar{g}^\prime(1),\\
    &0=\frac{1}{(x_{\tau}-x_{\tau+1})^2}\bar{g}^{\prime\prime}(0), \quad K_1\Delta=\frac{1}{(x_{\tau}-x_{\tau+1})^2}\bar{g}^{\prime\prime}(1).
\end{align*}
Therefore, we have 
\begin{align*}
    A=2\epsilon\paren{x_{\tau}-x_{\tau+1}},\quad B=\sqrt{K_1}\Delta \paren{x_{\tau}-x_{\tau+1}},\quad C=0,\quad
    D=K_1\Delta \paren{x_{\tau}-x_{\tau+1}}^2.
\end{align*}
Note that
$$x_{\tau+1}=x_\tau-\eta_\tau g^\prime(x_\tau)=x_\tau-\eta_\tau\sqrt{K_1}\Delta.$$
We consider the case $\eta K_1\Delta>6$. 
Since $x_{\tau}-x_{\tau+1}=\eta_\tau \sqrt{K_1}\Delta$,
it is not hard to verify that $\frac{D}{B}=\eta_\tau K_1\Delta>6$. Let 
$$e=\frac{1}{12}D-\frac{1}{2}B+7\Delta+\frac{5\epsilon}{4\sqrt{K_1}}.$$ 

Now we obtain our desired link function $g$: 
\begin{equation}\label{eq:lower_polynomial}
    \begin{aligned}
        &g(y):=\bar{g}\paren{\frac{y-x_{\tau+1}}{x_\tau-x_{\tau+1}}}, y\in[x_{\tau+1},x_\tau], \quad\text{where}\\
        &\bar{g}(z)=az^4+bz^3+cz^2+dz+e,z\in[0,1],\\
        & a=\frac{2A-2B+C+D}{4},\quad b=\frac{-3A+3B-2C-D}{3}, \quad c=\frac{C}{2}, \quad d=A,\\
        & e= \frac{1}{12}D-\frac{1}{2}B+7\Delta+\frac{5\epsilon}{4\sqrt{K_1}},\\
        &A=2\epsilon\paren{x_\tau-x_{\tau+1}},\quad B=\sqrt{K_1}\Delta\paren{x_\tau-x_{\tau+1}},\quad C=0,\quad D=K_1\Delta\paren{x_\tau-x_{\tau+1}}^2.
    \end{aligned}
\end{equation}

By Lemma~\ref{lem:lower_property_polynomial}, 
we have that
$g$ is $(1,0,K_1)$-smooth and
$g(x_\tau)=\frac{1}{2}A+7\Delta+\frac{5\epsilon}{4\sqrt{K_1}}\in[7\Delta,8\Delta]$, where we use $\frac{5\epsilon}{4\sqrt{K_1}}\leq \frac{1}{2}\Delta$ and $A\leq\Delta$.

Now we conclude that $f$ defined in \eqref{eq:lower_overall_function} is $(1,\epsilon\sqrt{K_1},K_1)$-smooth,
since at each junction point,
the left and right functions share identical values, first and second derivatives.
$f$ is also $(1,\sqrt{K_1},K_1)$-smooth if $\epsilon\leq 1$.
Finally, by $g(x_{\tau+1})=h(x_{\tau+1})=e=\frac{1}{12}D-\frac{1}{2}B+7\Delta+\frac{5\epsilon}{4\sqrt{K_1}}$, we have
\begin{align*}
    2\epsilon\paren{x_{\tau+1}-\frac{1}{\sqrt{K_1}}}+\frac{5\epsilon}{4\sqrt{K_1}}=\frac{1}{12}D-\frac{1}{2}B+7\Delta+\frac{5\epsilon}{4\sqrt{K_1}}.
\end{align*}
It takes at least 
\begin{align*}
    \frac{\frac{1}{12}D-\frac{1}{2}B+7\Delta}{4\epsilon^2\paren{\frac{2}{K_1\Delta}}}\geq \frac{2\Delta}{4\epsilon^2\paren{\frac{2}{K_1\Delta}}}=\frac{K_1\Delta^2}{4\epsilon^2}
\end{align*}
iterations to reach an $\epsilon$-stationary point.

For the case $\eta_\tau K_1\Delta\in (2,6]$, it is easier to construct such a function $g$ using a similar analysis since $\eta_\tau K_1\Delta$ is bounded. We thus omit this case.

\end{proof}

\begin{lemma}\label{lem:lower_property_trigonometric}
    Suppose $g(x_\tau)>2\pi\Delta+\frac{1}{2}\Delta$.
    Consider the function defined in \eqref{eq:lower_trigonometric}. We have $f_t(x)>0$, 
    $\left|f^{\prime\prime}(x)\right|\leq K_1 f(x)$
    and
    $f_t(x_t)=g(x_\tau)$
    for all $x\in(x_{t+1},x_t]$, $t\leq \tau-1$.
\end{lemma}

\begin{proof}
    By the definition of $\alpha_t:=\frac{\sqrt{K_1}}{b_t}$, we have
    $$a_t=\Delta\alpha_t\sqrt{1+\alpha_t^2},\quad b_t=\frac{\sqrt{K_1}}{\alpha_t}.$$
    We calculate $f_t^{\prime\prime}(x)$:
    \begin{align*}
        \left|f_t^{\prime\prime}(x)\right|=\left|a_tb_t^2\operatorname{sin}\paren{b_tx+c_t}\right|\leq a_tb_t^2=K_1\Delta\sqrt{1+\frac{1}{\alpha_t^2}}\leq 2\pi K_1\Delta,
    \end{align*}
    where in the last inequality we use $\alpha_t=\frac{1}{2\pi}\eta_t K_1\Delta>\frac{1}{\pi}$.
    Then, we calculate $f_t(x)$:
    \begin{align*}
        f_t(x)=a_t\operatorname{sin}\paren{b_tx+c_t}+d_t\geq d_t-a_t=g(x_\tau)-\frac{\alpha_t}{\alpha_t+\sqrt{1+\alpha_t^2}}\Delta>g(x_\tau)-\frac{1}{2}\Delta>2\pi\Delta,
    \end{align*}
    where in the second inequality we use $\frac{\alpha_t}{\alpha_t+\sqrt{1+\alpha_t^2}}<\frac{1}{2}$ and the last inequality is due to the condition on $g(x_\tau)$.
    Then we immediately obtain that $\left|f^{\prime\prime}(x)\right|\leq K_1f(x)$.

    We calculate $f_t(x_t)$:
    \begin{align*}
        f_t(x_t)&=a_t\operatorname{sin}(c_t)+d_t=-\Delta\alpha_t^2+d_t\\
        &=-\Delta\alpha_t^2+a_t+g(x_\tau)-\frac{\alpha_t}{\alpha_t+\sqrt{1+\alpha_t^2}}\Delta\\
        &=-\Delta\alpha_t^2+\Delta\alpha_t\sqrt{1+\alpha_t^2}+g(x_\tau)-\frac{\alpha_t}{\alpha_t+\sqrt{1+\alpha_t^2}}\Delta\\
        &=g(x_\tau).
    \end{align*}

\end{proof}

\begin{lemma}\label{lem:lower_property_polynomial}
    Consider the function $g(y)$ defined in \eqref{eq:lower_polynomial}. Then it holds that
    \begin{align*}
        &g(x_{\tau+1})=e, g(x_\tau)=\frac{1}{2}A+7\Delta+\frac{5\epsilon}{4\sqrt{K_1}},\\
        &g(y)\geq \Delta, \quad \left|g^{\prime\prime}(y)\right|\leq K_1\Delta, \forall y\in[x_{\tau+1},x_\tau].
    \end{align*}
\end{lemma}

\begin{proof}
Firstly, 
\begin{equation}\nonumber
    g(x_{\tau+1})=\bar{g}(0)=e,\quad g(x_{\tau})=\bar{g}(1)=\frac{1}{2}A+\frac{1}{2}B-\frac{1}{12}D+e=\frac{1}{2}A+7\Delta+\frac{5\epsilon}{4\sqrt{K_1}}.
\end{equation}
Moreover, for $0\leq z\leq 1$, we compute
\begin{equation}\nonumber
    \begin{aligned}
        \bar{g}(z)&=az^4+bz^3+dz+e\\
        &=A\paren{\frac{1}{2}z^4-z^3+z}+Bz^3\paren{1-\frac{1}{2}Z}+Dz^3\paren{-\frac{1}{3}+\frac{1}{4}z}+e\\
        &\geq Bz^3\paren{1-\frac{1}2{z}}+Dz^3\paren{-\frac{1}{3}+\frac{1}{4}z}+\frac{1}{12}D-\frac{1}{2}B+7\Delta+\frac{5\epsilon}{4\sqrt{K_1}},
    \end{aligned}
\end{equation}
where the inequality is due to $\frac{1}{2}z^4-z^3+z\geq0$,  $\forall 0\leq z\leq 1$. Let $\eta K_1\Delta=n$ for simplicity. We have that $B=n\Delta$ and $D=n^2\Delta$. Then we have
\begin{equation}\nonumber
    \begin{aligned}
        \bar{g}(z)&\geq B\paren{z^3-\frac{1}{2}z^4}+B\paren{-\frac{n}{3}z^3+\frac{n}{4}z^4}+\frac{nB}{12}-\frac{1}{2}B+7\Delta+\frac{5\epsilon}{4\sqrt{K_1}}\\
        &=B\paren{-\frac{1}{2}z^4+\frac{n}{4}z^4+z^3-\frac{n}{3}z^3+\frac{1}{12}n-\frac{1}{2}}+7\Delta+\frac{5\epsilon}{4\sqrt{K_1}}\\
        &=\Delta\frac{n\paren{6n^2-28n+33}}{-12\paren{n-2}^3}+7\Delta+\frac{5\epsilon}{4\sqrt{K_1}}\\
        &\geq \Delta,
    \end{aligned}
\end{equation}
where in the last inequality we use $\frac{n\paren{6n^2-28n+33}}{-12\paren{n-2}^3}>-1,\forall n>6$.
Therefore, we have $g(y)\geq \Delta,\forall\ y\in[x_{\tau+1},x_\tau]$.

Finally, we calculate $\bar{g}^{\prime\prime}$:
\begin{align*}
    \bar{g}^{\prime\prime}(z)=12az^2+6bz=6Az\paren{z-1}+6Bz\paren{1-z}+Dz\paren{3z-2},\quad z\in[0,1]
\end{align*}
Then
\begin{equation}\nonumber
    \begin{aligned}
        |\bar{g}^{\prime\prime}(z)|&\leq \frac{3}{2}A+\frac{3}{2}B+\frac{1}{3}D\\
        &\leq \frac{3}{2}A+\frac{1}{4}D+\frac{1}{3}D\\
        &\leq 2\Delta+\frac{7}{12}D\\
        &=\frac{3}{2}\Delta+\frac{7}{12}\Delta\paren{\eta_\tau K_1\Delta}^2\leq \Delta\paren{\eta_\tau K_1\Delta}^2,
    \end{aligned}
\end{equation}
where we use $A\leq\Delta$ and $\eta_\tau K_1\Delta>6$. By $g^{\prime\prime}(y)=\frac{1}{\paren{x_\tau-x_{\tau+1}}^2}\bar{g}\paren{\frac{y-x_{\tau+1}}{x_\tau-x_{\tau+1}}}$, we have
\begin{align*}
    \left|g^{\prime\prime}(y)\right|\leq \frac{\Delta\paren{\eta_\tau K_1\Delta}^2}{\paren{x_{\tau}-x_{\tau+1}}^2}=K_1\Delta,
\end{align*}
where we use $x_{\tau}-x_{\tau+1}=\eta_\tau\sqrt{K_1}\Delta$. This completes the proof.
\end{proof}

\subsection{Explanation of Example~\ref{example:river_valley}}\label{apx:proof_example_river_valley}
\begin{itemize}[leftmargin=*]
\item 
{
To show $f$ is $\left(1, K_1, K_1\right)$-smooth, first note that $f(x, y) \geq 0, f^{\star}=0$ and 
$\left\|\nabla^2 f(x, y)\right\|=\max \left\{\left|h^{\prime \prime}(x)\right|, K_1\right\}$. 
For $|x| \leq 1 / \sqrt{K_1}$, 
we have $\left\|\nabla^2 f(x, y)\right\|=K_1 \leq K_1+f(x, y)$.
For $x>1 / \sqrt{K_1}$, we have $\left\|\nabla^2 f(x, y)\right\|=K_1 e^{\sqrt{K_1} x-1} \leq K_1+K_1 f(x, y)$.}
\item Consider first the constant learning rate case. At the starting point, we have $h(x) = \Delta_0$ and $\nabla h(x) = K_1 \Delta_0$. To enable stable training in the first place, we need to take $\eta = \eta_0 \le \frac{2}{K_1 \Delta}$, or otherwise the algorithm will suffer oscillation on the $x$-axis. 

Then we look at the $y$-axis, which is a simple quadratic problem. For each iteration, we have
\begin{align}\label{eq:proof_example_river_valley_quadratic}
    y_{t+1} = y_t - \eta K_1 y_t = (1 - \eta K_1) y_t = (1-\eta K_1 )^{t+1}y_0 ,
\end{align}
which requires $t = \Theta(\frac{1}{\eta K_1}) = \Theta(\Delta_0)$ iterations to converge in the $y$-axis.

\item Next, we consider the adaptive warmup strategy, with $\eta_t = \cO\left( \min\{ \frac{1}{K_1}, \frac{1}{K_1 \Delta_t} \} \right)$. With this learning rate schedule, it takes $\log(\Delta_0)$ steps to converge to around $0$ in the $x$-axis. 

Then, we know that the local smoothness of the curvature is $K_1$ when $x \le \frac{1}{\sqrt{K_1}}$, which means that after converging to around $0$ in the $x$-axis, we have $\eta_t = \Theta(\frac{1}{K_1})$. Based on \eqref{eq:proof_example_river_valley_quadratic}, it takes around $\Theta(\frac{1}{\eta K_1}) = \Theta(1)$ iterations to converge in the $y$-axis. 

\end{itemize}
Therefore, we can conclude that GD with warmup can converge $\Theta(\frac{\Delta_0}{\log \Delta_0}) = \tilde{\Theta}(\Delta_0)$ times faster than using constant learning rates in this specific river-valley example.

\section{Proof for Section~\ref{sec:SGD_bounded}}
Let $\bn_t=\bg_t-\nabla f(\bw_t)$. We also define $r_t=\min\left\{C_1\Delta_t^{-\frac{\rho-1}{2}},C_2\right\},L_t=2K_0+\Kr\paren{2\Delta_t}^\rho$ and $G_t=\sqrt{K_0\Delta_t+\Kr 3^\rho\Delta_t^{\rho+1}}$.

\begin{lemma}[Azuma-Hoeffding Inequality]\label{lem:hoeffding}
    Let $\left\{X_k, \mathcal{F}_k\right\}_{k=0}^n$ be a martingale difference sequence with respect to a filtration $\left\{\mathcal{F}_k\right\}$. Suppose the increments are bounded almost surely:

    $$
    \left|X_k-X_{k-1}\right| \leq c_k, \quad \text{a.s. for all } k\geq 0. 
    $$

    Then, for any $t>0$,
    
    $$
    \mathbb{P}\left(X_n-X_0 \geq t\right) \leq \exp \left(-\frac{t^2}{2 \sum_{k=1}^n c_k^2}\right).
    $$
    Equivalently, 
    $$
    \mathbb{P}\left(X_n-X_0 \leq \sqrt{2 \sum_{t=1}^n c_k^2 \log (1 / \delta)}\right) \geq 1-\delta.
    $$
    
\end{lemma}

\begin{lemma}\label{lem:sgd_stepsize_adaptive_bounded}
    Let the adaptive learning rate in Theorem~\ref{thm:SGD_adaptive_a.s.bounded} be 
    \begin{align*}
        \eta_t=\min\left\{\frac{1}{8(\sqrt{2}+1)K_0},\frac{1}{8\paren{\sqrt{2}+1}\Kr\paren{3\Delta_t}^{\rho}},\frac{r_t}{2\sigma},\sqrt{\frac{\Delta_0}{\sigma^2TL_t}}, \frac{\Delta_0}{2\sigma G_t\sqrt{T\log\frac{1}{\delta}}}\right\}.
    \end{align*}
    Then it holds that
    $\eta_t\paren{\Norm{\nabla f(\bw_t)}+\sigma}\leq r_t$, $\eta_t\leq \frac{1}{2L_t}$, $\sigma\eta_t\Norm{\nabla f(\bw_t)}\leq \Delta_0$, $\sum_{t=0}^{T-1}\sigma^2\eta_t^2L_t\leq\Delta_0$ and $2\sum_{t=0}^{T}\sigma^2\eta_t^2\Norm{\nabla f(\bw_t)}^2\log\frac{1}{\delta}\leq\Delta_0^2$.
\end{lemma}
\begin{proof}
    As in Appendix~\ref{apx:proof_GD}, we know that
    $$\frac{1}{4\paren{\sqrt{2}+1}}\min\left\{\frac{1}{K_0},\frac{1}{\Kr\paren{3\Delta_t}^\rho}\right\}\leq \frac{r_t}{\Norm{\nabla f(\bw_t)}}.$$
    By the first three conditions of $\eta_t$, we have
    $$\eta_t\paren{\sigma+\Norm{\nabla f(\bw_t)}}\leq\frac{r_t}{2}+\frac{r_t}{2}=r_t. $$
    Since $L_t=2K_0+\Kr\paren{2\Delta_t}^\rho$, it is not hard to verify that $\eta_t\leq\frac{1}{2L_t}$. The remaining three inequalities can be directly verified by noting that $\Norm{\nabla f(\bw_t)}\leq G_t$ by Lemma~\ref{lem:bounded_gradient}.
\end{proof}

\begin{lemma}\label{lem:sgd_stepsize_constant_bounded}
    Let the constant learning rate in Theorem~\ref{thm:SGD_constant_a.s.bounded} be
    \begin{align*}
        \eta=\min\left\{\frac{1}{8(\sqrt{2}+1)K_0},\frac{1}{8\paren{\sqrt{2}+1}\Kr\paren{3\Delta_c}^{\rho}},\frac{r}{2\sigma},\sqrt{\frac{\Delta_0}{\sigma^2TL}}, \frac{\Delta_0}{\sigma G\sqrt{2T\log\frac{1}{\delta}}}, \frac{\Delta_0}{\sigma\alpha}\right\},
    \end{align*}
    where $\Delta_c=4\Delta_0, r=\min\left\{C_1\Delta_c^{-\frac{\rho-1}{2}}, C_2\right\},L=2K_0+\Kr\paren{2\Delta_c}^\rho$,  $G=\sqrt{K_0\Delta_c+\Kr 3^\rho\Delta_c^{\rho+1}}$, and 
    $\alpha=\paren{G+LC_2}\paren{1+\sqrt{2\log\frac{1}{\delta}}}$.
    Then as long as $\Delta_t\leq 4\Delta_0$, it holds that 
    $\eta\paren{\sigma+\Norm{\nabla f(\bw_t)}}\leq r, \eta\leq\frac{1}{2L}, \sigma^2\eta^2LT\leq\Delta_0, 2\sigma^2\eta^2G^2T\log\frac{1}{\delta}\leq\Delta_0^2$ and $\sigma\eta \alpha\leq\Delta_0$.

\end{lemma}
\begin{proof}
    Note that $\Norm{\nabla f(\bw_t)}\leq G_t\leq G$ if $\Delta_t\leq\Delta_c=4\Delta_0$. Then 
    the proof is almost the same as in Lemma~\ref{lem:sgd_stepsize_adaptive_bounded} by replacing  $\Delta_t$ with $4\Delta_0$.
\end{proof}

\begin{lemma}\label{lem:sgd_stepsize_complexity_bounded}
    Consider the adaptive learning rate defined in Lemma~\ref{lem:sgd_stepsize_adaptive_bounded}. Suppose $\Delta_t\leq4\Delta_0$.
    Then  we have 
    \begin{align*}
       \frac{\Delta_0}{\sum_{t=0}^{T-1}\eta_t}&\leq \mathcal{O}\left(\frac{\Delta_0}{T}\paren{K_0+\Kr\Delta_{avg,\rho}}+\sigma\frac{\Delta_0}{T}\paren{C_1^{-1}\Delta_{avg,\frac{\rho-1}{2}}+C_2^{-1}}\right)\\
        &+ \mathcal{O}\paren{\sigma\sqrt{\frac{\Delta_0\log\frac{1}{\delta}}{T}}\paren{K_0+\Kr\Delta_{avg,\rho}}^{1/2}}.
    \end{align*}
    Moreover, consider the constant learning rate defined in Lemma~\ref{lem:sgd_stepsize_constant_bounded}. We have 
    \begin{align*}
        \frac{\Delta_0}{\eta T}&\leq \mathcal{O}\left(\frac{\Delta_0}{T}\paren{K_0+\Kr\Delta_{0}^\rho}+\sigma\frac{\Delta_0}{T}\paren{C_1^{-1}\Delta_{0}^{\frac{\rho-1}{2}}+C_2^{-1}}\right)\\
        &+ \mathcal{O}\paren{\sigma\sqrt{\frac{\Delta_0\log\frac{1}{\delta}}{T}}\paren{K_0+\Kr\Delta_{0}^{\rho}}^{1/2}+\sigma\frac{\sqrt{\log\frac{1}{\delta}}}{T}C_2\paren{K_0+\Kr\Delta_0^\rho}}.
    \end{align*}
\end{lemma}
\begin{proof}
    By the HM-AM inequality, we have
    \begin{equation}\label{eq:adaptive_a.s.bounded_HM-AM}
        \frac{1}{\sum_{t=0}^{T-1}\eta_t}\leq \frac{\sum_{t=0}^T \frac{1}{\eta_t}}{T^2}.
    \end{equation}
    The summation $\sum_{t<T}\frac{1}{\eta_t}$ of the first two items in $\eta_t$ is $\mathcal{O}\paren{T\paren{K_0+\Kr\Delta_{avg,\rho}}}$. We then calculate
    $$\sum_{t=0}^{T-1}\frac{2\sigma}{r_t}\leq 2\sigma\sum_{t=0}^{T-1}\paren{C_1^{-1}\Delta_t^{\frac{\rho-1}{2}}+C_2^{-1}}=2\sigma T\paren{C_1^{-1}\Delta_{avg,\frac{\rho-1}{2}}+C_2^{-1}},$$
    \begin{align*}
        \sum_{t=0}^{T-1}\sqrt{\frac{\sigma^2TL_t}{\Delta_0}}\leq \sqrt{\frac{\sigma^2 T^2}{\Delta_0}}\sqrt{\sum_{t=0}^{T-1}L_t}=\mathcal{O}\paren{\sqrt{\frac{\sigma^2 T^3}{\Delta_0}}\paren{K_0+\Kr\Delta_{avg,\rho}}^{1/2}},
    \end{align*}
    \begin{align*}
        &\sum_{t=0}^{T-1}\frac{\sigma G_t\sqrt{T\log\frac{1}{\delta}}}{\Delta_0}=\frac{\sigma\sqrt{T\log\frac{1}{\delta}}}{\Delta_0}\sum_{t=0}^{T-1} \sqrt{K_0\Delta_t+\Kr3^\rho\Delta_t^{\rho+1}}\\
        &\leq \frac{2\sigma\sqrt{T\log\frac{1}{\delta}}}{\sqrt{\Delta_0}}\sum_{t=0}^{T-1}\sqrt{K_0+\Kr3^\rho\Delta_t^{\rho}}\\
        &\leq \frac{2\sigma\sqrt{T^2\log\frac{1}{\delta}}}{\sqrt{\Delta_0}}\sqrt{\sum_{t=0}^{T-1}K_0+\Kr3^\rho\Delta_t^{\rho}}\\
        &=\mathcal{O}\paren{\sqrt{\frac{\sigma^2T^3\log\frac{1}{\delta}}{\Delta_0}}\paren{K_0+\Kr\Delta_{avg,\rho}}^{1/2}}.\\
    \end{align*}
    Plugging the above inequations into \eqref{eq:adaptive_a.s.bounded_HM-AM} we obtain the desired result.

    For constant learning rate, we simply replace $\Delta_t$ with $4\Delta_0$. 
    The proof is almost the same as adaptive learning rate.
\end{proof}

\subsection{Proof of Theorem~\ref{thm:SGD_adaptive_a.s.bounded}}\label{apx:proof_thm_a.s.bounded_adaptive}

\begin{proof}
    We define $$\tau:=\min\left\{\min\left\{t:f(\bw_t)-f^\star>4\Delta_0\right\}, T\right\}.$$
    For $t<\tau$, by Lemma~\ref{lem:sgd_stepsize_adaptive_bounded}  we have
    $\Norm{\bw_{t+1}-\bw_t}=\eta_t\Norm{\bg_t}\leq \eta_t\paren{\sigma+\Norm{\nabla f(\bw_t)}}\leq r_t$. By Lemma~\ref{lem:improved_descent_smooth}, 
    \begin{align*}
        f(\bw_{t+1})&\leq f(\bw_t)+\dotprod{\nabla f(\bw_t)}{\bw_{t+1}-\bw_t}+\frac{L_t}{2}\Norm{\bw_{t+1}-\bw_t}^2\\
        &= f(\bw_t)-\eta_t\dotprod{\nabla f(\bw_t)}{\bg_t}+\frac{L_t}{2}\eta_t^2\Norm{\bg_t}^2\\
        &\leq f(\bw_t)-\eta_t\Norm{\nabla f(\bw_t)}^2-\eta_t\dotprod{\nabla f(\bw_t)}{\bn_t}+L_t\eta_t^2\Norm{\bn_t}^2+L_t\eta_t^2\Norm{\nabla f(\bw_t)}^2\\
        &\leq f(\bw_t)-\frac{1}{2}\eta_t\Norm{\nabla f(\bw_t)}^2-\eta_t\dotprod{\nabla f(\bw_t)}{\bn_t}+L_t\eta_t^2\Norm{\bn_t}^2,
    \end{align*}
    where in the last inequality we use $\eta_t\leq\frac{1}{2L_t}$ by Lemma~\ref{lem:sgd_stepsize_adaptive_bounded}.
    Telescoping the above inequality from $t=0$ to $\tau-1$, we obtain that
    \begin{equation}\label{eq:sgd_stoppingtime_sum}
        f(\bw_\tau)\leq f(\bw_0)-\frac{1}{2}\sum_{t=0}^{\tau-1}\eta_t\Norm{\nabla f(\bw_t)}^2-\sum_{t=0}^{\tau-1}\eta_t\dotprod{\nabla f(\bw_t)}{\bn_t}+\sum_{t=0}^{\tau-1}L_t\eta_t^2\Norm{\bn_t}^2.
    \end{equation}
    Let $X_t:=-\eta_t\dotprod{\nabla f(\bw_t)}{\bn_t}\boldsymbol{1}_{\tau\geq t}$.
    It is not hard to verify that $-\sum_{t=0}^\tau\eta_t\dotprod{\nabla f(\bw_t)}{\bn_t}=\sum_{t=0}^T X_t$ and $\EE\left[X_t\middle|\bg_0,\dots,\bg_{t-1}\right]=0,\forall t\in[T]$. Therefore, $\{X_t\}$ is a martingale difference sequence. By Lemma~\ref{lem:hoeffding}, 
    $$-\sum_{t=0}^\tau\eta_t\dotprod{\nabla f(\bw_t)}{\bn_t}\leq \sqrt{2\sum_{t=0}^T\eta_t^2\Norm{\nabla f(\bw_t)}^2\Norm{\bn_t}^2\boldsymbol{1}_{\tau\geq t}\log\frac{1}{\delta}}$$
    holds with probability at least $1-\delta$.
    Plugging this into \eqref{eq:sgd_stoppingtime_sum}, we obtain that
    \begin{align*}
        f(\bw_\tau)&\leq f(\bw_0)-\frac{1}{2}\sum_{t=0}^{\tau-1}\eta_t\Norm{\nabla f(\bw_t)}^2+\eta_\tau\dotprod{\nabla f(\bw_\tau)}{\bn_\tau}+\sum_{t=0}^{\tau-1}L_t\eta_t^2\Norm{\bn_t}^2\\
        &\quad+\sqrt{2\sum_{t=0}^T\eta_t^2\Norm{\nabla f(\bw_t)}^2\Norm{\bn_t}^2\boldsymbol{1}_{\tau\geq t}\log\frac{1}{\delta}}\\
        &\leq f(\bw_0)+\eta_\tau{\Norm{\nabla f(\bw_\tau)}}{\Norm{\bn_\tau}}
        +\sum_{t=0}^{T-1}L_t\eta_t^2\Norm{\bn_t}^2\\
        &\quad+\sqrt{2\sum_{t=0}^T\eta_t^2\Norm{\nabla f(\bw_t)}^2\Norm{\bn_t}^2\log\frac{1}{\delta}}\\
        &\leq f(\bw_0)+3\Delta_0,
    \end{align*}
    where we use $\Norm{\bn_t}\leq\sigma$ and Lemma~\ref{lem:sgd_stepsize_adaptive_bounded}.
    Therefore, $\Delta_\tau\leq 4\Delta_0$ and we must have $\tau=T$ with probability at least $1-\delta$. This means $\Delta_t\leq 4\Delta_0,\forall t\in[T]$. By \eqref{eq:sgd_stoppingtime_sum} and $\tau=T$, we obtain that
    \begin{align*}
        \frac{1}{2}\sum_{t=0}^{T-1}\eta_t\Norm{\nabla f(\bw_t)}^2\leq 4\Delta_0.
    \end{align*}
    Therefore, 
    $$\frac{1}{8}\min_{t<T}\Norm{\nabla f(\bw_t)}^2\leq \frac{\Delta_0}{\sum_{t=0}^{T-1}\eta_t}.$$
    By Lemma~\ref{lem:sgd_stepsize_complexity_bounded}, we obtain the desired result.
    
\end{proof}

\subsection{Proof of Theorem~\ref{thm:SGD_constant_a.s.bounded}}
\begin{proof}
    We define $$\tau:=\min\left\{\min\left\{t:f(\bw_t)-f^\star>4\Delta_0\right\}, T\right\}.$$ 
    We also define $r=\min\left\{C_1\paren{4\Delta_0}^{-\frac{\rho-1}{2}},C_2\right\}, L=2K_0+\Kr\paren{8\Delta_0}^\rho$ and $G=\sqrt{4K_0\Delta_0+\Kr3^\rho\paren{4\Delta_0}^{\rho+1}}$.

    For $t<\tau$, by Lemma~\ref{lem:sgd_stepsize_constant_bounded}, we have $\Norm{\bw_{t+1}-\bw_t}=\eta\Norm{\bg_t}\leq\eta\paren{\Norm{\nabla f(\bw_t)}+\sigma}\leq r$. By similar analysis to Appendix~\ref{apx:proof_thm_a.s.bounded_adaptive}, we obtain that with probability at least $1-\delta$,
    \begin{equation}\label{eq:sgd_stoppingtime_sum_constant}
        \begin{aligned}
            f(\bw_\tau)&\leq f(\bw_0)-\frac{1}{2}\sum_{t=0}^{\tau-1}\eta\Norm{\nabla f(\bw_t)}^2-\sum_{t=0}^{\tau-1}\dotprod{\nabla f(\bw_t)}{\bn_t}+\sum_{t=0}^{\tau-1}L\eta^2\Norm{\bn_t}^2\\
            &\leq f(\bw_0)-\frac{1}{2}\sum_{t=0}^{\tau-1}\eta\Norm{\nabla f(\bw_t)}^2+\sum_{t=0}^{\tau-1}L\eta^2\Norm{\bn_t}^2+ \eta\Norm{\nabla f(\bw_\tau)}\Norm{\bn_\tau}\\
            &\quad +\sqrt{2\sum_{t=0}^{\tau}\eta^2\Norm{\nabla f(\bw_t)}^2\Norm{\bn_t}^2\log\frac{1}{\delta}}\\
            &\leq f(\bw_0)+\sum_{t=0}^{T-1}L\eta^2\Norm{\bn_t}^2+\eta\Norm{\nabla f(\bw_\tau)}\Norm{\bn_\tau}+\sqrt{2\eta^2\Norm{\nabla f(\bw_\tau)}^2\Norm{\bn_\tau}^2\log\frac{1}{\delta}}\\
            &\quad+\sqrt{2\sum_{t=0}^{\tau-1}\eta^2\Norm{\nabla f(\bw_t)}^2\Norm{\bn_t}^2\log\frac{1}{\delta}}\\
            &\leq f(\bw_0)+\sum_{t=0}^{T-1}L\eta^2\Norm{\bn_t}^2+\eta\Norm{\nabla f(\bw_\tau)}\Norm{\bn_\tau}+\sqrt{2\eta^2\Norm{\nabla f(\bw_\tau)}^2\Norm{\bn_\tau}^2\log\frac{1}{\delta}}\\
            &\quad+\sqrt{2T\eta^2G^2\sigma^2\log\frac{1}{\delta}}\\
            &\leq f(\bw_0)+2\Delta_0+\eta\Norm{\nabla f(\bw_\tau)}\Norm{\bn_\tau}\paren{1+\sqrt{2\log\frac{1}{\delta}}},
        \end{aligned}
    \end{equation}
    where in the second inequality we use Lemma~\ref{lem:hoeffding}, 
    the second to last inequality is due to $t<\tau$
    and 
    the last inequality is due to $\Norm{\bn_t}\leq \sigma$ and  Lemma~\ref{lem:sgd_stepsize_constant_bounded}.

    Since $\Norm{\bw_\tau-\bw_{\tau-1}}\leq r$, by Lemma~\ref{lem:improved_descent_smooth} we have
    \begin{equation}\label{eq:stoppingtime_gradient_bound}
        \begin{aligned}
            \Norm{\nabla f(\bw_\tau)}&\leq\Norm{\nabla f(\bw_{\tau-1})}+\Norm{\nabla f(\bw_{\tau-1})-\nabla f(\bw_\tau)}\\
            &\leq \Norm{\nabla f(\bw_{\tau-1})}+L\Norm{\bw_{\tau-1}-\bw_\tau}\\
            &\leq G+Lr\leq G+LC_2.
        \end{aligned}
    \end{equation}
    Plugging \eqref{eq:stoppingtime_gradient_bound} into \eqref{eq:sgd_stoppingtime_sum_constant}, we obtain that
    \begin{align*}
        f(\bw_\tau)\leq f(\bw_0)+2\Delta_0+\eta\sigma\paren{1+\sqrt{2\log\frac{1}{\delta}}}\paren{G+LC_2}\leq f(\bw_0)+3\Delta_0,
    \end{align*}
    where the last inequality is due to Lemma~\ref{lem:sgd_stepsize_constant_bounded}. This means $\Delta_\tau\leq 4\Delta_0$ and $\tau=T$ with probability at least $1-\delta$. By \eqref{eq:sgd_stoppingtime_sum_constant} and $\tau=T$, we have
    $$\frac{1}{2}\eta\sum_{t=0}^{T-1}\Norm{\nabla f(\bw_t)}^2\leq 4\Delta_0.$$
    Therefore, with probability at least $1-\delta$ we have
    $$\frac{1}{8}\min_{t<T}\Norm{\nabla f(\bw_t)}^2\leq \frac{\Delta_0}{\eta T}.$$
    By Lemma~\ref{lem:sgd_stepsize_complexity_bounded}, we obtain the desired result.
    
\end{proof}

\section{Proof for Section~\ref{sec:SGD_ABC}}\label{apx:pf_ABC}
In Theorem~\ref{thm:SGD_adaptive_a.s.ABC}, we employ the following adaptive learning rate
\begin{align}\label{eq:learning_rate_thm_ABC_adaptive}
    \begin{gathered}
    \eta_t=\min \left\{\frac{1}{\sqrt{6}(B+1)(4 \sqrt{2}+4)}\left\{\frac{1}{K_0}, \frac{1}{K_\rho\left(3 \Delta_t\right)^\rho}\right\}, \frac{1}{\sqrt{6 A}(2+\sqrt{2})}\left\{\frac{1}{\sqrt{K_0}}, \frac{1}{\sqrt{K_1\left(3 \Delta_t\right)^\rho}}\right\}\right. \\
    \left.\frac{r_t}{\sqrt{6} \sigma}, \sqrt{\frac{\Delta_0^2}{4 G_t^2\left(A \Delta_t+B G_t^2+\sigma^2\right) T \log \frac{1}{\delta}}}, \sqrt{\frac{\Delta_0}{L_t\left(A \Delta_t+\sigma^2\right) T}}\right\} .
    \end{gathered}
\end{align}

In Theorem~\ref{thm:SGD_constant_a.s.ABC}, we employ the following constant learning rate
\begin{align}\label{eq:learning_rate_thm_ABC_constant}
    \begin{aligned}
    \eta=\min&\left\{
        \frac{1}{\sqrt{6}(B+1)(4\sqrt{2}+4)}\left\{\frac{1}{K_0}, \frac{1}{\Kr\paren{3\Delta_c}^\rho}\right\},
        \frac{1}{\sqrt{6A}(2+\sqrt{2})}\left\{\frac{1}{\sqrt{K_0}}, \frac{1}{\sqrt{K_1 \paren{3\Delta_c}^\rho}}\right\},\right.\\
        &\left.\frac{r}{\sqrt{6}\sigma},\sqrt{\frac{\Delta_0^2}{2G^2\paren{A\Delta_c+BG^2+\sigma^2}T\log\frac{1}{\delta}}},
        \sqrt{\frac{\Delta_0}{L\paren{A\Delta_c+\sigma^2}T}}
        \right.,\\
        &\left.
        \frac{1}{\sqrt{A}\alpha}, \frac{1}{\alpha\paren{\frac{1}{2}\sqrt{A}+\sqrt{B}\paren{G+C_2L}+\sigma}}
        \right\},
    \end{aligned}
\end{align}
where $\Delta_c=8\Delta_0$.

\begin{lemma}\label{lem:sgd_stepsize_adaptive_ABC}
    Let the adaptive learning rate in Theorem~\ref{thm:SGD_adaptive_a.s.ABC} be 
    \begin{align*}
        \eta_t=\min&\left\{
        \frac{1}{\sqrt{6}(B+1)(4\sqrt{2}+4)}\left\{\frac{1}{K_0}, \frac{1}{\Kr\paren{3\Delta_t}^\rho}\right\},
        \frac{1}{\sqrt{6A}(2+\sqrt{2})}\left\{\frac{1}{\sqrt{K_0}}, \frac{1}{\sqrt{K_1 \paren{3\Delta_t}^\rho}}\right\},\right.\\
        &\left.\frac{r_t}{\sqrt{6}\sigma},\sqrt{\frac{\Delta_0^2}{4G_t^2\paren{A\Delta_t+BG_t^2+\sigma^2}T\log\frac{1}{\delta}}},
        \sqrt{\frac{\Delta_0}{L_t\paren{A\Delta_t+\sigma^2}T}}
        \right\}.
    \end{align*}
    Then it holds that 
    \begin{align*}
        &2\eta_t^2\paren{A\Delta_t+(B+1)\Norm{\nabla f(\bw_t)}^2+\sigma^2}\leq r_t^2,\\
        &\eta_t\leq\frac{1}{2(B+1)L_t},\quad\sum_{t=0}^{T-1}L_t\eta_t^2 \paren{A\Delta_t+\sigma^2}\leq\Delta_0\\
        &\eta_t\Norm{\nabla f(\bw_t)}\paren{A\Delta_t+B\Norm{\nabla f(\bw_t)}^2+\sigma^2}^{1/2}\leq\Delta_0\\
        &2\sum_{t=0}^T\eta_t^2\Norm{\nabla f(\bw_t)}^2\paren{A\Delta_t+B\Norm{
    \nabla f(\bw_t)}^2+\sigma^2}\log\frac{1}{\delta}\leq\Delta_0^2.
    \end{align*}
\end{lemma}

\begin{proof}
    We first note that
    $$\frac{r_t}{\Norm{\nabla f(\bw_t)}}\geq \frac{1}{4\paren{\sqrt{2}+1}}\min\left\{\frac{1}{K_0},\frac{1}{\Kr\paren{3\Delta_t}^\rho}\right\},$$
    $$\frac{r_t}{\sqrt{\Delta_t}}\geq\frac{1}{2+\sqrt{2}}\min\left\{\frac{1}{\sqrt{K_0}},\frac{1}{\sqrt{\Kr\paren{3\Delta_t}^\rho}}\right\}.$$
    By considering the first five terms in $\eta_t$ and noting that $\sqrt{B+1}\leq B+1$, we have
    \begin{align*}
        2\eta_t^2\paren{A\Delta_t+(B+1)\Norm{\nabla f(\bw_t)}^2+\sigma^2}\leq \frac{r_t^2}{3}\times 3=r_t^2.
    \end{align*}

    It is not hard to verify that $\eta_t\leq\frac{1}{2(B+1)L_t}$. The remaining inequations can be directly verified by noting that $\Norm{\nabla f(\bw_t)}\leq G_t$ by Lemma~\ref{lem:bounded_gradient}.
\end{proof}

\begin{lemma}\label{lem:sgd_stepsize_constant_ABC}
    Let the constant learning rate in Theorem~\ref{thm:SGD_constant_a.s.ABC} be 
    \begin{align*}
        \eta=\min&\left\{
        \frac{1}{\sqrt{6}(B+1)(4\sqrt{2}+4)}\left\{\frac{1}{K_0}, \frac{1}{\Kr\paren{3\Delta_c}^\rho}\right\},
        \frac{1}{\sqrt{6A}(2+\sqrt{2})}\left\{\frac{1}{\sqrt{K_0}}, \frac{1}{\sqrt{K_1 \paren{3\Delta_c}^\rho}}\right\},\right.\\
        &\left.\frac{r}{\sqrt{6}\sigma},\sqrt{\frac{\Delta_0^2}{2G^2\paren{A\Delta_c+BG^2+\sigma^2}T\log\frac{1}{\delta}}},
        \sqrt{\frac{\Delta_0}{L\paren{A\Delta_c+\sigma^2}T}}
        \right.,\\
        &\left.
        \frac{1}{\sqrt{A}\alpha}, \frac{1}{\alpha\paren{\frac{1}{2}\sqrt{A}+\sqrt{B}\paren{G+C_2L}+\sigma}}
        \right\}
    \end{align*}
    where $\Delta_c=8\Delta_0, r=\min\left\{C_1\Delta_c^{-\frac{\rho-1}{2}}, C_2\right\},L_t=2K_0+\Kr\paren{2\Delta_c}^\rho$, $G=\sqrt{K_0\Delta_c+\Kr 3^\rho\Delta_c^{\rho+1}}$ and $\alpha=\paren{G+LC_2}\paren{1+\sqrt{2\log\frac{1}{\delta}}}$. Then as long as $\Delta_t\leq 8\Delta_0$, we have
    \begin{align*}
        &2\eta^2\paren{A\Delta_c+(B+1)G^2+\sigma^2}\leq r^2,\\
        &\eta\leq\frac{1}{2(B+1)L},\quad\eta^2LT\paren{A\Delta_c+\sigma^2}\leq\Delta_0,\\
        &2\eta^2G^2\paren{A\Delta_c+BG^2+\sigma^2}T\log\frac{1}{\delta}\leq\Delta_0^2,\\
        &\eta\sqrt{A}\alpha\leq 1,\\
        &\eta\alpha\paren{\frac{1}{2}\sqrt{A}+\sqrt{B}G+\sqrt{B}C_2L+\sigma}\leq \Delta_0.
    \end{align*}
\end{lemma}

\begin{proof}
    Note that $\Norm{\nabla f(\bw_t)}\leq G_t\leq G$ if $\Delta_t\leq\Delta_c=8\Delta_0$. Then the proof is almost the same as in Lemma~\ref{lem:sgd_stepsize_adaptive_ABC}, by replacing $\Delta_t$ with $8\Delta_0$.
\end{proof}

\begin{lemma}\label{lem:sgd_stepsize_complexity_ABC}
    Consider the adaptive learning rate defined in Lemma~\ref{lem:sgd_stepsize_adaptive_ABC}. 
    Suppose $\Delta_t\leq 4\Delta_0,\forall t\in[T]$.
    Then  we have 
    \begin{align*}
        \frac{\Delta_0}{\sum_{t=0}^{T-1}\eta_t}&\leq
        \mathcal{O}\paren{\frac{\Delta_0}{T}\paren{K_0+\Kr\Delta_{avg,\rho}}+\sigma\frac{\Delta_0}{T}\paren{C_1^{-1}\Delta_{avg,\frac{\rho-1}{2}}+C_2^{-1}}
        +\frac{\Delta_0}{T}\sqrt{A}\paren{\sqrt{K_0}+\sqrt{\Kr\Delta_{avg,\rho}}}}\\
        &+\mathcal{O}\paren{\sqrt{\frac{\Delta_0\log\frac{1}{\delta}}{T}}\paren{ \paren{\sigma+\sqrt{A\Delta_0}}\paren{K_0+\Kr\Delta_{avg,\rho}}^{1/2}+\sqrt{B\Delta_0}\paren{K_0+\Kr\Delta_{avg,\rho}}}},
    \end{align*}
    Moreover, consider the constant learning rate defined in Lemma~\ref{lem:sgd_stepsize_constant_ABC}. We have 
    \begin{align*}
        \frac{\Delta_0}{\eta T}&\leq
        \mathcal{O}\paren{\frac{\Delta_0}{T}\paren{K_0+\Kr\Delta_{0}^{\rho}}+\sigma\frac{\Delta_0}{T}\paren{C_1^{-1}\Delta_{0}^{\frac{\rho-1}{2}}+C_2^{-1}}
        +\frac{\Delta_0}{T}\sqrt{A}\paren{\sqrt{K_0}+\sqrt{\Kr\Delta_{0}^{\rho}}}}\\
        &+\mathcal{O}\paren{\sqrt{\frac{\Delta_0\log\frac{1}{\delta}}{T}}\paren{ \paren{\sigma+\sqrt{A\Delta_0}}\paren{K_0+\Kr\Delta_{0}^{\rho}}^{1/2}+\sqrt{B\Delta_0}\paren{K_0+\Kr\Delta_{0}^{\rho}}}}\\
        &+\mathcal{O}\paren{\frac{\alpha\sqrt{A}\Delta_0}{T}+\frac{\alpha\paren{\sqrt{A}+\sqrt{B}(G+C_2L)+\sigma}}{T}}.
    \end{align*}
\end{lemma}

\begin{proof}
     First, by the HM-AM inequality, we have
    \begin{equation}\label{eq:adaptive_a.s.ABC_HM-AM}
        \frac{1}{\sum_{t=0}^{T-1}\eta_t}\leq \frac{\sum_{t=0}^T \frac{1}{\eta_t}}{T^2}.
    \end{equation}
    The summation $\sum_{t<T}\frac{1}{\eta_t}$ of the first five terms in $\eta_t$ in Lemma~\ref{lem:sgd_stepsize_adaptive_ABC} is $$\mathcal{O}\paren{T(B+1)\paren{K_0+\Kr\Delta_{avg,\rho}}+T\sqrt{A}\paren{\sqrt{K_0}+\sqrt{\Kr\Delta_{avg,\rho}}}+T\sigma\paren{C_1^{-1}\Delta_{avg,\frac{\rho-1}{2}}+C_2^{-1}}}.$$
    We note that 
    $$G_t=2\sqrt{K_0\Delta_t+3^\rho\Kr\Delta_t^{\rho+1}}\leq 
    2\sqrt{4\Delta_0}\sqrt{K_0+\Kr 3^\rho\Delta_t^{\rho}},$$
    where in the equality we use the definition of $G_t$ and in the inequality we use $\Delta_t\leq 4\Delta_0$. 
    Consider the second to last term in $\eta_t$ in Lemma~\ref{lem:sgd_stepsize_adaptive_ABC}, we calculate
    \begin{align*}
        &\sum_{t=0}^{T-1}G_t\sqrt{A\Delta_t+BG_t^2+\sigma^2}\leq G_t\paren{\sqrt{A\Delta_t}+\sqrt{B}G_t+\sigma}\\
        &\leq \sum_{t=0}^{T-1} 8\sqrt{A}\Delta_0\paren{K_0+\Kr3^\rho\Delta_t^\rho}^{1/2}
        +16\sqrt{B}\Delta_0\paren{K_0+\Kr3^\rho\Delta_t^\rho}+4\sigma\sqrt{\Delta_0}\paren{K_0+\Kr3^\rho\Delta_t^\rho}^{1/2}\\
        &=\mathcal{O}\paren{T\sqrt{A}\Delta_0\paren{K_0+\Kr\Delta_{avg,\rho}}^{1/2}+T\sqrt{B}\Delta_0\paren{K_0+\Kr\Delta_{avg,\rho}}+T\sigma\sqrt{\Delta_0}\paren{K_0+\Kr\Delta_{avg,\rho}}^{1/2}}.
    \end{align*}

    Then we have
    \begin{align*}
        &\sum_{t=0}^{T-1}\sqrt{\frac{4G_t^2\paren{A\Delta_t+BG_t^2+\sigma^2}T\log\frac{1}{\delta}}{\Delta_0^2}}\\
        &=\mathcal{O}\paren{\paren{\sqrt{A\Delta_0}+\sigma}\sqrt{\frac{T^{3/2}\log\frac{1}{\delta}}{\Delta_0}}\paren{K_0+\Kr\Delta_{avg,\rho}}^{1/2}+\sqrt{B}\sqrt{T^{3/2}\log\frac{1}{\delta}}\paren{K_0+\Kr\Delta_{avg,\rho}}}.
    \end{align*}

    Consider the last term in $\eta_t$ in Lemma~\ref{lem:sgd_stepsize_adaptive_ABC}, we calculate
    \begin{align*}
        &\sum_{t=0}^{T-1} \sqrt{L_t\paren{A\Delta_t+\sigma^2}}
        \leq \sum_{t=0}^{T-1} \sqrt{A\Delta_tL_t}+\sigma\sqrt{L_t}
        \leq
        \sum_{t=0}^{T-1} 2\sqrt{A\Delta_0L_t}+\sigma\sqrt{L_t}\\
        &=\sum_{t=0}^{T-1}\paren{2\sqrt{A\Delta_0}+\sigma}\sqrt{2K_0+\Kr2^\rho\Delta_t^\rho}\\
        &=\mathcal{O}\paren{T\paren{2\sqrt{A\Delta_0}+\sigma}\paren{K_0+\Kr\Delta_{avg,\rho}}^{1/2}}.
    \end{align*}
    Then we have
    \begin{align*}
        &\sum_{t=0}^{T-1}\sqrt{\frac{L_t\paren{A\Delta_t+\sigma^2}T}{\Delta_0}}\leq\sqrt{\frac{T}{\Delta_0}}\sum_{t=0}^{T-1}\sqrt{L_t}\paren{\sqrt{A\Delta_t}+\sigma}\\
        &\leq \sqrt{\frac{T}{\Delta_0}}\sum_{t=0}^{T-1}\sqrt{L_t}\paren{2\sqrt{A\Delta_0}+\sigma}\\
        &=\sqrt{\frac{T^{3/2}}{\Delta_0}}\paren{2\sqrt{A\Delta_0}+\sigma}\paren{2K_0+\Kr2^\rho\Delta_{avg,\rho}}^{1/2}\\
        &=\mathcal{O}\paren{\paren{\sqrt{A\Delta_0}+\sigma}\sqrt{\frac{T^{3/2}}{\Delta_0}}\paren{K_0+\Kr\Delta_{avg,\rho}}^{1/2}}.
    \end{align*}
    Combining the above results and plugging into \eqref{eq:adaptive_a.s.ABC_HM-AM}, we get the desired result.

    For constant learning rate, we simply replace $\Delta_t$ with $8\Delta_0$. 
    The proof is almost the same as adaptive learning rate.

\end{proof}

\subsection{Proof of Theorem~\ref{thm:SGD_adaptive_a.s.ABC}}\label{Apx:proof_thm_a.s.ABC_adaptive}

\begin{proof}
    We define $$\tau:=\min\left\{\min\left\{t:f(\bw_t)-f^\star>4\Delta_0\right\}, T\right\}.$$
    For $t<\tau$, by Lemma~\ref{lem:sgd_stepsize_adaptive_ABC} we have
    $$\Norm{\bw_{t+1}-\bw_t}^2=\eta_t^2\Norm{\bg_t}^2\leq 2\eta_t^2\left(\Norm{\bn_t}^2+\Norm{\nabla f(\bw_t)}^2\right)\leq 2\eta_t^2\paren{A\Delta_t+(B+1)\Norm{\nabla f(\bw_t)}^2+\sigma^2}\leq r_t^2.$$ By Lemma~\ref{lem:improved_descent_smooth}, 
    \begin{align*}
        f(\bw_{t+1})&= f(\bw_t)+\dotprod{\nabla f(\bw_t)}{\bw_{t+1}-\bw_t}+\frac{L_t}{2}\Norm{\bw_{t+1}-\bw_t}^2\\
        &\leq f(\bw_t)-\eta_t\Norm{\nabla f(\bw_t)}^2-\eta_t\dotprod{\nabla f(\bw_t)}{\bn_t}+L_t\eta_t^2\paren{\Norm{\bn_t}^2+\Norm{\nabla f(\bw_t)}^2}
        \\
        &\leq f(\bw_t)-\eta_t\Norm{\nabla f(\bw_t)}^2-\eta_t\dotprod{\nabla f(\bw_t)}{\bn_t}+L_t\eta_t^2\paren{1+B}\Norm{\nabla f(\bw_t)}^2+L_t\eta_t^2\paren{A\Delta_t+\sigma^2}
        \\
        &\leq f(\bw_t)-\frac{1}{2}\eta_t\Norm{\nabla f(\bw_t)}^2-\eta_t\dotprod{\nabla f(\bw_t)}{\bn_t}+L_t\eta_t^2\paren{A\Delta_t+\sigma^2},
    \end{align*}
    where in the last inequality we use $\eta_t\leq\frac{1}{2(B+1)L_t}$ by Lemma~\ref{lem:sgd_stepsize_adaptive_ABC}.
    Telescoping the above inequation from $t=0$ to $\tau-1$, we obtain that
    \begin{equation}\nonumber
        f(\bw_\tau)\leq f(\bw_0)-\frac{1}{2}\sum_{t=0}^{\tau-1}\eta_t\Norm{\nabla f(\bw_t)}^2-\sum_{t=0}^{\tau-1}\eta_t\dotprod{\nabla f(\bw_t)}{\bn_t}+\sum_{t=0}^{\tau-1}L_t\eta_t^2\paren{A\Delta_t+\sigma^2}.
    \end{equation}
    Similar to the analysis in Appendix~\ref{apx:proof_thm_a.s.bounded_adaptive}, by Lemma~\ref{lem:hoeffding} we have that with probability at least $1-\delta$,
    \begin{align*}
        f(\bw_\tau)&\leq f(\bw_0)-\frac{1}{2}\sum_{t=0}^{\tau-1}\eta_t\Norm{\nabla f(\bw_t)}^2+\sum_{t=0}^{T-1}L_t\eta_t^2\paren{A\Delta_t+\sigma^2}
        +\eta_\tau\Norm{\nabla f(\bw_\tau)}\Norm{\bn_\tau}\\
        &\quad+\sqrt{2\sum_{t=0}^T\eta_t^2\Norm{\nabla f(\bw_t)}^2\Norm{\bn_t}^2\log\frac{1}{\delta}}\\
        &\leq f(\bw_0)+3\Delta_0,
    \end{align*}
     where the last inequality is due to $\Norm{\bn_t}^2\leq A\Delta_t+B\Norm{\nabla f(\bw_t)}^2+\sigma^2$ and 
     Lemma~\ref{lem:sgd_stepsize_adaptive_ABC}.
    Therefore $\Delta_\tau\leq 4\Delta_0$ and we must have $\tau=T$ with probability at least $1-\delta$. Following similar analysis to Appendix~\ref{apx:proof_thm_a.s.bounded_adaptive}, we have
    $$\frac{1}{8}\min_{t<T}\Norm{\nabla f(\bw_t)}^2\leq \frac{\Delta_0}{\sum_{t=0}^{T-1}\eta_t}.$$
    By Lemma~\ref{lem:sgd_stepsize_complexity_ABC}, we obtain the desired result.
\end{proof}

\subsection{Proof of Theorem~\ref{thm:SGD_constant_a.s.ABC}}

\begin{proof}
    We define $$\tau:=\min\left\{\min\left\{t:f(\bw_t)-f^\star>8\Delta_0\right\}, T\right\}.$$
     We also define $r=\min\left\{C_1\paren{8\Delta_0}^{-\frac{\rho-1}{2}},C_2\right\}, L=2K_0+\Kr\paren{16\Delta_0}^\rho$ and $G=\sqrt{8K_0\Delta_0+\Kr3^\rho\paren{8\Delta_0}^{\rho+1}}$.
    For $t<\tau$, we have $L_t\leq L$ and $G_t\leq G$.

    For $t<\tau$, by Lemma~\ref{lem:sgd_stepsize_constant_ABC}, we have
    \begin{align*}
    \Norm{\bw_{t+1}-\bw_t}^2\leq & 2\eta^2\paren{\Norm{\nabla f(\bw_t)}^2+\Norm{\bn_t}^2}\leq 2\eta^2\paren{A\Delta_t+(B+1)\Norm{\nabla f(\bw_t)}^2+\sigma^2}\\ \leq & 2\eta^2\paren{8A\Delta_0+(B+1)G^2+\sigma^2}\leq r^2.
    \end{align*}
    By similar analysis to Appendix~\ref{Apx:proof_thm_a.s.ABC_adaptive}, we obtain that with probability at least $1-\delta$, 
     
    \begin{equation}\label{eq:stoppingtime_sum_a.s.ABC_constant}
        \begin{aligned}
        f(\bw_\tau)&\leq f(\bw_0)-\frac{1}{2}\sum_{t=0}^{\tau-1}\eta\Norm{\nabla f(\bw_t)}^2+\sum_{t=0}^{\tau-1}L\eta^2\paren{A\Delta_t+\sigma^2}+\eta\Norm{\nabla f(\bw_\tau)}\Norm{\bn_\tau}\\
        &\quad+\sqrt{2\sum_{t=0}^{\tau}\eta^2\Norm{\nabla f(\bw_t)}^2\Norm{\bn_t}^2\log\frac{1}{\delta}}\\
        &\leq f(\bw_0)-\frac{1}{2}\sum_{t=0}^{\tau-1}\eta\Norm{\nabla f(\bw_t)}^2+\sum_{t=0}^{\tau-1}L\eta^2\paren{A\Delta_t+\sigma^2}+\eta\Norm{\nabla f(\bw_\tau)}\Norm{\bn_\tau}\paren{1+\sqrt{2\log\frac{1}{\delta}}}\\
        &\quad+\sqrt{2\sum_{t=0}^{\tau-1}\eta^2\Norm{\nabla f(\bw_t)}^2\Norm{\bn_t}^2\log\frac{1}{\delta}}\\
        &\leq f(\bw_0)+TL\eta^2\paren{8A\Delta_0+\sigma^2}+\sqrt{2T\eta^2G^2\paren{8A\Delta_0+BG^2+\sigma^2}\log\frac{1}{\delta}}\\
        &\quad + \eta\Norm{\nabla f(\bw_\tau)}\Norm{\bn_\tau}\paren{1+\sqrt{2\log\frac{1}{\delta}}}\\
        &\leq f(\bw_0)+2\Delta_0+\eta\Norm{\nabla f(\bw_\tau)}\Norm{\bn_\tau}\paren{1+\sqrt{2\log\frac{1}{\delta}}},
    \end{aligned}
    \end{equation}
    
where the last inequality is due to Lemma~\ref{lem:sgd_stepsize_constant_ABC}.
Note that $\Norm{\bn_\tau}\leq \sqrt{A\Delta_\tau}+\sqrt{B}\Norm{\nabla f(\bw_\tau)}+\sigma\leq \frac{1}{2}\sqrt{A}\Delta_\tau+\frac{1}{2}\sqrt{A}+\sqrt{B}\Norm{\nabla f(\bw_\tau)}+\sigma$. Then we have
\begin{equation}\label{eq:stoppingtime_a.s.ABC_constant}
    \begin{aligned}
    &\eta\Norm{\nabla f(\bw_\tau)}\Norm{\bn_\tau}\paren{1+\sqrt{2\log\frac{1}{\delta}}}\\
    &=\paren{1+\sqrt{2\log\frac{1}
    {\delta}}}\Norm{\nabla f(\bw_\tau)}\paren{\frac{\eta}{2}\sqrt{A}\Delta_\tau+\eta\paren{\frac{1}{2}\sqrt{A}+\sqrt{B}\Norm{\nabla f(\bw_\tau)}+\sigma}}\\
    &\leq \frac{1}{2}\Delta_\tau+\Delta_0,
\end{aligned}
\end{equation}
where in the inequality we bound $\Norm{\nabla f(\bw_\tau)}$ as in \eqref{eq:stoppingtime_gradient_bound} and use Lemma~\ref{lem:sgd_stepsize_constant_ABC}.
Plugging \eqref{eq:stoppingtime_a.s.ABC_constant} into \eqref{eq:stoppingtime_sum_a.s.ABC_constant} we obtain that with probability at least $1-\delta$, 
\begin{align*}
    f(\bw_\tau)\leq f(\bw_0)+3\Delta_0+\frac{1}{2}\Delta_\tau.
\end{align*}
This means with probability at least $1-\delta$, $\Delta_\tau\leq 8\Delta_0$ and thus $\tau=T$. 
Similar to the analysis in Appendix~\ref{Apx:proof_thm_a.s.ABC_adaptive}, 
we have
$$\frac{1}{16}\min_{t<T}\Norm{\nabla f(\bw_t)}^2\leq \frac{\Delta_0}{\eta T}.$$
By Lemma~\ref{lem:sgd_stepsize_complexity_ABC}, we get the desired result.
\end{proof}

\section{Extension to Sub-gaussian Noise}\label{apx:subgaussian}

We first present the sub-Gaussian version of the ABC inequality noise assumption.

\begin{assumption}\label{assumption:subgaussian_ABC}
    $\mathbb{P}\left(\|\bn_t\|\geq t\right)\leq 2\exp\left\{
-\frac{t^2}{c\left(A\Delta_t+B\|\nabla f(x_t)\|^2+\sigma^2\right)}\right\}$  for some $c>0$ and $\forall t>0$.
\end{assumption}

Let $E_t=c(A\Delta_t+BG_t^2+\sigma^2)\log\left(\frac{2T}{\delta}\right)$. We have
$$P(\cup_{t=0}^{T-1}\|\bn_t\|^2>E_t)\leq\sum_{t=0}^{T-1}P(\|\bn_t\|^2>E_t)\leq 2Te^{-\log(2T/\delta)}=\delta.$$
Then, with probability at least $1-\delta$, we have
$$\|\bn_t\|^2\leq c(A\Delta_t+BG_t^2+\sigma^2)\log\left(\frac{2T}{\delta}\right).$$

Therefore, the convergence rate under Assumption~\ref{assumption:subgaussian_ABC} exceeds that in Theorems~\ref{thm:SGD_adaptive_a.s.ABC} and \ref{thm:SGD_constant_a.s.ABC} by at most a logarithmic factor.

\end{document}